\documentclass[nohyperref]{article}

\usepackage{microtype}
\usepackage{graphicx}
\usepackage{subfigure}
\usepackage{booktabs} 

\usepackage{hyperref}

\usepackage[accepted]{icml2022_hacked_for_arXiv}

\usepackage{amsmath}
\usepackage{amssymb}
\usepackage{mathtools}
\usepackage{amsthm}
\usepackage{diagbox}
\usepackage[capitalize,noabbrev]{cleveref}

\theoremstyle{plain}
\newtheorem{theorem}{Theorem}[section]

\newtheorem{lemma}[theorem]{Lemma}
\newtheorem{corollary}[theorem]{Corollary}

\theoremstyle{definition}
\newtheorem{definition}[theorem]{Definition}
\newtheorem{assumption}[theorem]{Assumption}
\theoremstyle{remark}

\input{preamble.tex}

\usepackage[colorinlistoftodos,bordercolor=orange,backgroundcolor=orange!20,linecolor=orange,textsize=scriptsize]{todonotes}

\icmltitlerunning{3PC: Three Point Compressors for Communication-Efficient Distributed Training}
\title{\bf 3PC: Three Point Compressors for Communication-Efficient Distributed Training and a Better Theory for Lazy Aggregation}

\author{\bf Peter Richt\'arik \\ KAUST\thanks{\scriptsize King Abdullah University of Science and Technology, Thuwal, Saudi Arabia.} \and \bf Igor Sokolov \\ KAUST  \and \bf Ilyas Fatkhullin \\ ETH AI Center \& ETH Zurich \and \bf Elnur Gasanov \\ KAUST \and   \bf Zhize Li \\KAUST \and \bf Eduard Gorbunov \\ MIPT\thanks{\scriptsize Moscow Institute of Physics and Technology, Dolgoprudny, Russia.}}

\date{}
\begin{document}
\maketitle

\begin{abstract}
We propose and study a new class of gradient communication mechanisms for communication-efficient training---three point compressors (\algname{3PC})---as well as efficient distributed nonconvex optimization algorithms that can take advantage of them. Unlike most established approaches, which rely on a static compressor choice (e.g., Top-$K$), our class allows the compressors to {\em evolve} throughout the training process, with the aim of improving the theoretical communication complexity and practical efficiency of the underlying methods. We show that our general approach can recover the recently proposed state-of-the-art error feedback mechanism \algname{EF21} \citep{EF21} and its theoretical properties as a special case, but also leads to a number of new efficient methods. Notably, our approach allows us to improve upon the state of the art in the algorithmic and theoretical foundations of the {\em lazy aggregation} literature \citep{LAG}. As a by-product that may be of independent interest, we provide a new and fundamental link between the lazy aggregation and error feedback literature. A special feature of our work is that we do not require the compressors to be unbiased.
\end{abstract}


\section{Introduction}

It has become apparent in the last decade that, other things equal, the practical utility of modern machine learning models grows with their size and with the amount of data points used in the training process. This {\em big model} and {\em big data} approach, however, comes with increased demands on the hardware, algorithms, systems and  software involved in the training process. 

\subsection{Big data and the need for distributed systems} In order to handle the large volumes of data involved in training SOTA models, it is now absolutely necessary to rely on (often massively) {\em distributed} computing systems \cite{Dean2012, DCGD, lin2018deep}. Indeed, due to storage and compute capacity limitations, large-enough training data sets can no longer be stored on a single machine, and instead need to be distributed across and processed by an often large number of parallel workers. 

In particular, in this work we consider distributed supervised learning problems of the form
\begin{equation} 
\label{eq:finit_sum} \squeeze \min \limits_{x\in \R^d} \sb{f(x)\eqdef \frac{1}{n}\sum \limits_{i=1}^n f_i(x)},
\end{equation}
where $n$ is the number of parallel workers/devices/clients,  $x$ is a vector representing the $d$ parameters of a machine learning model (e.g., the weights in a neural network), and $f_i(x)$ is the loss of model $x$ on the training data stored on client $i \in [n]\eqdef \{1, 2,\dots, n\}$.

In some applications, as in {\em federated learning (FL)} \cite{FedAvg2016,FEDLEARN,FEDOPT,FL2017-AISTATS}, the training data is captured in a distributed fashion in the first place, and there are reasons to process it in this decentralized fashion as well, as opposed to  first moving it to a centralized location, such as a datacenter, and subsequently processing it there. Indeed, FL refers to machine learning in the environment where a large collection of highly heterogeneous clients (e.g., mobile devices, smart home appliances or corporations) tries to collaboratively train a model using the diverse data stored on these devices, but without compromising the clients' data privacy. 

\subsection{Big model and the need for communication reduction} While distributing the data across several workers certainly alleviates the per-client storage and compute bottlenecks, the training task is obviously not fully decomposed this way. Indeed,  the $n$ clients still need to work together to train the model, and working together means {\em communication}. 

Since currently the most efficient training mechanisms rely on gradient-type methods \cite{bottou2012stochastic, ADAM, gorbunov2021marina}, and since these operate  by iteratively updating {\em all} the $d$ parameters describing the model, relying on big models  leads to the need to communicate large-dimensional gradient vectors, which is expensive. For this reason, modern distributed methods need to rely on mechanisms that alleviate this communication burden. 

Several  orthogonal algorithmic approaches have been proposed in the literature to tackle this issue. One strain of  methods, particularly popular in FL,  is based on 
{\em richer local training} (e.g., \algname{LocalSGD}), which typically means going beyond a single local gradient step before communication/aggregation across the workers is performed. This strategy is based on the hope that richer local training will ultimately lead to a dramatic reduction in the number of communication rounds  without increasing the local computation time by much \citep{localSGD-Stich, localSGD-AISTATS2020, Blake2020}. Another notable strain of  methods is based on 
{\em communication compression} (e.g., \algname{QSGD}),  which means applying a lossy transformation to the communicated gradient information. This strategy is based on the hope that communication compression will lead to a dramatic reduction in the communication time within each round without affecting the number of communication rounds by much \cite{DCGD, Alistarh-EF-NIPS2018, DIANA, ADIANA, Nonconvex-sigma_k, CANITA}.

\subsection{Gradient descent with compressed communication}

In this work we focus on algorithms based on the latter line of work: {\em communication compression}.

Perhaps conceptually the simplest yet versatile  gradient-based method for solving the distributed problem  \eqref{eq:finit_sum} employing communication compression is distributed compressed gradient descent (\algname{DCGD})  \cite{DCGD}. Given a sequence  $\{\gamma^t\}$ of learning rates, \algname{DCGD} performs the iterations \begin{equation}\label{eq:DCGD} \squeeze x^{t+1} = x^t - \gamma^t \frac{1}{n}\sum \limits_{i=1}^n g_i^t, \quad g_i^t = \cM_i^t(\nabla f_i(x^t)).\end{equation}
Above, $\cM_i^t$ represents {\em any} suitable gradient {\em communication mechanism}\footnote{\scriptsize We do not borrow the phrase ``communication mechanisms'' from any prior literature. We coined this phrase in order to be able to refer to a potentially arbitrary mechanism for transforming a $d$-dimensional gradient vector into another $d$-dimensional vector that is easier to communicate. This allows us to step back, and critically reassess the methodological foundations of the field in terms of the mathematical properties one should impart on such mechanisms for them to be effective.} for mapping the  possibly dense, high-dimensional, and hence hard-to-communicate gradient $\nabla f_i(x^t) \in \R^d$ into a vector of equal dimension, but one that can hopefully be communicated using much fewer bits.

\begin{table*}[t]
		\centering
		\scriptsize
		\caption{Summary of the methods fitting our general \algname{3PC} framework. For each method we give the formula for the \algname{3PC} compressor $\cC_{h,y}(x)$, its parameters $\thetaNEW$, $\betaNEW$, and the ratio $\nicefrac{\betaNEW}{\thetaNEW}$ appearing in the convergence rate. Notation: $\alpha$ = parameter of the contractive compressor $\cC$, $\omega$ = parameter of the unbiased compressor $\cQ$, $\thetaNEW_1, \betaNEW_1$ = parameters of three points compressor $\cC_{h,y}^1(x)$, $\bar\alpha = 1 - (1 - \al_1)(1 - \al_2)$, where $\alpha_1, \alpha_2$ are the parameters of the contractive compressors $\cC_1, \cC_2$, respectively.}
		\label{tab:methods}    
		\begin{threeparttable}
			\begin{tabular}{|c|c|c c c c|}
				\hline
				Variant of \algname{\tiny 3PC} (Alg.\ \ref{alg:3PC}) & Alg.\ \# & $\cC_{h,y}(x) =$ & $\thetaNEW$ & $\betaNEW$ & $\frac{\betaNEW}{\thetaNEW}$\\ 
				\hline
\hline
				\begin{tabular}{c}
				\algname{\tiny EF21}\\ \scriptsize{\citep{EF21}}
\end{tabular}				 & Alg.\ \ref{alg:b_diana} & $h + \cC(x - h)$ & $1 - \sqrt{1-\alpha}$ & $\frac{1-\alpha}{1 - \sqrt{1-\alpha}}$ & $\cO\left(\frac{1-\alpha}{\alpha^2}\right)$ \\ 				
\hline
				\begin{tabular}{c}
				\algname{\tiny LAG}\\ \citep{LAG} \tnote{{\color{blue}(3)}}
\end{tabular}				 & Alg.\ \ref{alg:lag} & \begin{tabular}{c}
	$\begin{cases} x,& \text{if } (\ast),\\ h,& \text{otherwise,} \end{cases}$\\ $(*)$ means $\|x- h\|^2 > \zeta \|x - y\|^2$
\end{tabular} & $1$ & $\zeta$ & $\cO\left(\zeta\right)$\\
\hline
				\begin{tabular}{c}
				\algname{\tiny CLAG}\\ \scriptsize{(NEW)}
\end{tabular}				 & Alg.\ \ref{alg:clag} & \begin{tabular}{c}
	$\begin{cases} h + \cC(x - h),& \text{if } (\ast),\\ h,& \text{otherwise,} \end{cases}$\\ $(*)$ means $\|x- h\|^2 > \zeta \|x - y\|^2$
\end{tabular} & $1 - \sqrt{1-\alpha}$ & $\max\left\{\frac{1-\alpha}{1 - \sqrt{1-\alpha}},\zeta\right\}$ & $\cO\left(\max\left\{\frac{1-\alpha}{\alpha^2}, \frac{\zeta}{\alpha}\right\}\right)$\\
\hline
				\begin{tabular}{c}
				\algname{\tiny 3PCv1}\\ \scriptsize{(NEW)}
\end{tabular}				 & Alg.\ \ref{alg:b_diana_2} & $y + \cC(x - y)$\tnote{{\color{blue}(1)}} & $1$ & $1-\alpha$ & $1-\alpha$\\ 
\hline
				\begin{tabular}{c}
				\algname{\tiny 3PCv2}\\ \scriptsize{(NEW)}
\end{tabular}				 & Alg.\ \ref{alg:anna} & \begin{tabular}{c}
	$b + \cC\left(x - b \right)$,\\ where $b = h + \cQ(x-y)$
\end{tabular} & $\alpha$ & $(1-\alpha)\omega$ & $\frac{(1-\alpha)\omega}{\alpha}$\\ 
\hline
				\begin{tabular}{c}
				\algname{\tiny 3PCv3}\\ \scriptsize{(NEW)}
\end{tabular}				 & Alg.\ \ref{alg:jeanne} & \begin{tabular}{c}
	$b + \cC\left(x - b \right)$,\\ where $b = \cC_{h,y}^1(x)$
\end{tabular} & $1 - (1 -\al)(1 -\thetaNEW_1)$  & $(1 - \alpha)\betaNEW_1$ & $\frac{(1 - \alpha)\betaNEW_1}{1 - (1 -\al)(1 -\thetaNEW_1)}$\\ 
\hline
				\begin{tabular}{c}
				\algname{\tiny 3PCv4}\\ \scriptsize{(NEW)}
\end{tabular}				 & Alg.\ \ref{alg:jacqueline} & \begin{tabular}{c}
	$b + \cC_1\left(x - b \right)$,\\ where $b = h + \cC_2(x-h)$
\end{tabular} & $1 - \sqrt{1-\bar\alpha}$ & $\frac{1-\bar\alpha}{1 - \sqrt{1-\bar\alpha}}$ & $\cO\left(\frac{1-\bar\alpha}{\bar\alpha^2}\right)$\\ 
				\hline				
				\begin{tabular}{c}
				\algname{\tiny 3PCv5}\\ \scriptsize{(NEW)}
\end{tabular}				 & Alg.\ \ref{alg:b_marina} & $\begin{cases}x,& \text{w.p.\ } p\\ h + \cC(x - y),& \text{w.p.\ } 1 - p \end{cases}$ & $1 - \sqrt{1-p}$ & $\frac{(1-p)(1-\alpha)}{1 - \sqrt{1-p}}$ & $\cO\left(\frac{(1-p)(1-\alpha)}{p^2}\right)$\\ 
				\hline
				\hline
				\begin{tabular}{c}
				\algname{\tiny MARINA}\\ \scriptsize{\citep{gorbunov2021marina}}
\end{tabular}				 & Alg.\ \ref{alg:marina} & N/A\tnote{{\color{blue}(2)}} & $p$ & $\frac{(1-p)\omega}{n}$ & $\frac{(1-p)\omega}{np}$\\
				\hline
			\end{tabular}
					\begin{tablenotes}
						\item [{\color{blue}(1)}] \algname{\tiny 3PCv1} requires communication of uncompressed vectors ($\nabla f_i(x^t)$). Therefore, the method is impractical. We include it as an idealized version of \algname{\tiny EF21}.						
						\item [{\color{blue}(2)}]  \algname{\tiny MARINA} does not fit the definition of three points compressor from \eqref{eq:ttp}. However, it satisfies \eqref{eq:key_inequality} with $G^t = \|g^t - \nabla f(x^t)\|^2$ and shown parameters $\thetaNEW$ and $\betaNEW$, i.e., \algname{\tiny MARINA} can be analyzed via our theoretical framework.
\item [{\color{blue}(3)}] \algname{\tiny LAG} presented in our work is a (massively) simplified version of \algname{\tiny LAG} considered by \cite{LAG}. However, we have decided to use the same name.							\end{tablenotes}
		\end{threeparttable}
	\end{table*}

\begin{table*}[t]
		\centering
		\scriptsize
		\caption{Comparison of exisiting and proposed theoretically-supported methods employing lazy aggregation. In the rates for our methods, $M_1 = L_{-} + L_{+}\sqrt{\nicefrac{\betaNEW}{\thetaNEW}}$ and $M_2 = \max\left\{L_{-} + L_{+}\sqrt{\nicefrac{2\betaNEW}{\thetaNEW}}, \nicefrac{\thetaNEW}{2\mu}\right\}$.}
		\label{tab:rates}    
		\begin{threeparttable}
			\begin{tabular}{| l | c c c c c |}
				\hline
				Method & \begin{tabular}{c} Simple \\ method? \end{tabular} & \begin{tabular}{c} Uses a contractive \\ compressor $\cC$? \end{tabular} & Strongly convex rate & {P\L} nonconvex  rate & General nonconvex rate  \\ 
				\hline
				\hline	
				\algname{\tiny LAG} \citep{LAG}            & \cmark &  \xmark    & linear  \tnote{\color{blue}(9)} & \xmark & \xmark \\ 							\algname{\tiny LAQ}  \citep{LAQ} & \xmark   & \cmark \tnote{\color{blue}(1)}  & linear  \tnote{\color{blue}(3)} & \xmark  & \xmark \\ 
				\algname{\tiny LENA}  \citep{LENA}   \tnote{\color{blue}(7)}     & \cmark \tnote{\color{blue}(4)} & \cmark \tnote{\color{blue}(8)}  & $\cO(G^4/T^2 \mu^2)$  \tnote{\color{blue}(5), (6)} & $\cO(G^4/T^2 \mu^2)$  \tnote{\color{blue}(5), (6)} & $\cO(G^{4/3}/T^{2/3})$ \tnote{\color{blue}(6)} \\ 				
\hline				
				\algname{\tiny LAG}  (NEW, 2022)       & \cmark & \xmark & $\cO(\exp(-T \mu /  M_2))$   & $\cO(\exp(-T \mu /  M_2))$  & $\cO(M_1/T)$ \\ 	
				\algname{\tiny CLAG}  (NEW, 2022)     & \cmark & \cmark \tnote{\color{blue}(2)}  & $\cO(\exp(-T \mu /  M_2))$   &$\cO(\exp(-T \mu /  M_2))$   & $\cO(M_1/T)$ \\ 	
				\hline
			\end{tabular}
					\begin{tablenotes}
						\item [{\color{blue}(1)}] They consider a specific form of quantization only.
						\item [{\color{blue}(2)}] Works with any contractive compressor, including low rank approximation, Top-$K$, Rand-$K$, quantization,  and more.
						\item [{\color{blue}(3)}] Their Theorem 1 does not present any {\em explicit} linear rate. 						\item [{\color{blue}(4)}] \algname{\tiny LENA} employs the classical \algname{\tiny EF} mechanism, but it is not clear what is this mechanism supposed to do.	
		\item [{\color{blue}(5)}] They consider an assumption ($\mu$-quasi-strong convexity) that is slightly stronger							than our {P\L} assumption. Both are weaker than strong convexity.
		\item [{\color{blue}(6)}] 
			They assume the local gradients to be bounded by $G$ ($\|\nabla f_i(x)\| \leq G$ for all $x$). We do not need such a strong assumption.
				\item [{\color{blue}(7)}] 
			They also consider the $0$-quasi-strong convex case (slight generalization of convexity); we do not consider the convex case. Moreover, they consider the stochastic case as well, we do not. We specialized all their results to the deterministic (i.e., full gradient) case for the purposes of this table.
				\item [{\color{blue}(8)}] Their contractive compressor depends on the trigger.
			\item [{\color{blue}(9)}] It is possible to specialize their method and proof so as to recover \algname{\tiny LAG} as presented in our work, and to recover a rate similar to ours.			
					\end{tablenotes}
		\end{threeparttable}
	\end{table*}

\section{Motivation and Background}
\label{sec:mot_and_back}

Our work is motivated by several methodological, theoretical and algorithmic issues and open problems arising in the literature related to two orthogonal approaches to designing  gradient {\em communication mechanisms} $\cM_i^t$: \begin{itemize} 
\item [i)]  {\em contractive compressors} \cite{Karimireddy_SignSGD,Stich-EF-NIPS2018,Alistarh-EF-NIPS2018,Koloskova2019DecentralizedDL,beznosikov2020biased}, and   
\item [ii)] {\em lazy aggregation} \cite{LAG,LAQ,LENA}.
\end{itemize}

The motivation for our work starts with several critical observations related to these two mechanisms.

\subsection{Contractive compression operators} \label{sub:contractive}

Arguably, the simplest class of communication mechanisms is based on the (as we shall see, naive) application of {\em contractive  compression operators} (or, {\em contractive compressors} for short) \cite{Koloskova2019DecentralizedDL, beznosikov2020biased}. In this approach, one sets
\begin{equation}\label{eq:b9fd-98ybhkfd}\cM_i^t(x) \equiv \cC(x),\end{equation}
where $\cC:\R^d \to \R^d$ is a (possibly randomized) mapping with the property 
\begin{equation}\label{eq:contractive-09u09fduf}\Exp{\norm{\cC(x)-x}^2} \leq (1- \lambdaNEW) \norm{x}^2, \quad \forall x\in \R^d,\end{equation}
where $0<\lambdaNEW\leq 1$ is the {\em contraction parameter}, and the expectation $\Exp{\cdot}$ is taken w.r.t.\ the randomness inherent in $\cC$.  For examples of contractive compressors (e.g., Top-$K$ and Rand-$K$ sparsifiers), please refer to Section~\ref{sec:contractive}, and Table~1 in \citep{UP2021, beznosikov2020biased}.

The algorithmic literature on {\em contractive compressors} (i.e., mappings $\cC$ satisfying \eqref{eq:contractive-09u09fduf}) is relatively much more developed, and dates back to at least 2014 with the work of \citet{Seide2014}, who proposed the {\em error feedback} (\algname{EF}) mechanism for fixing certain divergence issues which  arise empirically with the naive approach based on \eqref{eq:b9fd-98ybhkfd}. 

Despite several advances in our theoretical understanding of \algname{EF} over the last few years \cite{Stich-EF-NIPS2018,Karimireddy_SignSGD,A_better_alternative,DoubleSqueeze,Lin_EC_SGD}, a satisfactory grasp of \algname{EF} remained elusive.  Recently, \citet{EF21} proposed \algname{EF21}, which is a new algorithmic and analysis approach to error feedback, effectively fixing the previous weaknesses. In particular, while 
previous results offered weak $\cO(1/T^{2/3})$ rates (for smooth nonconvex problems), and did so under strong and often unrealistic assumptions (e.g., boundedness of the gradients), the \algname{EF21} approach offers \algname{GD}-like $\cO(1/T)$ rates, with standard assumptions only.\footnote{\scriptsize The \algname{\tiny EF21} method was extended by \citet{EF21BW} to deal with stochastic gradients, variance reduction, regularizers, momentum, server compression, and partial participation. However, such extensions are not the subject of our work.}

The heart of the \algname{EF21} method is a new communication mechanism $\cM_i^t$, generated from a  contractive compressor $\cC$, which fixes (in a theoretically and practically superior way to the standard fix offered by classical \algname{EF}) the above mentioned divergence issues. Their construction is {\em synthetic}: is starts with the choice of $\cC$ preferred by the user, and then constructs a new and adaptive communication mechanism based on it. We will describe this method in Section~\ref{sec:3PC}.

\subsection{Lazy aggregation}

An orthogonal approach to applying contractive  operators, whether with or without error feedback, is ``skipping'' communication. The basic idea of the {\em lazy aggregation} communication mechanism is for each worker $i$ to communicate its local gradient only if it differs ``significantly'' from the last gradient communicated before. 

In its simplest form, the \algname{LAG} method of \citet{LAG} is initialized with $g_i^0 = \nabla f_i(x^0)$ for all $i\in [n]$, which means that all the workers communicate their gradients at the start. In all subsequent iterations, each worker $i\in [n]$ defines $g_i^{t+1}$, which may be  interpreted as a ``compressed'' version of the true gradient $\nabla f_i(x^{t+1})$, via the  {\em lazy aggregation} rule
\begin{equation}\label{eq:LAG}g_i^{t+1}= \begin{cases}
				\nabla f_i(x^{t+1}) & \text{if } \|g_i^t - \nabla f_i(x^{t + 1})\|^2 > \zeta D_i^t, \\
				g_i^t  & \text{otherwise,}
		\end{cases}\end{equation}
where $D_i^t \eqdef \|\nabla f_i(x^{t+1}) - \nabla f_i(x^t)\|^2$ and $\zeta>0$ is the {\em trigger}.\footnote{\scriptsize It is possible to replace $D_i^t$ by $X_i^t = \zeta L_i^2 \|x^{t+1} - x^t\|^2$, and our theory will still trivially hold. This is the choice for the trigger condition made by \citet{LAG}. One can also work with the more general choice $X_i^t = \zeta_i \|x^{t+1} - x^t\|^2$; our theory can be adapted to this trivially.} The smaller the trigger $\zeta$, the more likely it is for the condition $\|g_i^t - \nabla f_i(x^{t + 1})\|^2 > \zeta D_i^t$, which triggers communication, to be satisfied. On the other hand, if $\zeta$ is very large, most iterations will skip communication and thus reuse the past gradient. Since the trigger fires dynamically based on conditions that change in time, it is hard to theoretically estimate how often communication skipping occurs. In fact, there are no results on this in the literature. Nevertheless, the lazy aggregation mechanism is empirically useful when compared to vanilla \algname{GD} \citep{LAG}.

Lazy aggregation is a much less studied and a much less understood communication mechanism than  contractive compressors. Indeed, only a handful of papers offer any convergence guarantees \citep{LAG, LAQ, LENA}, and the results presented in the first two of these papers are hard to penetrate. For example, no simple proof exists for the simple \algname{LAG} variant presented above. The best known rate in the smooth nonconvex regime is $\cO(1/T^{2/3})$, which differs from the $\cO(1/T)$ rate of \algname{GD}. The known rates in the strongly convex regime are also highly problematic: they are either not explicit \citep{LAG, LAQ}, or sublinear \citep{LENA}. Furthermore, it is not clear whether an \algname{EF} mechanism is needed to stabilize lazy aggregation methods, which is a necessity in the case of contractive compressors. While \citet{LENA} proposed a combination of \algname{LAG} and \algname{EF}, their analysis leads to weak rates (see Table~\ref{tab:rates}), and does not seem to point to  theoretical advantages due to  \algname{EF}.

\section{Summary of Contributions}
\label{sec:contributions}

We now summarize our main contributions:

{$\bullet$ \bf Unification through the \algname{3PC} method.} At present, the two communication mechanisms outlined above, {\em contractive compressors}  and    {\em lazy aggregation}, are viewed as different approaches to the same problem---reducing the communication overhead in distributed gradient-type methods---requiring  different tools, and facing different theoretical challenges. 
 We propose a {\em unified method}---which we call \algname{3PC} (Algorithm~\ref{alg:3PC})---which includes \algname{EF21} (Algorithm~\ref{alg:b_diana})  and \algname{LAG} (Algorithm~\ref{alg:lag}) as special cases. 

{$\bullet$ \bf Several new methods.} The  \algname{3PC} method is much more general than either \algname{EF21} or \algname{LAG}, and includes a number of new specific methods. For example, we propose \algname{CLAG}, which is a combination of \algname{EF21} and \algname{LAG} benefiting from both contractive compressors and lazy aggregation.  We show experimentally that \algname{CLAG} can be better than both \algname{EF21} and \algname{LAG}: that is, we obtain combined benefits of both approaches. We obtain a number of other new methods, such as \algname{3PCv2}, \algname{3PCv3} and \algname{3PCv4}.  We show experimentally that \algname{3PCv2} can outperform \algname{EF21}. See Table~\ref{tab:methods} for a summary of the proposed methods.

{$\bullet$ \bf Three point compressors.} Our proposed method, \algname{3PC}, can be viewed as \algname{DCGD} with a {\em new class of communication mechanisms}, based on the new notion of a {\em three point compressor} (\algname{3PC})\footnote{\scriptsize We use the same name for the method and the compressor on purpose.}; see Section~\ref{sec:3PC} for  details.  By design, and in contrast to  contractive compressors,  our communication mechanism based on the \algname{3PC} compressor is able to ``evolve'' and thus improve throughout the iterations. In particular, its {\em compression error decays}, which is the key reason behind its superior theoretical properties. In summary, the properties defining the \algname{3PC} compressor distill the important characteristics of a theoretically well performing communication mechanism, and this is the first time such characteristics have been explicitly identified and formalized.

The observation that lazy aggregation is a \algname{3PC} compressor explains why error feedback is {\em not} needed to stabilize \algname{LAG} and similar methods.

{$\bullet$ \bf Strong rates.} We prove an $\cO(1/T)$ rate for \algname{3PC} for smooth nonconvex problems, which up to constants matches the rate of \algname{GD}. Furthermore, we prove a \algname{GD}-like linear convergence rate under the Polyak-{\L}ojasiewicz condition.  Our general theory recovers the \algname{EF21} rates proved by \citet{EF21} exactly. Our rates for lazily aggregated methods (\algname{LAG} and \algname{CLAG}) are new, and better than the results obtained by \citet{LAG, LAQ} and \citet{LENA} in all regimes considered. In the general smooth nonconvex regime, only \citet{LENA} obtain rates. However, they require strong assumptions (gradients bounded by a constant $G$), and their rate is $\cO(G^{4/3}/T^{2/3})$, whereas we do not need such assumptions and obtain the \algname{GD}-like rate $\cO(M_1/T)$. In the   strongly convex regime, \citet{LAG} and \citet{LAQ} obtain non-specific linear rates, while \citet{LENA} obtain the  sublinear rate $\cO(G^4/ T^2 \mu^2)$. In contrast, we obtain explicit  \algname{GD}-like linear rates under the weaker {P\L} condition. 

Furthermore, our variant of \algname{LAG}, and our convergence theory and proofs, are much simpler than those presented in \citep{LAG}. In fact, it is not clear to us whether the many additional features employed by \citet{LAG} have any theoretical or practical benefits. We believe that our simple treatment can be useful for other researchers to further advance the field. 

For a detailed comparison of rates, please refer to  Table~\ref{tab:rates}.

\begin{algorithm*}[t]
   \caption{\algname{3PC} (\algname{DCGD} method using the \algname{3PC} communication mechanism)}\label{alg:3PC}
\begin{algorithmic}[1]
   \STATE {\bfseries Input:} starting point $x^0\in \R^d$ (on all workers), stepsize $\gamma>0$, number of iterations $T$, starting vectors $g_i^0\in \R^d$ for $i \in [n]$ (known to the server and all workers)
   \STATE {\bfseries Initialization:}  $g^0 = \frac{1}{n}\sum_{i=1}^n g_i^0$ \hfill (Server aggregates initial gradient estimates)
   \FOR{$t=0, 1, \ldots, T-1$}
   \STATE Broadcast $g^t$ to all workers      
   \FOR{$i = 1, \ldots, n$ in parallel} 
   \STATE $x^{t+1} = x^t - \gamma g^t$ \hfill (Take a gradient-type step)
   \STATE {\color{red}Set $g_i^{t+1} = \cM_i^{t+1}(\nabla f_i(x^{t+1})) \eqdef \cC_{g_i^t, \nabla f_i(x^t)}(\nabla f_i(x^{t+1}))$  \hfill (Apply \algname{3PC} to compress the latest gradient)}
   \STATE Communicate $g_i^{t+1}$ to the server    
   \ENDFOR
   \STATE Server aggregates received messages: $g^{t+1} = \tfrac{1}{n}\sum_{i=1}^n g_i^{t+1} $ 
   \ENDFOR
   \STATE {\bfseries Return:} $\hat x^T$ chosen uniformly at random from $\{x^t\}_{t=0}^{T-1}$
\end{algorithmic}
\end{algorithm*}

\section{Three Point Compressors}
\label{sec:3PC}

We now formally introduce the concept of a {\em three point compressor} (\algname{3PC}).

\begin{definition}[Three point compressor]\label{def:ttp}
	We say that a (possibly randomized) map $$\cC_{h,y}(x): \underbrace{\R^d}_{h\in} \times \underbrace{\R^d}_{y\in} \times \underbrace{\R^d}_{x\in} \rightarrow \R^d$$ is a 
	three point compressor (\algname{3PC}) if there exist constants $0<\thetaNEW \le 1$ and $\betaNEW \ge 0$ such that the following relation holds for all $x, y, h \in \R^d$
	\begin{eqnarray}
		\Exp{\sqnorm{\cC_{h,y}(x) - x }} & \leq & (1 - \thetaNEW)\sqnorm{h - y}  \notag \\
		&& \quad + \betaNEW \sqnorm{x-y}.\label{eq:ttp}
	\end{eqnarray}
\end{definition}

The vectors $y \in \R^d$ and $h \in \R^d$ are parameters defining the compressor. Once fixed, $\cC_{h,y}:\R^d \to \R^d$ is the  compression mapping used to compress vector $x\in \R^d$. 

\subsection{Connection with contractive compressors} Note that if we set $h=0$ and $y=x$, then inequality \eqref{eq:ttp} specializes to 
 	\begin{eqnarray}
		\Exp{\sqnorm{\cC_{0,x}(x) - x }} & \leq & (1 - \thetaNEW)\sqnorm{x} ,
	\end{eqnarray}
which is the inequality defining a contractive compressor. In other words, a particular {\em restriction} of the parameters of any \algname{3PC} compressor is necessarily a contractive compressor. 

However, this is not the restriction we will use to design our compression mechanism. Instead, as we shall describe next, we will choose the sequence of vectors $h$ and $y$ in an adaptive fashion, based on the path generated by \algname{DCGD}.

\subsection{Designing a communication mechanism using a \algname{3PC} compressor}

We now describe our proposal for how to use a \algname{3PC} compressor to design a good communication mechanism $\{\cM_{i}^k\}$ to be used within \algname{DCGD}. Recall from \eqref{eq:DCGD} that all we need to do is to define the mapping 
\[   \cM_i^t : \nabla f_i(x^t) \mapsto g_i^t.\] 

First, we allow the initial compressed gradients $\{g_i^0\}_{i=1}^n$ to be chosen arbitrarily. Here are some examples of possible choices: 
{\bf a) Full gradients:} $g_i^0 = \nabla f_i(x^0)$ for all $i\in [n]$. The benefit of this choice is that no information is loss at the start of the process. On the other hand, the full $d$-dimensional gradients need to be sent by the workers to the server, which is potentially an expensive pre-processing step.
 {\bf b) Compressed gradients:} $g_i^0 = \cC(\nabla f_i(x^0))$ for all $i\in [n]$, where $\cC$ is an arbitrary compression mapping (e.g., a contractive compressor). While some information is lost right at the start of the process (compared to a \algname{GD} step), the benefit of this choice is that no full dimensional vectors need to be communicated.
 {\bf c) Zero preprocessing:} $g_i^0 = 0$  for all $i\in [n]$.

Having chosen $g_i^0, \dots, g_i^0$ for all $i\in [n]$, it remains to define the communication mechanism $\cM_i^t$ for $t\geq 1$. We will do this on-the-fly as \algname{DCGD} is run, with the help of the parameters $h$ and $y$, which we choose adaptively.    Consider the viewpoint of a worker $i\in [n]$ in iteration $t+1$, with $t\geq 0$. In this iteration, worker $i$ wishes to compress the vector  $x = \nabla f_i(x^{t+1})$. Let $g_i^t$ denote the compressed version of the vector $\nabla f_i(x^{t})$, i.e., $g_i^t = \cM_i^t(\nabla f_i(x^{t}))$. We choose $$y = \nabla f_i(x^t) \qquad  \text{and} \qquad h = g_i^t.$$ 
With these parameter choices, we define the compressed version of $x = \nabla f_i(x^{t+1})$ by setting \begin{equation}\label{eq:CM-987943}\cM_i^{t+1}(\nabla f_i(x^{t+1})) \overset{\eqref{eq:DCGD}}{=} g_i^{t+1} \eqdef \cC_{g_i^t, \nabla f_i(x^t)}(\nabla f_i(x^{t+1})).\end{equation}

Our porposed \algname{3PC} method (Algorithm~\ref{alg:3PC}) is just \algname{DCGD} with the compression mechanism described above.

\subsection{The \algname{3PC} inequality}

For the parameter choices made above, \eqref{eq:ttp} specializes to
 	\begin{eqnarray}
		\Exp{E_i^{t+1} \;|\; x^t, g_i^t} \leq (1 - \thetaNEW) E_i^t + \betaNEW D_i^t,\label{eq:ttp-specific}
	\end{eqnarray}
	where $$E_i^t \eqdef \sqnorm{g_i^{t} -  \nabla f_i(x^{t}) }$$ and $$D_i^t \eqdef  \sqnorm{\nabla f_i(x^{t+1})-\nabla f_i(x^t)}.$$ This inequality has a natural interpretation. It enforces the compression error $E_i^{t}$ to shrink by the factor of $1-\thetaNEW$ in each communication round, subject to an additive penalty proportional to $D_i^t$. If the iterates converge, then the penalty will eventually vanish as well provided that the gradient of $f_i$ is continuous. Intuitively speaking, this forces the compression error $E^t$ to improve in time. 
	
	We note that applying a simple contractive compressor in place of $\cM_i^t$ does not have this favorable property, and this is what causes the convergence issues in existing literature on this topic. This is what the \algname{EF} literature was trying to solve since 2014, and what the \algname{EF21} mechanism resolved in 2021. However, no such progress happened in the lazy aggregation literature yet, and one of the key contributions of our work is to remedy this situation.

\subsection{\algname{EF21} mechanism is a \algname{3PC} compressor}

We will now show that the compression mechanism $\cM_i^t$ employed in  \algname{EF21} comes from a \algname{3PC} compressor. Let $\cC:\R^d\to \R^d$ be a contractive compressor with contraction parameter $\alpha$, and define
	\begin{eqnarray}
	\cC_{h,y}(x) \eqdef h + \cC(x - h).  \label{eq:---=08}
	\end{eqnarray}
If we use this mapping to define a compression mechanism $\cM_i^t$ via \eqref{eq:CM-987943}, use this within \algname{DCGD}, 
we obtain the \algname{EF21} method of \citet{EF21}. Indeed, observe that Algorithm~\ref{alg:b_diana} (\algname{EF21}) is a special case of Algorithm~\ref{alg:3PC} (\algname{3PC}).

The next lemma shows that \eqref{eq:---=08} is a \algname{3PC} compressor.

\begin{lemma}\label{lem:biased_diana}
	The mapping \eqref{eq:---=08}
	satisfies \eqref{eq:ttp} with $\thetaNEW \eqdef 1 - (1-\alpha)(1+s)$ and $\betaNEW \eqdef (1-\alpha)\rb{1+ s^{-1} }$, where $s> 0$ is any scalar satisfying  $(1-\alpha)\rb{1 + s} < 1$.
\end{lemma}

\subsection{\algname{LAG} mechanism is a \algname{3PC} compressor}

We will now show that the compression mechanism $\cM_i^t$ employed in \algname{LAG} comes from a \algname{3PC} compressor. In fact, let us define \algname{CLAG}, and recover \algname{LAG} from it as a special case. Let $\cC:\R^d\to \R^d$ be a contractive compressor with contraction parameter $\alpha$. Choose a trigger $\zeta>0$, and define
	\begin{eqnarray}
	\cC_{h,y}(x) \eqdef \begin{cases} h + \cC(x - h),& \text{if } \|x- h\|^2 > \zeta \|x - y\|^2,\\ h,& \text{otherwise,} \end{cases}  \label{eq:clag_98f9d8}
	\end{eqnarray}

If we use this mapping to define a compression mechanism $\cM_i^t$ via \eqref{eq:CM-987943}, use this within \algname{DCGD}, 
we obtain our new \algname{CLAG} method. Indeed, observe that Algorithm~\ref{alg:clag} (\algname{CLAG}) is a special case of Algorithm~\ref{alg:3PC} (\algname{3PC}). 

The next lemma shows that \eqref{eq:clag_98f9d8} is a \algname{3PC} compressor.

\begin{lemma}\label{lem:clag}
The mapping \eqref{eq:clag_98f9d8} satisfies \eqref{eq:ttp} with $\thetaNEW \eqdef 1 - (1-\alpha)(1+s)$ and $\betaNEW \eqdef \max\left\{(1-\alpha)\rb{1+ s^{-1} }, \zeta\right\}$, where $s> 0$ is any scalar satisfying  $(1-\alpha)\rb{1 + s} < 1$.
\end{lemma}

The \algname{LAG} method is obtained as a special case of \algname{CLAG} by choosing $\cC$ to be the identity mapping (for which $\alpha=1$).

\subsection{Further \algname{3PC} compressors and methods}

In Table~\ref{tab:methods}  we summarize several further \algname{3PC} compressors and the new algorithms they lead to (e.g., \algname{3PCv1}--\algname{3PCv5}). The details are given the in the appendix. 

\section{Theory}
\label{sec:theory}
	
We are now ready to present our theoretical convergence results for the \algname{3PC} method (Algorithm~\ref{alg:3PC}), the main steps of which are 
\begin{equation}
\squeeze	x^{t+1} = x^t - \gamma g^t,\quad g^t = \frac{1}{n}\sum\limits_{i=1}^n g_i^t, \label{eq:CGD_1}
\end{equation}
\begin{equation}
	g_i^{t+1} = \cC_{g_i^t, \nabla f_i(x^t)}(\nabla f_i(x^{t+1})).\label{eq:CGD_2}
\end{equation}

Recall that the \algname{3PC} method is  \algname{DCGD} with a particular choice of the communication mechanism $\{\cM_i^t\}$ based on an arbitrary \algname{3PC} compressor $\cC_{h,y}(x)$.

\subsection{Assumptions}

We rely on the following standard assumptions.

\begin{assumption} \label{ass:diff} The functions $f_1,\dots,f_n: \R^d\to \R$ are differentiable. Moreover, $f$ is lower bounded, i.e., there exists $f^{\inf} \in \R$ such that $ f(x) \geq f^{\inf}$ for all $x \in \R^d$. 
	\end{assumption}
	
	\begin{assumption}\label{as:L_smoothness} The function $f:\R^d \to \R$ is $L_{-}$-smooth, i.e., it is differentibale and its gradient satisfies
	\begin{equation}
		\|\nabla f(x) - \nabla f(y)\| \leq L_{-}\|x-y\|\quad \forall x,y\in \R^d. \label{eq:L_smoothness}
	\end{equation}
	\end{assumption}
	
	\begin{assumption} \label{as:L_+}There is a constant $L_+>0$ such that $\frac{1}{n} \sum_{i=1}^n \norm{\nabla f_i(x) - \nabla f_i(y)}^2 \leq L^2_{+} \norm{x-y}^2$ for all $x,y\in \R^d$.  Let $L_+$ be the smallest such number.
	\end{assumption}
		It is easy to see that $L_-\leq L_+$. We borrow this notation for the smoothness constants from \citep{PermK}.

\subsection{Convergence for general nonconvex functions}\label{sec:gen_non_cvx}

The following lemma is based on the properties of the \algname{3PC} compressor. It establishes the key inequality for the convergence analysis. The proof follows easily from the definition of a \algname{3PC} compressor and Assumption~\ref{as:L_+}.

\begin{lemma}\label{lem:key_lemma}
	Let Assumption~\ref{as:L_+} hold. Consider the \algname{3PC} method. Then, the sequence
	\begin{equation}\label{eq:G^t} 
	\squeeze G^t \eqdef \frac{1}{n} \sum \limits_{i=1}^n \sqnorm{ g_i^t-\nabla f_i(x^{t}) }
	\end{equation}
	for all $t\ge 0$  satisfies
	\begin{eqnarray}
	\squeeze	\Exp{G^{t+1}} \leq (1-\thetaNEW)\Exp{G^t} + \betaNEW L_{+}^2 \Exp{\sqnorm{x^{t+1} - x^t}}. \label{eq:key_inequality}
	\end{eqnarray}
\end{lemma}

Using this lemma and arguments from the analysis of \algname{SGD} for non-convex problems \citep{PAGE, EF21}, we derive the following result.

\begin{theorem}\label{thm:main_thm_gen_non_cvx}
	Let Assumptions~\ref{ass:diff},~\ref{as:L_smoothness},~\ref{as:L_+} hold. Assume that the stepsize $\gamma$ of the \algname{3PC} method satisfies $0 \leq \gamma \leq \nicefrac{1}{M_1}$, where $M_1 = L_{-} + L_{+}\sqrt{\nicefrac{\betaNEW}{\thetaNEW}}$. Then, for any $T \ge 1$ we have
	\begin{equation}
\squeeze		\Exp{\norm{ \nabla f(\hat x^T)}^2} \leq \frac{2 \Delta^0}{\gamma T} + \fr{\Exp{G^0}}{\thetaNEW T}, \label{eq:main_thm_gen_non_cvx}
	\end{equation}
	where $\hat x^T$ is sampled uniformly at random from the points $\{x^0, x^1, \ldots, x^{T-1}\}$ produced by \algname{3PC} , $\Delta^0 \eqdef f(x^0) - f^{\inf}$, and $G^0$ is defined in \eqref{eq:G^t}.
\end{theorem}

The theorem implies the following fact.
\begin{corollary}\label{cor:main_cor_gen_non_cvx}
	Let the assumptions of Theorem~\ref{thm:main_thm_gen_non_cvx} hold and choose the stepsize
	\begin{equation*}
	\squeeze	\gamma = \frac{1}{L_{-} + L_{+}\sqrt{\nicefrac{\betaNEW}{\thetaNEW}}}.
	\end{equation*}
	Then for any $T\geq 1$ we have
	\begin{equation}
	\squeeze	\Exp{\norm{ \nabla f(\hat x^T)}^2} \leq \frac{2 \Delta^0\left(L_{-} + L_{+}\sqrt{\nicefrac{\betaNEW}{\thetaNEW}}\right)}{ T} + \fr{\Exp{G^0}}{\thetaNEW T}.\notag
	\end{equation}
	That is, to achieve $\Exp{\norm{ \nabla f(\hat x^T)}^2} \leq \varepsilon^2$ for some $\varepsilon > 0$, the \algname{3PC} method requires
	\begin{equation}
	\squeeze	T = \cO\left(\frac{\Delta^0\left(L_{-} + L_{+}\sqrt{\nicefrac{\betaNEW}{\thetaNEW}}\right)}{\varepsilon^2} + \fr{\Exp{G^0}}{\thetaNEW \varepsilon^2}\right) \label{eq:main_complexity_gen_non_cvx}
	\end{equation}
	iterations (=communication rounds).
\end{corollary}

\subsection{Convergence under the {P\L} condition}\label{sec:PL}
In this part we provide our main convergence result under the Polyak-{\L}ojasiewicz (P\L) condition.

\begin{assumption}[P{\L} condition]\label{as:PL}
Function $f:\R^d \to \R$ satisfies the Polyak-{\L}ojasiewicz (P\L) condition with parameter $\mu > 0$, i.e.,
\begin{equation}
	\|\nabla f(x)\|^2 \geq 2\mu\left(f(x) - f^*\right), \quad \forall x\in \R^d, \label{eq:PL}
\end{equation}
where $x^* \eqdef \argmin_{x\in\R^d} f(x)$ and $f^*\eqdef f(x^*)$.
\end{assumption}

In this setting, we get the following result.

\begin{theorem}\label{thm:main_thm_PL}
	Let Assumptions~\ref{ass:diff},~\ref{as:L_smoothness},~\ref{as:L_+},~\ref{as:PL} hold. Assume that the stepsize $\gamma$ of the \algname{3PC} method satisfies $0 \leq \gamma \leq \nicefrac{1}{M_2}$, where $M_2 = \max\left\{L_{-} + L_{+}\sqrt{\nicefrac{2\betaNEW}{\thetaNEW}}, \nicefrac{\thetaNEW}{2\mu}\right\}$. Then, for any $T \ge 0$ and $\Delta^0 \eqdef f(x^0) - f(x^*)$ we have
	\begin{equation}
	\squeeze	\Exp{f(x^T)} - f^* \leq \left(1 - \gamma\mu\right)^T\left(\Delta^0 + \frac{\gamma}{\thetaNEW}\Exp{G^{0}}\right). \label{eq:main_thm_PL}
	\end{equation}
\end{theorem}

The theorem implies the following fact.
\begin{corollary}\label{cor:main_cor_PL}
	Let the assumptions of Theorem~\ref{thm:main_thm_PL} hold and choose the stepsize
	\begin{equation*}
	\squeeze	\gamma = \min\left\{\frac{1}{L_{-} + L_{+}\sqrt{\nicefrac{2\betaNEW}{\thetaNEW}}}, \frac{\thetaNEW}{2\mu}\right\}.
	\end{equation*}
	Then to achieve $\Exp{f(x^T)} - f^* \leq \varepsilon$ for some $\varepsilon > 0$ the method requires
	\begin{equation}
\squeeze		\cO\left(\max\left\{\frac{L_{-} + L_{+}\sqrt{\nicefrac{\betaNEW}{\thetaNEW}}}{\mu}, \thetaNEW\right\}\log \frac{\Delta^0 + \Exp{G^{0}}\nicefrac{\gamma}{\thetaNEW}}{\varepsilon}\right) \label{eq:main_complexity_PL}
	\end{equation}
	iterations (=communication rounds).
\end{corollary}

\subsection{Commentary}
As mentioned in the contributions section, the above rates match those of \algname{GD} in the considered regimes, up to constants factors, and at present constitute the new best-known rates for methods based on lazy aggregation. They recover the best known rates for error feedback since they match the rate of \algname{EF21}.

\section{Experiments}
\label{sec:experiments}
Now we empirically test the new variants of \algname{3PC} in two expriments. In the first experiment, we focus on compressed lazy aggregation mechanism and study the behavior of \algname{CLAG} (\cref{alg:clag}) combined with Top-$K$ compressor. In the second one, we compare \algname{3PCv2} (\cref{alg:anna}) to \algname{EF21} with Top-$K$ on a practical task of learning a representation of MNIST dataset \cite{LeCun_MNIST}. 
\begin{figure*}[h]
	\centering
	\includegraphics[width=\textwidth]{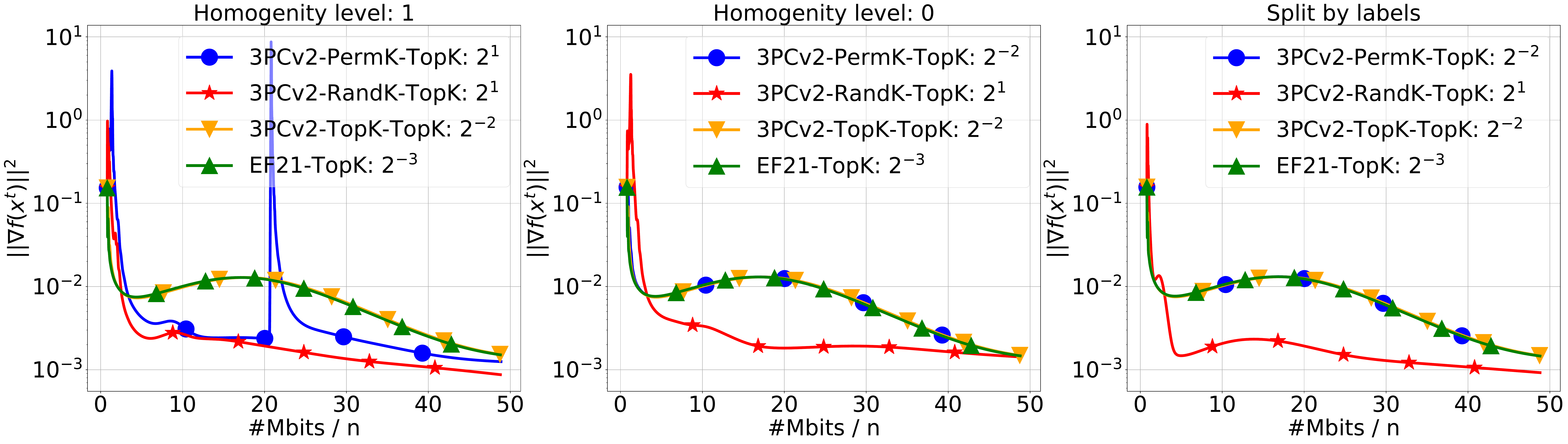}
	\caption{Comparison of \algname{3PCv2} with Perm-$K$, Rand-$K$ and Top-$K$  as the first compressor. Top-$K$ is used as the second compressor. Number of clients $n = 100$, compression level $K = 251$. \algname{EF21} with Top-$K$ is provided for the reference.}
	\label{fig:anna-100-nodes-grads_main}
\end{figure*}
\subsection{Is \algname{CLAG} better than \algname{LAG} and \algname{EF21}?}
Consider solving the non-convex logistic regression problem
\begin{equation*}
\squeeze	\min\limits_{x \in \R^d} \left[ f(x) \eqdef  \frac{1}{N}\sum\limits_{i=1}^N \log(1 + e^{-y_i a_i^\top x}) + \lambda \sum\limits_{j=1}^d \frac{x_j^2}{1 + x_j^2}\right],
\end{equation*}
where $a_i \in \R^d$, $y_i \in \{-1, 1\}$ are the training data and labels, and $\lambda > 0$ is a regularization parameter, which is fixed to $\lambda =0.1$. We use four LIBSVM~\cite{chang2011libsvm} datasets \emph{phishing, w6a, a9a, ijcnn1} as training data. Each dataset is shuffled and split into $n=20$ equal parts.

We vary two parameters of \algname{CLAG}, $K$ and $\zeta$, and report the number of bits (per worker) sent from clients to the server to achieve $\|\nabla f(x^t)\| < 10^{-2}$. For each pair ($K$, $\zeta$), we fine-tune the stepsize of \algname{CLAG} with multiples ($1, 2^1, 2^2, \dots, 2^{11}$) of the theoretical stepsize. We report the results on a heatmap (see \cref{fig:heatmapreducedclagijcnn1}) for the representative dataset \emph{ijcnn1}. Other datasets are included in \cref{sec:extra_experiments}. On the heatmap, we vary $\zeta$ along rows and $K$ along columns. Notice that \algname{CLAG} reduces to \algname{LAG} when $K=d$ (bottom row) and to \algname{EF21} when $\zeta = 0$ (left column). 

The experiment shows that the minimum communication complexity is attained at a combination of $(K, \zeta)$, which does not reduce \algname{CLAG} to its special cases: \algname{EF21} or \algname{LAG}. This empirically confirms that \algname{CLAG} \textit{has better communication complexity than \algname{EF21} and \algname{LAG}.}
\begin{figure}[h]
	\centering
	\includegraphics[width=1.1\linewidth]{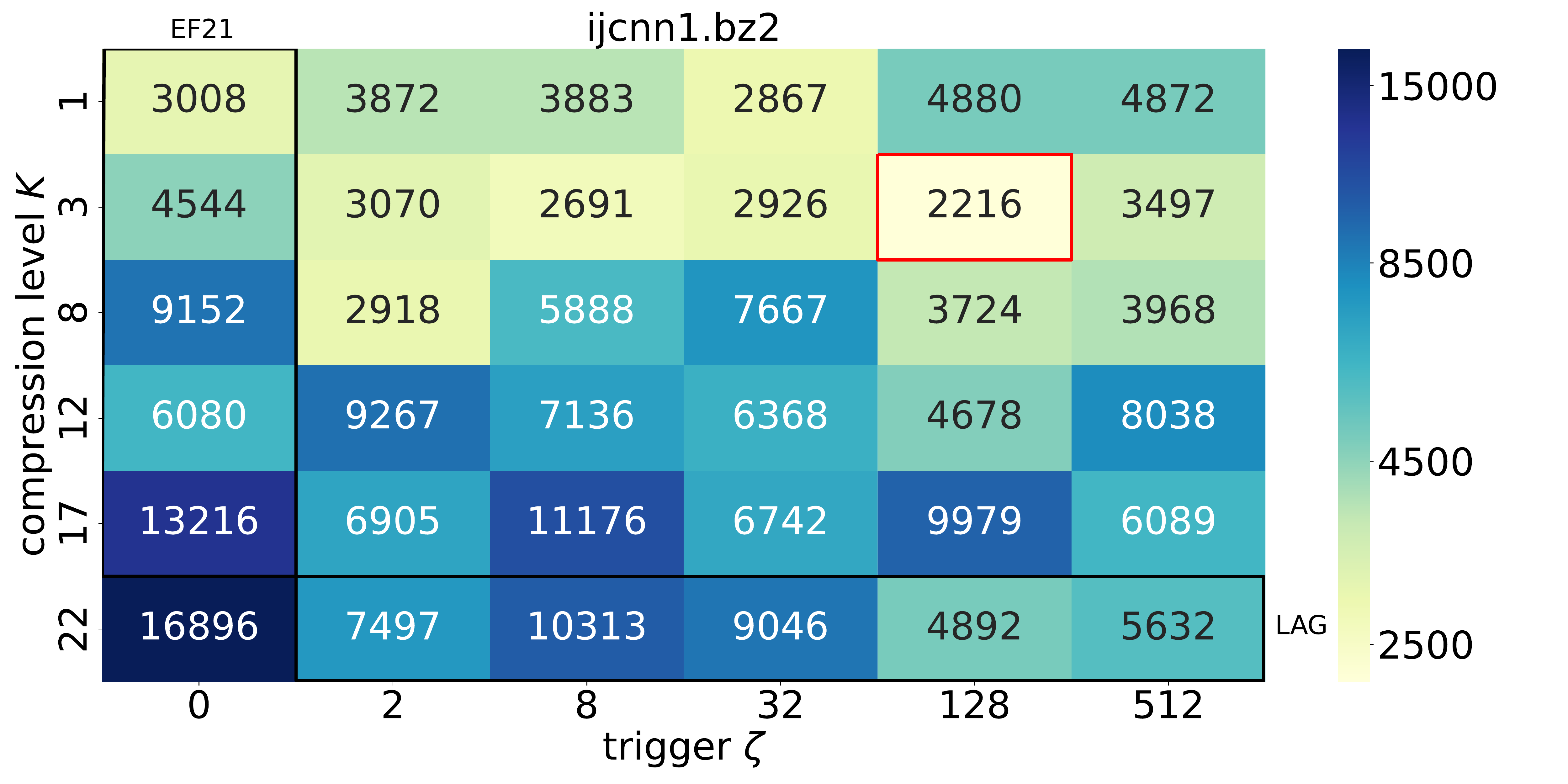}
	\caption{Heatmap of communication complexities of \algname{CLAG} for different combination of compression levels $K$ and triggers $\zeta$ with tuned stepsizes on {\em ijcnn1} dataset. We contour cells corresponding to \algname{EF21} and \algname{LAG}, as special cases of \algname{CLAG}, by black rectangles. The red-contoured cell indicates the experiment with the smallest communication cost.}
	\label{fig:heatmapreducedclagijcnn1}
\end{figure}
Additional experiments validating the performance of \algname{CLAG} are reported in \cref{sec:extra_experiments}
\subsection{Other \algname{3PC} variants}
We consider the objective
\begin{equation*}
	\small \min _{\boldsymbol{D} \in \R^{d_{f} \times d_{e}}, \boldsymbol{E} \in \mathbb{R}^{d_{e} \times d_{f}}} \left[f(\boldsymbol{D}, \boldsymbol{E}):=\frac{1}{N} \sum_{i=1}^{N}\left\|\boldsymbol{D} \boldsymbol{E} a_{i}-a_{i}\right\|^{2}\right],
\end{equation*}
where $a_i$ are flattened represenations of images with $d_f = 784$, $\boldsymbol{D}$ and $\boldsymbol{E}$ are learned parameters of the autoencoder model. We fix the encoding dimensions as $d_e = 16$ and distribute the data samples across $n = 100$ clients. In order to control the heterogenity of this distribution, we consider three cases. First, each client owns the same data (``\textit{homogeneity level: $1$}''). Second, the data is randomly split among client (``\textit{homogeneity level: $0$}''). Finally, we consider an extremely heterogeneous case, where the images are ``\textit{split by labels}''. $K$ is set to $\nfr{d}{n}$, where $d = 2 \cdot d_f \cdot d_e = 25088$ is the total dimension of learning parameters $\boldsymbol{D}$ and $\boldsymbol{E}$. We apply three different sparsifiers (Top-$K$, Rand-$K$, Perm-$K$) for the first compressor of \algname{3PCv2} (\cref{alg:anna}) and fix the second one as Top-$K$.\footnote{\scriptsize See \cref{sec:extra_experiments,sec:contractive} for the definitions and more detailed description.} \algname{3PCv2} method communicates two sparse sequences at each communication round, while \algname{EF21} only one. To account for this, we select $K_1, K_2$ from the set $\{\nfr{K}{2}, K\}$, that is there are four possible choices  for compression levels $K_1, K_2$ of two sparsifiers in \algname{3PCv2}. Then we select the pair which works best. We fine-tune every method with the stepsizes from the set $\{2^{-12}, 2^{-11}, \dots, 2^{5}\}$ and select the best run based on the value of $\sqnorm{\nfxt}$ at the last iterate. The stepsize for each method is indicated in the legend of each plot. 

\Cref{fig:anna-100-nodes-grads_main} demonstrates that \algname{3PCv2} is \textit{competitive with the  \algname{EF21} method and, in some cases, superior}. The improvement is particularly prominent in the heterogeneous regime.
Experiments with other variants, \algname{3PCv1}--\algname{3PCv5}, including the experiments on a carefully designed syntetic quadratic problem, are reported in \cref{sec:extra_experiments}.

\bibliography{icml2022_3_points}

\begin{thebibliography}{39}
\providecommand{\natexlab}[1]{#1}
\providecommand{\url}[1]{\texttt{#1}}
\expandafter\ifx\csname urlstyle\endcsname\relax
  \providecommand{\doi}[1]{doi: #1}\else
  \providecommand{\doi}{doi: \begingroup \urlstyle{rm}\Url}\fi

\bibitem[Alistarh et~al.(2018)Alistarh, Hoefler, Johansson, Khirirat,
  Konstantinov, and Renggli]{Alistarh-EF-NIPS2018}
Alistarh, D., Hoefler, T., Johansson, M., Khirirat, S., Konstantinov, N., and
  Renggli, C.
\newblock The convergence of sparsified gradient methods.
\newblock In \emph{Advances in Neural Information Processing Systems
  (NeurIPS)}, 2018.

\bibitem[Beznosikov et~al.(2020)Beznosikov, Horv{\'a}th, Richt{\'a}rik, and
  Safaryan]{beznosikov2020biased}
Beznosikov, A., Horv{\'a}th, S., Richt{\'a}rik, P., and Safaryan, M.
\newblock On biased compression for distributed learning.
\newblock \emph{arXiv preprint arXiv:2002.12410}, 2020.

\bibitem[Bottou(2012)]{bottou2012stochastic}
Bottou, L.
\newblock \emph{Stochastic Gradient Descent Tricks}, volume 7700 of
  \emph{Lecture Notes in Computer Science (LNCS)}, pp.\  430--445.
\newblock Springer, neural networks, tricks of the trade, reloaded edition,
  January 2012.
\newblock URL
  \url{https://www.microsoft.com/en-us/research/publication/stochastic-gradient-tricks/}.

\bibitem[Chang \& Lin(2011)Chang and Lin]{chang2011libsvm}
Chang, C.-C. and Lin, C.-J.
\newblock {LIBSVM}: a library for support vector machines.
\newblock \emph{{ACM} {T}ransactions on {I}ntelligent {S}ystems and
  {T}echnology (TIST)}, 2\penalty0 (3):\penalty0 1--27, 2011.

\bibitem[Chen et~al.(2018)Chen, Giannakis, Sun, and Yin]{LAG}
Chen, T., Giannakis, G., Sun, T., and Yin, W.
\newblock {LAG}: Lazily aggregated gradient for communication-efficient
  distributed learning.
\newblock \emph{Advances in Neural Information Processing Systems}, 2018.

\bibitem[Dean et~al.(2012)Dean, Corrado, Monga, Chen, Devin, Mao, Senior,
  Tucker, Yang, Le, and et~al]{Dean2012}
Dean, J., Corrado, G., Monga, R., Chen, K., Devin, M., Mao, M., Senior, A.,
  Tucker, P., Yang, K., Le, Q.~V., and et~al.
\newblock Large scale distributed deep networks.
\newblock In \emph{Advances in Neural Information Processing Systems}, pp.\
  1223--1231, 2012.

\bibitem[Fatkhullin et~al.(2021)Fatkhullin, Sokolov, Gorbunov, Li, and
  Richt\'{a}rik]{EF21BW}
Fatkhullin, I., Sokolov, I., Gorbunov, E., Li, Z., and Richt\'{a}rik, P.
\newblock Ef21 with bells \& whistles: practical algorithmic extensions of
  modern error feedback.
\newblock \emph{arXiv preprint arXiv:2110.03294}, 2021.

\bibitem[Ghadikolaei et~al.(2021)Ghadikolaei, Stich, and Jaggi]{LENA}
Ghadikolaei, H.~S., Stich, S., and Jaggi, M.
\newblock {LENA}: Communication-efficient distributed learning with
  self-triggered gradient uploads.
\newblock In \emph{International Conference on Artificial Intelligence and
  Statistics}, pp.\  3943--3951. PMLR, 2021.

\bibitem[Gorbunov et~al.(2020)Gorbunov, Kovalev, Makarenko, and
  Richt\'{a}rik]{Lin_EC_SGD}
Gorbunov, E., Kovalev, D., Makarenko, D., and Richt\'{a}rik, P.
\newblock Linearly converging error compensated {SGD}.
\newblock In \emph{34th Conference on Neural Information Processing Systems
  (NeurIPS)}, 2020.

\bibitem[Gorbunov et~al.(2021)Gorbunov, Burlachenko, Li, and
  Richt\'arik]{gorbunov2021marina}
Gorbunov, E., Burlachenko, K.~P., Li, Z., and Richt\'arik, P.
\newblock {MARINA}: Faster non-convex distributed learning with compression.
\newblock In Meila, M. and Zhang, T. (eds.), \emph{Proceedings of the 38th
  International Conference on Machine Learning}, volume 139 of
  \emph{Proceedings of Machine Learning Research}, pp.\  3788--3798. PMLR,
  18--24 Jul 2021.
\newblock URL \url{https://proceedings.mlr.press/v139/gorbunov21a.html}.

\bibitem[Horv{\'a}th \& Richt{\'a}rik(2020)Horv{\'a}th and
  Richt{\'a}rik]{horvath2020better}
Horv{\'a}th, S. and Richt{\'a}rik, P.
\newblock A better alternative to error feedback for communication-efficient
  distributed learning.
\newblock \emph{arXiv preprint arXiv:2006.11077}, 2020.

\bibitem[Horv\'{a}th \& Richt\'{a}rik(2021)Horv\'{a}th and
  Richt\'{a}rik]{A_better_alternative}
Horv\'{a}th, S. and Richt\'{a}rik, P.
\newblock A better alternative to error feedback for communication-efficient
  distributed learning.
\newblock In \emph{9th International Conference on Learning Representations
  (ICLR)}, 2021.

\bibitem[Karimireddy et~al.(2019)Karimireddy, Rebjock, Stich, and
  Jaggi]{Karimireddy_SignSGD}
Karimireddy, S.~P., Rebjock, Q., Stich, S., and Jaggi, M.
\newblock Error feedback fixes {S}ign{SGD} and other gradient compression
  schemes.
\newblock In \emph{36th International Conference on Machine Learning (ICML)},
  2019.

\bibitem[Khaled et~al.(2020)Khaled, Mishchenko, and
  Richt\'{a}rik]{localSGD-AISTATS2020}
Khaled, A., Mishchenko, K., and Richt\'{a}rik, P.
\newblock Tighter theory for local {SGD} on identical and heterogeneous data.
\newblock In \emph{The 23rd International Conference on Artificial Intelligence
  and Statistics (AISTATS 2020)}, 2020.

\bibitem[Khirirat et~al.(2018)Khirirat, Feyzmahdavian, and Johansson]{DCGD}
Khirirat, S., Feyzmahdavian, H.~R., and Johansson, M.
\newblock Distributed learning with compressed gradients.
\newblock \emph{arXiv preprint arXiv:1806.06573}, 2018.

\bibitem[Kingma \& Ba(2014)Kingma and Ba]{ADAM}
Kingma, D.~P. and Ba, J.
\newblock Adam: a method for stochastic optimization.
\newblock In \emph{The 3rd International Conference on Learning
  Representations}, 2014.
\newblock URL \url{https://arxiv.org/pdf/1412.6980.pdf}.

\bibitem[Koloskova et~al.(2020)Koloskova, Lin, Stich, and
  Jaggi]{Koloskova2019DecentralizedDL}
Koloskova, A., Lin, T., Stich, S., and Jaggi, M.
\newblock Decentralized deep learning with arbitrary communication compression.
\newblock In \emph{International Conference on Learning Representations
  (ICLR)}, 2020.

\bibitem[Kone\v{c}n\'{y} et~al.(2016{\natexlab{a}})Kone\v{c}n\'{y}, McMahan,
  Ramage, and Richt\'{a}rik]{FEDOPT}
Kone\v{c}n\'{y}, J., McMahan, H.~B., Ramage, D., and Richt\'{a}rik, P.
\newblock Federated optimization: distributed machine learning for on-device
  intelligence.
\newblock \emph{arXiv:1610.02527}, 2016{\natexlab{a}}.

\bibitem[Kone\v{c}n\'{y} et~al.(2016{\natexlab{b}})Kone\v{c}n\'{y}, McMahan,
  Yu, Richt\'{a}rik, Suresh, and Bacon]{FEDLEARN}
Kone\v{c}n\'{y}, J., McMahan, H.~B., Yu, F., Richt\'{a}rik, P., Suresh, A.~T.,
  and Bacon, D.
\newblock Federated learning: strategies for improving communication
  efficiency.
\newblock In \emph{NIPS Private Multi-Party Machine Learning Workshop},
  2016{\natexlab{b}}.

\bibitem[LeCun et~al.(2010)LeCun, Cortes, and Burges]{LeCun_MNIST}
LeCun, Y., Cortes, C., and Burges, C.
\newblock Mnist handwritten digit database.
\newblock \emph{ATTLabs [Online]}, 2010.
\newblock URL \url{http://yann.lecun.com/exdb/mnist}.

\bibitem[Li \& Richt\'{a}rik(2020)Li and Richt\'{a}rik]{Nonconvex-sigma_k}
Li, Z. and Richt\'{a}rik, P.
\newblock A unified analysis of stochastic gradient methods for nonconvex
  federated optimization.
\newblock \emph{arXiv preprint arXiv:2006.07013}, 2020.

\bibitem[Li \& Richt{\'a}rik(2021)Li and Richt{\'a}rik]{CANITA}
Li, Z. and Richt{\'a}rik, P.
\newblock {CANITA}: Faster rates for distributed convex optimization with
  communication compression.
\newblock In \emph{Advances in Neural Information Processing Systems}, 2021.
\newblock arXiv:2107.09461.

\bibitem[Li et~al.(2020)Li, Kovalev, Qian, and Richt{\'a}rik]{ADIANA}
Li, Z., Kovalev, D., Qian, X., and Richt{\'a}rik, P.
\newblock Acceleration for compressed gradient descent in distributed and
  federated optimization.
\newblock In \emph{International Conference on Machine Learning (ICML)}, pp.\
  5895--5904. PMLR, 2020.

\bibitem[Li et~al.(2021)Li, Bao, Zhang, and Richt{\'a}rik]{PAGE}
Li, Z., Bao, H., Zhang, X., and Richt{\'a}rik, P.
\newblock {PAGE}: A simple and optimal probabilistic gradient estimator for
  nonconvex optimization.
\newblock In \emph{International Conference on Machine Learning (ICML)}, pp.\
  6286--6295. PMLR, 2021.

\bibitem[Lin et~al.(2018)Lin, Han, Mao, Wang, and Dally]{lin2018deep}
Lin, Y., Han, S., Mao, H., Wang, Y., and Dally, B.
\newblock Deep gradient compression: Reducing the communication bandwidth for
  distributed training.
\newblock In \emph{International Conference on Learning Representations}, 2018.

\bibitem[McMahan et~al.(2016)McMahan, Moore, Ramage, and Ag\"{u}era~y
  Arcas]{FedAvg2016}
McMahan, B., Moore, E., Ramage, D., and Ag\"{u}era~y Arcas, B.
\newblock Federated learning of deep networks using model averaging.
\newblock \emph{arXiv preprint arXiv:1602.05629}, 2016.

\bibitem[McMahan et~al.(2017)McMahan, Moore, Ramage, Hampson, and Ag\"{u}era~y
  Arcas]{FL2017-AISTATS}
McMahan, H.~B., Moore, E., Ramage, D., Hampson, S., and Ag\"{u}era~y Arcas, B.
\newblock Communication-efficient learning of deep networks from decentralized
  data.
\newblock In \emph{Proceedings of the 20th International Conference on
  Artificial Intelligence and Statistics (AISTATS)}, 2017.

\bibitem[Mishchenko et~al.(2019)Mishchenko, Gorbunov, Tak{\'a}{\v{c}}, and
  Richt{\'a}rik]{DIANA}
Mishchenko, K., Gorbunov, E., Tak{\'a}{\v{c}}, M., and Richt{\'a}rik, P.
\newblock Distributed learning with compressed gradient differences.
\newblock \emph{arXiv preprint arXiv:1901.09269}, 2019.

\bibitem[Nesterov et~al.(2018)]{nesterov2018lectures}
Nesterov, Y. et~al.
\newblock \emph{Lectures on convex optimization}, volume 137.
\newblock Springer, 2018.

\bibitem[Richt\'{a}rik et~al.(2021)Richt\'{a}rik, Sokolov, and
  Fatkhullin]{EF21}
Richt\'{a}rik, P., Sokolov, I., and Fatkhullin, I.
\newblock {EF21}: A new, simpler, theoretically better, and practically faster
  error feedback.
\newblock In \emph{Advances in Neural Information Processing Systems}, 2021.

\bibitem[Safaryan et~al.(2021{\natexlab{a}})Safaryan, Islamov, Qian, and
  Richt\'{a}rik]{FedNL}
Safaryan, M., Islamov, R., Qian, X., and Richt\'{a}rik, P.
\newblock {FedNL}: Making {N}ewton-type methods applicable to federated
  learning.
\newblock \emph{arXiv preprint arXiv:2106.02969}, 2021{\natexlab{a}}.

\bibitem[Safaryan et~al.(2021{\natexlab{b}})Safaryan, Shulgin, and
  Richt\'{a}rik]{UP2021}
Safaryan, M., Shulgin, E., and Richt\'{a}rik, P.
\newblock Uncertainty principle for communication compression in distributed
  and federated learning and the search for an optimal compressor.
\newblock \emph{Information and Inference: A Journal of the IMA},
  2021{\natexlab{b}}.

\bibitem[Seide et~al.(2014)Seide, Fu, Droppo, Li, and Yu]{Seide2014}
Seide, F., Fu, H., Droppo, J., Li, G., and Yu, D.
\newblock 1-bit stochastic gradient descent and its application to
  data-parallel distributed training of speech {DNN}s.
\newblock In \emph{Fifteenth Annual Conference of the International Speech
  Communication Association}, 2014.

\bibitem[Stich(2020)]{localSGD-Stich}
Stich, S.~U.
\newblock Local {SGD} converges fast and communicates little.
\newblock In \emph{International Conference on Learning Representations}, 2020.

\bibitem[Stich et~al.(2018)Stich, Cordonnier, and Jaggi]{Stich-EF-NIPS2018}
Stich, S.~U., Cordonnier, J.-B., and Jaggi, M.
\newblock Sparsified {SGD} with memory.
\newblock In \emph{Advances in Neural Information Processing Systems
  (NeurIPS)}, 2018.

\bibitem[Sun et~al.(2019)Sun, Chen, Giannakis, and Yang]{LAQ}
Sun, J., Chen, T., Giannakis, G., and Yang, Z.
\newblock Communication-efficient distributed learning via lazily aggregated
  quantized gradients.
\newblock \emph{Advances in Neural Information Processing Systems},
  32:\penalty0 3370--3380, 2019.

\bibitem[Szlendak et~al.(2021)Szlendak, Tyurin, and Richt\'{a}rik]{PermK}
Szlendak, R., Tyurin, A., and Richt\'{a}rik, P.
\newblock Permutation compressors for provably faster distributed nonconvex
  optimization.
\newblock \emph{arXiv preprint arXiv:2110.03300, 2021}, 2021.

\bibitem[Tang et~al.(2020)Tang, Lian, Yu, Zhang, and Liu]{DoubleSqueeze}
Tang, H., Lian, X., Yu, C., Zhang, T., and Liu, J.
\newblock {D}ouble{S}queeze: {P}arallel stochastic gradient descent with
  double-pass error-compensated compression.
\newblock In \emph{Proceedings of the 36th International Conference on Machine
  Learning (ICML)}, 2020.

\bibitem[Woodworth et~al.(2020)Woodworth, Patel, Stich, Dai, Bullins, McMahan,
  Shamir, and Srebro]{Blake2020}
Woodworth, B., Patel, K.~K., Stich, S.~U., Dai, Z., Bullins, B., McMahan,
  H.~B., Shamir, O., and Srebro, N.
\newblock Is local {SGD} better than minibatch {SGD}?
\newblock \emph{arXiv preprint arXiv:2002.07839}, 2020.

\end{thebibliography}
\bibliographystyle{icml2022}

\newpage
\appendix
\onecolumn

\part*{APPENDIX}

\renewcommand{\contentsname}{Table of Contents}
{\setlength{\parskip}{0.15em}\tableofcontents}

\newpage



\section{Examples of Contractive Compressors}  \label{sec:contractive}

The simplest example of a contractive compressor is the identity mapping, $\cC(x)\equiv x$, which satisfies \eqref{eq:contractive-09u09fduf} with $\lambdaNEW=1$, and using which \algname{DCGD} reduces to (distributed) gradient descent.

\subsection{Top-$K$} A typical non-trivial example of a contractive compressor is the Top-$K$ sparsification operator \cite{Alistarh-EF-NIPS2018}, which is a deterministic mapping characterized by a parameter $1\leq K\leq d$ defining the required level of sparsification. The smaller this parameter is, the higher compression level is applied, and the smaller the contraction parameter $\lambdaNEW$ becomes, which indicates that there is a larger error between the message $x$ we wanted to send, and the compressed message $\cC(x)$ we actually  sent. In the extreme case $K=d$, we have $\cC(x)=x$, and the input vector is left intact, and hence uncompressed. In this case, $\lambdaNEW=1$. If $K=1$, then all entries of $x$ are zeroed out, except for the largest entry in absolute value, breaking ties arbitrarily. This choice offers a $d:1$ compression ratio, which can be dramatic if $d$ is large, which is the case when working with  big models. In this case,  $\lambdaNEW=1/d$. The general choice of $K$ leaves just $K$ nonzero entries intact, those that are largest in absolute value (again, breaking ties arbitrarily), with the remaining $d-K$ entries zeroed out. This offers a $(d-K):1$ compression ratio, with contraction factor $\lambdaNEW=K/d$. 

\subsection{Rand-$K$} One of the simplest randomized sparsification operators is Rand-$K$ \cite{DCGD}. It is similar to Top-$K$, with the exception that the $K$ entries that are retained are chosen uniformly at random rather than greedily. Just like in the case of Top-$K$, the worst-case (expected) error produced by Rand-$K$ is characterized by $\lambdaNEW=K/d$. However, on inputs $x$ that are not worse-case, which naturally happens often throughout the training process, the empirical error of the greedy Top-$K$ sparsifier can be much smaller than that of its randomized cousin. This has been observed in practice, and this is one of the reasons why greedy compressors, such as Top-$K$, are often preferred to their randomized counterparts. 

\subsection{cRand-$K$} Contractive Rand-$K$ operator applied to vector $x \in \R^d$ uniformly at random chooses $K$ entries out of $d$ but, unlike Rand-$K$, does not scale the resulting vector. In this case, the resulting vector is no more unbiased but it still satisfies the definition of the contractive operator. Indeed, let $\cS$ be a set of indices of size $K$. Then,
\vspace{-1mm}
\begin{align*}
\Exp{\|\cC(x) - x\|^2_2} =  \Exp{ \sum\limits_{i=1}^d 1_{i \notin \cS} x_i^2} = \sum\limits_{i=1}^d \Exp {1_{i \notin \cS} x_i^2} = \sum\limits_{i=1}^d \left(1 - \frac{K}{d}\right) x_i^2 =  \left(1 - \frac{K}{d}\right) \|x\|^2_2.
\end{align*}

\vspace{-4mm}
\subsection{Perm-$K$ and cPerm-$K$} Permutation compressor (Perm-$K$) is described in~\cite{PermK} (case $d > n$, Definition 2 in the original paper). Contractive permutation compressor (cPerm-$K$) on top of Perm-$K$ scales the resulting vector by factor $\frac{1}{1 + \omega}$.

\subsection{Unbiased compressors}

Rand-$K$, as defined above, arises from a more general class of compressors, which we now present, by appropriate scaling.

\begin{definition}[Unbiased Compressor]\label{def:unbiased_compressor}
	We say that a randomized map $\cQ: \R^{d} \rightarrow \R^{d}$ is an {\em unbiased compression operator}, or simply just {\em unbiased compressor}, if there exists a constant $\omega \geq 0$ such that 
		\begin{eqnarray}\label{eq:unb_compressor}
		\Exp{\cQ(x)} = x,\quad \Exp{\|\cQ(x) - x\|^{2}} \leq \omega \|x\|^{2}, \quad \forall x\in \R^d.
	\end{eqnarray}
\end{definition}

Its is well known and trivial to check that for any unbiased compressor $\cQ$, the compressor $\frac{1}{\omega+1}\cQ$ is contractive, with contraction parameter $\alpha = \frac{1}{\omega+1}$. It is easy to see that the contractive Rand-$K$ operator defined above becomes unbiased once it is scaled by the factor $\frac{d}{K}$.

\vspace{-1mm}
\subsection{Further examples}
For further examples of  contractive compressors (e.g., quantization-based, rank-based), we refer the reader to \citet{beznosikov2020biased} and \citet{UP2021, FedNL}.

\newpage

\section{Proofs of The Main Results}

\subsection{Three Lemmas}

We will rely on two lemmas, one from \citep{EF21}, and one from \citep{PAGE}. The first lemma will allow us to simplify the expression for the maximal allowable stepsize in our method (at the cost of being suboptimal by the factor of 2 at most), and the second forms an important step in our convergence proof.

\begin{lemma}[Lemma 5 of \citep{EF21}]\label{lem:stepsize_page_fact}
	If $0 \leq \gamma \leq \fr{1}{\sqrt{a}+b}$, then $a \gamma^{2}+b \gamma \leq 1$. Moreover, the bound is tight up to the factor of 2 since $\fr{1}{\sqrt{a}+b} \leq \min \left\{\fr{1}{\sqrt{a}}, \fr{1}{b}\right\} \leq \fr{2}{\sqrt{a}+b}$.
\end{lemma}

\begin{lemma}[Lemma 2 of \citep{PAGE}]\label{lem:aux_smooth_lemma}
	Suppose that function $f$ is $L_{-}$-smooth and let $x^{t+1}\eqdef x^{t}-\g g^{t} ,$ where $g^t\in \R^d$ is any vector, and $\g>0$ is any scalar. Then we have
	\begin{eqnarray}\label{eq:aux_smooth_lemma}
		f(x^{t+1}) \leq f(x^{t})-\fr{\g}{2}\sqnorm{\nabla f(x^{t})}-\left(\fr{1}{2 \g}-\fr{L_{-}}{2}\right)\sqnorm{x^{t+1}-x^{t}}+\fr{\g}{2}\sqnorm{g^{t}-\nabla f(x^{t})}.
	\end{eqnarray}
\end{lemma}

We  now state and  derive the main technical lemma.
\begin{lemma}[Lemma~\ref{lem:key_lemma}]\label{lem:key_lemma_appendix}
	Let Assumption~\ref{as:L_+} hold. Consider the method from \eqref{eq:CGD_1}--\eqref{eq:CGD_2}. Then, for all $t\ge 0$ the sequence
	\begin{equation}\label{eq:G^t_appendix} 
	 G^t \eqdef \frac{1}{n} \sum \limits_{i=1}^n \sqnorm{ g_i^t-\nabla f_i(x^{t}) }
	\end{equation}
	satisfies
	\begin{eqnarray}
		\Exp{G^{t+1}} \leq (1-\thetaNEW)\Exp{G^t} + \betaNEW L_{+}^2 \Exp{\sqnorm{x^{t+1} - x^t}}. \label{eq:key_inequality_appendix}
	\end{eqnarray}
\end{lemma}
\begin{proof}
	By definition of $G^t$ and three points compressor we have
	\begin{eqnarray*}
		\Exp{G^{t+1}} &=& \frac{1}{n} \sum \limits_{i=1}^n \Exp{\sqnorm{ g_i^{t+1}-\nabla f_i(x^{t+1}) }}\\
		&\overset{\eqref{eq:CGD_2},\eqref{eq:ttp}}{\leq}& \frac{1-\thetaNEW}{n}\sum\limits_{i=1}^n \Exp{\sqnorm{ g_i^{t}-\nabla f_i(x^{t}) }} + \frac{\betaNEW}{n}\sum\limits_{i=1}^n \sqnorm{ \nabla f_i(x^{t+1}) - \nabla f_i(x^t) }\\
		&=& (1-\thetaNEW)\Exp{G^t} + \frac{\betaNEW}{n}\sum\limits_{i=1}^n \sqnorm{ \nabla f_i(x^{t+1}) - \nabla f_i(x^t) }.
	\end{eqnarray*}
	Using Assumption~\ref{as:L_+}, we upper bound the last term by $\betaNEW L_{+}^2\Exp{\sqnorm{x^{t+1} - x^t}}$ and get the result.
\end{proof}

\subsection{General Non-Convex Functions}
Below we restate the main result for general non-convex functions and provide the full proof.
\begin{theorem}[Theorem~\ref{thm:main_thm_gen_non_cvx}]\label{thm:main_thm_gen_non_cvx_appendix}
	Let Assumptions~\ref{ass:diff},~\ref{as:L_smoothness},~\ref{as:L_+} hold. Assume that the stepsize $\gamma$ of the method from \eqref{eq:CGD_1}--\eqref{eq:CGD_2} satisfies $0 \leq \gamma \leq \nicefrac{1}{M}$, where $M = L_{-} + L_{+}\sqrt{\nicefrac{\betaNEW}{\thetaNEW}}$. Then, for any $T \ge 0$ we have
	\begin{equation}
		\Exp{\norm{ \nabla f(\hat x^T)}^2} \leq \frac{2 \Delta^0}{\gamma T} + \fr{\Exp{G^0}}{\thetaNEW T}, \label{eq:main_thm_gen_non_cvx_appendix}
	\end{equation}
	where $\hat x^T$ is sampled uniformly at random from the points $\{x^0, x^1, \ldots, x^{T-1}\}$ produced by \eqref{eq:CGD_1}--\eqref{eq:CGD_2}, $\Delta^0 = f(x^0) - f^{\inf}$, and $G^0$ is defined in \eqref{eq:G^t}.
\end{theorem}
\begin{proof}
Using Lemma~\ref{lem:aux_smooth_lemma} and Jensen's inequality applied of the squared norm, we get 
\begin{eqnarray}
			f(x^{t+1}) &\letext{\eqref{eq:aux_smooth_lemma}}& f(x^{t})-\frac{\g}{2}\sqnorm{\nabla f(x^{t})}-\left(\frac{1}{2 \g}-\frac{L_{-}}{2}\right)\sqnorm{x^{t+1}-x^{t}}+\frac{\g}{2}\sqnorm{\suminn \rb{g_i^t-\nabla f_i(x^{t}) }}\notag \\
			&\overset{\eqref{eq:G^t}}{\leq}&
			f(x^{t})-\frac{\g}{2}\sqnorm{\nabla f(x^{t})}-\left(\frac{1}{2 \g}-\frac{L_{-}}{2}\right)\sqnorm{x^{t+1}-x^{t}}+\frac{\g}{2} G^t. \label{eq:aux_smooth_lemma_distrib}
\end{eqnarray}

Subtracting $f^{\text {inf }}$ from both sides of \eqref{eq:aux_smooth_lemma_distrib} and taking expectation, we get
\begin{eqnarray}\label{eq:func_diff_distrib}
	\Exp{f(x^{t+1})-\finf} &\leq& \quad \Exp{f(x^{t})-\finf}-\frac{\gamma}{2} \Exp{\sqnorm{\nabla f(x^{t})}} \notag \\
	&& \qquad -\left(\frac{1}{2 \gamma}-\frac{L_{-}}{2}\right) \Exp{\sqnorm{x^{t+1}-x^{t}}}+ \frac{\gamma}{2}\Exp{G^t}.
\end{eqnarray}
		
Next, we add \eqref{eq:func_diff_distrib} to a $\frac{\gamma}{2 \thetaNEW}$ multiple of \eqref{eq:key_inequality} and derive
\begin{eqnarray*}
	\Exp{f(x^{t+1})-\finf}+\frac{\gamma}{2 \thetaNEW}\Exp{G^{t+1}} &\leq& \Exp{f(x^{t})-\finf}-\frac{\gamma}{2}\Exp{\sqnorm{\nabla f(x^{t})}} - \left(\frac{1}{2 \gamma}-\frac{L_{-}}{2}\right) \Exp{\sqnorm{x^{t+1} - x^t}}\\
	&&\qquad + \frac{\gamma}{2} \Exp{G^t}+\frac{\gamma}{2 \thetaNEW}\left((1-\thetaNEW)\Exp{G^t} + \betaNEW L_{+}^2 \Exp{\sqnorm{x^{t+1} - x^t}}\right) \\
	&=& \Exp{f(x^{t})-\finf}+\frac{\gamma}{2 \thetaNEW}\Exp{G^{t}} - \frac{\gamma}{2}\Exp{\sqnorm{\nabla f(x^{t})}}\\
	&&\qquad -\left(\frac{1}{2\g} -\frac{L_{-}}{2} - \frac{\g\betaNEW L_+^2}{2\thetaNEW} \right) \Exp{\sqnorm{x^{t+1} - x^t}} \\
	& \leq& \Exp{f(x^{t})-\finf}+\frac{\gamma}{2 \thetaNEW}\Exp{G^{t}} - \frac{\gamma}{2}\Exp{\sqnorm{\nabla f(x^{t})}},
\end{eqnarray*}
where the last inequality follows from the bound $\g ^2\frac{\betaNEW L_+^2}{\thetaNEW} + L_{-}\g \leq 1,$ which holds because of Lemma \ref{lem:stepsize_page_fact} and our assumption on the stepsize. Summing up inequalities for $t =0, \ldots, T-1,$ we get
		$$
		0 \leq \Exp{f(x^{T})-\finf}+\frac{\gamma}{2 \thetaNEW}\Exp{G^{T}} \leq \Exp{f(x^{0})-\finf}+\frac{\gamma}{2 \thetaNEW}\Exp{G^{0}}-\frac{\gamma}{2} \sum_{t=0}^{T-1} \Exp{\sqnorm{\nabla f(x^{t})}}.
		$$
		Multiplying both sides by $\frac{2}{\gamma T}$, after rearranging we obtain
		$$
		\sum_{t=0}^{T-1} \frac{1}{T} \Exp{\sqnorm{\nabla f (x^{t})}} \leq \frac{2 \Delta^{0}}{\gamma T} + \frac{\Exp{G^0}}{\thetaNEW T},
		$$
		where $\Delta^0 = f(x^0) - f^{\inf}$. It remains to notice that the left hand side can be interpreted as $\Exp{\sqnorm{\nabla f(\hat{x}^{T})}}$, where $\hat{x}^{T}$ is chosen from $x^{0}, x^{1}, \ldots, x^{T-1}$ uniformly at random.
\end{proof}

\subsection{P{\L} Functions}
Below we restate the main result for P{\L} functions and provide the full proof.

\begin{theorem}[Theorem~\ref{thm:main_thm_PL}]\label{thm:main_thm_PL_appendix}
	Let Assumptions~\ref{ass:diff},~\ref{as:L_smoothness},~\ref{as:L_+},~\ref{as:PL} hold. Assume that the stepsize $\gamma$ of the method from \eqref{eq:CGD_1}--\eqref{eq:CGD_2} satisfies $0 \leq \gamma \leq \nicefrac{1}{M}$, where $M = \max\left\{L_{-} + L_{+}\sqrt{\nicefrac{2\betaNEW}{\thetaNEW}}, \nicefrac{\thetaNEW}{2\mu}\right\}$. Then, for any $T \ge 0$ and $\Delta^0 = f(x^0) - f(x^*)$ we have
	\begin{equation}
		\Exp{f(x^T) - f(x^*)} \leq \left(1 - \gamma\mu\right)^T\left(\Delta^0 + \frac{\gamma}{\thetaNEW}\Exp{G^{0}}\right). \label{eq:main_thm_PL_appendix}
	\end{equation}
\end{theorem}
\begin{proof}
	First of all, we notice that \eqref{eq:func_diff_distrib} holds in this case as well. Therefore, using P{\L} condition we derive
	\begin{eqnarray}\label{eq:func_diff_distrib_PL}
	\Exp{f(x^{t+1})-\finf} &\overset{\eqref{eq:func_diff_distrib}}{\leq}& \quad \Exp{f(x^{t})- f(x^*)}-\frac{\gamma}{2} \Exp{\sqnorm{\nabla f(x^{t})}} \notag \\
	&& \qquad -\left(\frac{1}{2 \gamma}-\frac{L_{-}}{2}\right) \Exp{\sqnorm{x^{t+1}-x^{t}}}+ \frac{\gamma}{2}\Exp{G^t}\notag\\
	&\overset{\eqref{eq:PL}}{\leq}& (1-\gamma\mu)\Exp{f(x^{t})- f(x^*)} -\left(\frac{1}{2 \gamma}-\frac{L_{-}}{2}\right) \Exp{\sqnorm{x^{t+1}-x^{t}}}+ \frac{\gamma}{2}\Exp{G^t}.
\end{eqnarray}

Next, we add \eqref{eq:func_diff_distrib} to a $\frac{\gamma}{\thetaNEW}$ multiple of \eqref{eq:key_inequality} and derive
\begin{eqnarray*}
	\Exp{f(x^{t+1})-f(x^*)+\frac{\gamma}{\thetaNEW}G^{t+1}} &\leq& (1-\gamma\mu)\Exp{f(x^{t})-f(x^*)}- \left(\frac{1}{2 \gamma}-\frac{L_{-}}{2}\right) \Exp{\sqnorm{x^{t+1} - x^t}}\\
	&&\qquad + \frac{\gamma}{2} \Exp{G^t}+\frac{\gamma}{\thetaNEW}\left((1-\thetaNEW)\Exp{G^t} + \betaNEW L_{+}^2 \Exp{\sqnorm{x^{t+1} - x^t}}\right) \\
	&=& (1-\gamma\mu)\Exp{f(x^{t})- f(x^*)} + \left(1 - \frac{\thetaNEW}{2}\right)\frac{\gamma}{\thetaNEW}\Exp{G^{t}}\\
	&&\qquad -\left(\frac{1}{2\g} -\frac{L_{-}}{2} - \frac{\g\betaNEW L_+^2}{\thetaNEW} \right) \Exp{\sqnorm{x^{t+1} - x^t}} \\
	& \leq& (1 - \gamma\mu)\Exp{f(x^{t})- f(x^*)+\frac{\gamma}{\thetaNEW} G^{t}},
\end{eqnarray*}
where the last inequality follows from the bound $\g ^2\frac{2\betaNEW L_+^2}{\thetaNEW} + L_{-}\g \leq 1$ and $1 - \nicefrac{\thetaNEW}{2} \leq 1 - \gamma\mu$, which holds because of our assumption on the stepsize and Lemma \ref{lem:stepsize_page_fact}. Unrolling the recurrence, we obtain
\begin{equation*}
	\Exp{f(x^T) - f(x^*)} \leq \Exp{f(x^{T})-f(x^*)+\frac{\gamma}{\thetaNEW}G^{T}} \leq (1 - \gamma\mu)^T\left(\Delta^0 + \Exp{\frac{\gamma}{\thetaNEW}G^{0}}\right).
\end{equation*}
\end{proof}

\newpage

\section{Three Point Compressor: Special Cases}\label{sec:3points_special_cases_appendix}
In this section, we show that several known approaches to compressed communication can be viewed as special cases of our framework \eqref{eq:CGD_1}--\eqref{eq:CGD_2}. Moreover, we design several new methods fitting our scheme. Please refer to Table~\ref{tab:methods}    for an overview.

\subsection{Error Feedback 2021: \algname{EF21}}

\begin{algorithm}[h]
   \caption{Error Feedback 2021 (\algname{EF21})}\label{alg:b_diana}
\begin{algorithmic}[1]
   \STATE {\bfseries Input:} starting point $x^0$, stepsize $\gamma$, number of iterations $T$, starting vectors $g_i^0$, $i \in [n]$
   \FOR{$t=0,1,\ldots,T-1$}
   \STATE Broadcast $g^t$ to all workers
   \FOR{$i = 1,\ldots,n$ in parallel} 
   \STATE $x^{t+1} = x^t - \gamma g^t$
   \STATE Set $g_i^{t+1} = g_i^t + \cC(\nabla f_i(x^{t+1}) - g_i^t)$ 
   \ENDFOR
   \STATE $g^{t+1} = \tfrac{1}{n}\sum_{i=1}^ng_i^{t+1} = g^t + \tfrac{1}{n}\sum_{i=1}^n \cC(\nabla f_i(x^{t+1}) - g_i^t)$
   \ENDFOR
   \STATE {\bfseries Return:} $\hat x^T$ chosen uniformly at random from $\{x^t\}_{t=0}^{T-1}$
\end{algorithmic}
\end{algorithm}

The next lemma shows that \algname{EF21} uses a special three point compressor.
\begin{lemma}\label{lem:biased_diana}
	The compressor
	\begin{eqnarray}
	\cC_{h,y}(x) \eqdef h + \cC(x - h),  \label{eq:b_diana_compressor}
	\end{eqnarray}
	satisfies \eqref{eq:ttp} with $\thetaNEW \eqdef 1 - (1-\alpha)(1+s)$ and $\betaNEW \eqdef (1-\alpha)\rb{1+ s^{-1} }$, where $s> 0$ is such that  $(1-\alpha)\rb{1 + s} < 1$.
\end{lemma}
\begin{proof}
	By definition of $\cC_{h,y}(x)$ and $\cC$ we have
	\begin{eqnarray*}
		\Exp{\sqnorm{\cC_{h,y}(x) - x }} &=& \Exp{\sqnorm{\cC(x - h) - (x - h) }}\\
		&\le& (1-\alpha) \sqnorm{x-h}\\
		&=& (1-\alpha) \sqnorm{(x - y) + (y - h)}\\
		&\le& (1-\alpha)\rb{1+{s}}\sqnorm{h - y} +  (1-\alpha)\rb{1+s^{-1}} \sqnorm{x - y}.
	\end{eqnarray*}
\end{proof}

Therefore, \algname{EF21} fits our framework. Using our general analysis (Theorems~\ref{thm:main_thm_gen_non_cvx} and \ref{thm:main_thm_PL}) we derive the following result.

\begin{theorem}\label{thm:b_diana} 
	\algname{EF21} is a special case of the method from \eqref{eq:CGD_1}--\eqref{eq:CGD_2} with $\cC_{h,y}(x)$ defined in \eqref{eq:b_diana_compressor} and $\thetaNEW =  \alpha - s(1-\alpha)$ and $\betaNEW = (1-\alpha)\rb{1+s^{-1}}$, where $s> 0$ is such that $(1-\alpha)(1+s) < 1$.	
	\begin{enumerate}		
	\item If Assumptions~\ref{ass:diff},~\ref{as:L_smoothness},~\ref{as:L_+} hold and the stepsize $\gamma$ satisfies $0 \leq \gamma \leq \nicefrac{1}{M}$, where $M = L_{-} + L_{+}\sqrt{\nicefrac{(1-\alpha)\rb{1+s^{-1}}}{(\alpha - s(1-\alpha))}}$, then for any $T \ge 0$ we have
	\begin{equation}
		\Exp{\norm{ \nabla f(\hat x^T)}^2} \leq \frac{2 \Delta^0}{\gamma T} + \fr{\Exp{G^0}}{(\alpha - s(1-\alpha)) T}, \label{eq:b_diana_gen_non_cvx}
	\end{equation}
	where $\hat x^T$ is sampled uniformly at random from the points $\{x^0, x^1, \ldots, x^{T-1}\}$ produced by \algname{EF21}, $\Delta^0 = f(x^0) - f^{\inf}$, and $G^0$ is defined in \eqref{eq:G^t}. 
	\item If additionaly Assumption~\ref{as:PL} hold and $0 \leq \gamma \leq \nicefrac{1}{M}$ for $M = \max\left\{L_{-} + L_{+}\sqrt{\nicefrac{2(1-\alpha)\rb{1+s^{-1}}}{(\alpha - s(1-\alpha))}}, \nicefrac{(\alpha - s(1-\alpha))}{2\mu}\right\}$, then for any $T \ge 0$ we have
	\begin{equation}
		\Exp{f(x^T) - f(x^*)} \leq \left(1 - \gamma\mu\right)^T\left(\Delta^0 + \frac{\gamma}{\alpha - s(1-\alpha)}\Exp{G^{0}}\right). \label{eq:b_diana_PL}
	\end{equation}
	\end{enumerate}
\end{theorem}

Guided by exactly the same arguments as we use in the anlysis pf \algname{3PCv5}, we consider $\nicefrac{\betaNEW}{\thetaNEW}$ as a function of $s$ and optimizing this function in $s$ and find the optimal value of this ratio.

\begin{lemma}\label{lem:b_diana_technical}
	The optimal value of
	\begin{equation}
		\frac{\betaNEW}{\thetaNEW}(s) = \frac{(1-\alpha)\rb{1+s^{-1}}}{(\alpha - s(1-\alpha))}\notag
	\end{equation}
	under the constraint $0 < s < \nicefrac{\alpha}{(1-\alpha)}$ equals
	\begin{equation}
		\frac{\betaNEW}{\thetaNEW}(s_*) = \frac{(1-\alpha)}{(1 - \sqrt{1-\alpha})^2} \leq \frac{4(1-\alpha)}{\alpha^2}\notag
	\end{equation}
	and it is achieved at $s^* = -1 + \sqrt{\nicefrac{1}{(1-\alpha)}}$.
\end{lemma}
\begin{proof}
	First of all, we find the derivative of the considered function:
	\begin{eqnarray*}
		\left(\frac{\betaNEW}{\thetaNEW}(s)\right)' &=& (1-\alpha)\frac{(1-\alpha)s^2 + 2(1-\alpha)s - \alpha}{(\alpha s - s^2(1-\alpha))^2}.
	\end{eqnarray*}
	The function has 2 critical points: $-1 \pm \sqrt{\nicefrac{1}{(1-\alpha)}}$. Moreover, the derivative is non-positive for $s \in (0, -1 + \sqrt{\nicefrac{1}{(1-\alpha)}}]$ and negative for $s \in (-1 + \sqrt{\nicefrac{1}{(1-\alpha)}}, +\infty)$. This implies that the optimal value on the interval $s \in (0, \nicefrac{\alpha}{(1-\alpha)})$ is achieved at $s_* = -1 + \sqrt{\nicefrac{1}{(1-\alpha)}}$. Via simple computations one can verify that
	\begin{equation*}
		\frac{\betaNEW}{\thetaNEW}(s_*) = \frac{(1-\alpha)}{(1 - \sqrt{1-\alpha})^2}.
	\end{equation*}
	Finally, since $1 - \sqrt{1-\alpha} \geq \nicefrac{\alpha}{2}$, we have
	\begin{equation}
		\frac{\betaNEW}{\thetaNEW}(s_*) \leq \frac{4(1-\alpha)}{\alpha^2}.\notag
	\end{equation}
\end{proof}

Using this and Corollaries~\ref{cor:main_cor_gen_non_cvx}, \ref{cor:main_cor_PL}, we get the following complexity results.
\begin{corollary}\label{cor:b_diana}
	\begin{enumerate}
		\item Let the assumptions from the first part of Theorem~\ref{thm:b_diana} hold, $s = s_* = -1 + \sqrt{\nicefrac{1}{(1-\alpha)}}$, and
	\begin{equation*}
		\gamma = \frac{1}{L_{-} + L_{+}\sqrt{\nicefrac{(1-\alpha)}{(1 - \sqrt{1-\alpha})^2}}}.
	\end{equation*}
	Then for any $T$ we have
	\begin{equation}
		\Exp{\norm{ \nabla f(\hat x^T)}^2} \leq \frac{2 \Delta^0\left(L_{-} + L_{+}\sqrt{\nicefrac{(1-\alpha)}{(1 - \sqrt{1-\alpha})^2}}\right)}{ T} + \fr{\Exp{G^0}}{(1 - \sqrt{1-\alpha}) T},\notag
	\end{equation}
	i.e., to achieve $\Exp{\norm{ \nabla f(\hat x^T)}^2} \leq \varepsilon^2$ for some $\varepsilon > 0$ the method requires
	\begin{equation}
		T = \cO\left(\frac{\Delta^0\left(L_{-} + L_{+}\sqrt{\nicefrac{(1 - \alpha)}{\alpha^2}}\right)}{\varepsilon^2} + \fr{\Exp{G^0}}{\alpha \varepsilon^2}\right) \label{eq:b_diana_complexity_gen_non_cvx}
	\end{equation}
	iterations/communication rounds.
	\item Let the assumptions from the second part of Theorem~\ref{thm:b_diana} hold and
	\begin{equation*}
		\gamma = \min\left\{\frac{1}{L_{-} + L_{+}\sqrt{\nicefrac{2(1-\alpha)}{(1 - \sqrt{1-\alpha})^2}}}, \frac{1-\sqrt{1-\alpha}}{2\mu}\right\}.
	\end{equation*}
	Then to achieve $\Exp{f(x^T) - f(x^*)} \leq \varepsilon$ for some $\varepsilon > 0$ the method requires
	\begin{equation}
		\cO\left(\max\left\{\frac{L_{-} + L_{+}\sqrt{\nicefrac{(1 - \alpha)}{\alpha^2}}}{\mu}, \alpha\right\}\log \frac{\Delta^0 + \Exp{G^{0}}\nicefrac{\gamma}{\alpha}}{\varepsilon}\right) \label{eq:b_diana_complexity_PL}
	\end{equation}
	iterations/communication rounds.
	\end{enumerate}
\end{corollary}

\newpage
\subsection{LAG: Lazily Aggregated Gradient}

\begin{algorithm}[h]
	\caption{\algname{LAG}: Lazily Aggregated Gradient}\label{alg:lag}
	\begin{algorithmic}[1]
		\STATE {\bfseries Input:} starting point $x^0$, stepsize $\gamma$, number of iterations $T$, starting vectors $g_i^0$, $i \in [n]$, trigger parameter $\zeta > 0$
		\FOR{$t=0,1,\ldots,T-1$}
		\STATE Broadcast $g^t$ to all workers
		\FOR{$i = 1,\ldots,n$ in parallel} 
		\STATE $x^{t+1} = x^t - \gamma g^t$
		\STATE Set $g_i^{t+1} = \begin{cases} \nabla f_i(x^{t+1}) ,& \text{if } \|\nabla f_i(x^{t+1}) - g_i^t\|^2 > \zeta \|\nabla f_i(x^{t+1}) - \nabla f_i(x^t)\|^2,\\ g_i^t,& \text{otherwise} \end{cases}$ 
		\ENDFOR
		\STATE $g^{t+1} = \tfrac{1}{n}\sum_{i=1}^ng_i^{t+1}$
		\ENDFOR
		\STATE {\bfseries Return:} $\hat x^T$ chosen uniformly at random from $\{x^t\}_{t=0}^{T-1}$
	\end{algorithmic}
\end{algorithm}

The next lemma shows that \algname{LAG} is a special three point compressor.
\begin{lemma}\label{lem:lag}
	The compressor
	\begin{eqnarray}
		\cC_{h,y}(x) \eqdef \begin{cases} x,& \text{if } \|x- h\|^2 > \zeta \|x - y\|^2,\\ h,& \text{otherwise,} \end{cases}  \label{eq:lag_compressor}
	\end{eqnarray}
	satisfies \eqref{eq:ttp} with $\thetaNEW \eqdef 1$ and $\betaNEW \eqdef \zeta$.
\end{lemma}
\begin{proof}
	If $\|x- h\|^2 \leq \zeta \|x - y\|^2$, then we have 
	$$
	\sqnorm{\cC_{h,y}(x) - x } = \sqnorm{h-x} \le \zeta \sqnorm{x-y}.
	$$
	Otherwise, 
	$$
	\sqnorm{\cC_{h,y}(x) - x } = \sqnorm{x - x} = 0 \leq \zeta \sqnorm{x - y}.
	$$
\end{proof}

Therefore, \algname{LAG} fits our framework. Using our general analysis (Theorems~\ref{thm:main_thm_gen_non_cvx} and \ref{thm:main_thm_PL}) we derive the following result.

\begin{theorem}\label{thm:lag} 
	\algname{LAG} is a special case of the method from \eqref{eq:CGD_1}--\eqref{eq:CGD_2} with $\cC_{h,y}(x)$ defined in \eqref{eq:lag_compressor} and $\thetaNEW =  1$ and $\betaNEW = \zeta$.
	\begin{enumerate}		
		\item If Assumptions~\ref{ass:diff},~\ref{as:L_smoothness},~\ref{as:L_+} hold and the stepsize $\gamma$ satisfies $0 \leq \gamma \leq \nicefrac{1}{M}$, where $M = L_{-} + L_{+}\sqrt{\zeta}$, then for any $T \ge 0$ we have
		\begin{equation}
			\Exp{\norm{ \nabla f(\hat x^T)}^2} \leq \frac{2 \Delta^0}{\gamma T} + \fr{\Exp{G^0}}{T}, \label{eq:lag_gen_non_cvx}
		\end{equation}
		where $\hat x^T$ is sampled uniformly at random from the points $\{x^0, x^1, \ldots, x^{T-1}\}$ produced by \algname{LAG}, $\Delta^0 = f(x^0) - f^{\inf}$, and $G^0$ is defined in \eqref{eq:G^t}. 
		\item If additionaly Assumption~\ref{as:PL} hold and $0 \leq \gamma \leq \nicefrac{1}{M}$ for $M = \max\left\{L_{-} + L_{+}\sqrt{2\zeta}, \nicefrac{1}{2\mu}\right\}$, then for any $T \ge 0$ we have
		\begin{equation}
			\Exp{f(x^T) - f(x^*)} \leq \left(1 - \gamma\mu\right)^T\left(\Delta^0 +\gamma\Exp{G^{0}}\right). \label{eq:lag_PL}
		\end{equation}
	\end{enumerate}
\end{theorem}

Using this and Corollaries~\ref{cor:main_cor_gen_non_cvx}, \ref{cor:main_cor_PL}, we get the following complexity results.
\begin{corollary}\label{cor:lag}
	\begin{enumerate}
		\item Let the assumptions from the first part of Theorem~\ref{thm:lag} hold, and
		\begin{equation*}
			\gamma = \frac{1}{L_{-} + L_{+}\sqrt{\zeta}}.
		\end{equation*}
		Then for any $T > 1$ we have
		\begin{equation}
			\Exp{\norm{ \nabla f(\hat x^T)}^2} \leq \frac{2 \Delta^0 (L_{-} + L_{+}\sqrt{\zeta})}{ T} + \fr{\Exp{G^0}}{T},\notag
		\end{equation}
		i.e., to achieve $\Exp{\norm{ \nabla f(\hat x^T)}^2} \leq \varepsilon^2$ for some $\varepsilon > 0$ the method requires
		\begin{equation}
			T = \cO\left(\frac{\Delta^0(L_{-} + L_{+}\sqrt{\zeta})}{\varepsilon^2} + \fr{\Exp{G^0}}{ \varepsilon^2}\right) \label{eq:lag_complexity_gen_non_cvx}
		\end{equation}
		iterations/communication rounds.
		\item Let the assumptions from the second part of Theorem~\ref{thm:lag} hold and
		\begin{equation*}
			\gamma = \min\left\{\frac{1}{L_{-} + L_{+}\sqrt{\zeta}}, \frac{1}{2\mu}\right\}.
		\end{equation*}
		Then to achieve $\Exp{f(x^T) - f(x^*)} \leq \varepsilon$ for some $\varepsilon > 0$ the method requires
		\begin{equation}
			\cO\left(\frac{L_{-} + L_{+}\sqrt{\zeta}}{\mu}\log \frac{\Delta^0 + \Exp{G^{0}}\gamma}{\varepsilon}\right) \label{eq:lag_complexity_PL}
		\end{equation}
		iterations/communication rounds.
	\end{enumerate}
\end{corollary}

\begin{proof}
	Both claims are straight-forward applications of Corollaries~\ref{cor:main_cor_gen_non_cvx}, \ref{cor:main_cor_PL}. In the second claim, we used that
	$$
	\frac{L_{-} + L_{+}\sqrt{\zeta}}{\mu} \geq \frac{L_{-}}{\mu} \geq 1,
	$$
	where the second inequality holds since $L_{-} \geq \mu$~\cite{nesterov2018lectures}.
\end{proof}

\newpage

\subsection{CLAG: Compressed Lazily Aggregated Gradient (NEW)}

\begin{algorithm}[h]
   \caption{\algname{CLAG}: Compressed Lazily Aggregated Gradient}\label{alg:clag}
\begin{algorithmic}[1]
   \STATE {\bfseries Input:} starting point $x^0$, stepsize $\gamma$, number of iterations $T$, starting vectors $g_i^0$, $i \in [n]$, trigger parameter $\zeta > 0$
   \FOR{$t=0,1,\ldots,T-1$}
   \STATE Broadcast $g^t$ to all workers
   \FOR{$i = 1,\ldots,n$ in parallel} 
   \STATE $x^{t+1} = x^t - \gamma g^t$
   \STATE Set $g_i^{t+1} = \begin{cases} g_i^t + \cC\left(\nabla f_i(x^{t+1}) - g_i^t\right),& \text{if } \|\nabla f_i(x^{t+1}) - g_i^t\|^2 > \zeta \|\nabla f_i(x^{t+1}) - \nabla f_i(x^t)\|^2,\\ g_i^t,& \text{otherwise} \end{cases}$ 
   \ENDFOR
   \STATE $g^{t+1} = \tfrac{1}{n}\sum_{i=1}^ng_i^{t+1}$
   \ENDFOR
   \STATE {\bfseries Return:} $\hat x^T$ chosen uniformly at random from $\{x^t\}_{t=0}^{T-1}$
\end{algorithmic}
\end{algorithm}

The next lemma shows that \algname{CLAG} uses a special three points compressor.
\begin{lemma}\label{lem:clag}
	The compressor
	\begin{eqnarray}
	\cC_{h,y}(x) \eqdef \begin{cases} h + \cC(x - h),& \text{if } \|x- h\|^2 > \zeta \|x - y\|^2,\\ h,& \text{otherwise,} \end{cases}  \label{eq:clag_compressor}
	\end{eqnarray}
	satisfies \eqref{eq:ttp} with $\thetaNEW \eqdef 1 - (1-\alpha)(1+s)$ and $\betaNEW \eqdef \max\left\{(1-\alpha)\rb{1+ s^{-1} }, \zeta\right\}$, where $s> 0$ is such that  $(1-\alpha)\rb{1 + s} < 1$.
\end{lemma}
\begin{proof}
	First of all, if $\|x- h\|^2 \leq \zeta \|x - y\|^2$, then we have
	\begin{eqnarray*}
		\Exp{\sqnorm{\cC_{h,y}(x) - x }} &=& \sqnorm{h-x} \le \zeta \sqnorm{x-y}.
	\end{eqnarray*}
	Next, if $\|x- h\|^2 > \zeta \|x - y\|^2$, then using the definition of $\cC_{h,y}(x)$ and $\cC$, we derive
	\begin{eqnarray*}
		\Exp{\sqnorm{\cC_{h,y}(x) - x }} &=& \Exp{\sqnorm{\cC(x - h) - (x - h) }}\\
		&\le& (1-\alpha) \sqnorm{x-h}\\
		&=& (1-\alpha) \sqnorm{(x - y) + (y - h)}\\
		&\le& (1-\alpha)\rb{1+{s}}\sqnorm{h - y} +  (1-\alpha)\rb{1+s^{-1}} \sqnorm{x - y}.
	\end{eqnarray*}
\end{proof}

Therefore, \algname{CLAG} fits our framework. Using our general analysis (Theorems~\ref{thm:main_thm_gen_non_cvx} and \ref{thm:main_thm_PL}) we derive the following result.

\begin{theorem}\label{thm:clag} 
	\algname{CLAG} is a special case of the method from \eqref{eq:CGD_1}--\eqref{eq:CGD_2} with $\cC_{h,y}(x)$ defined in \eqref{eq:clag_compressor} and $\thetaNEW =  \alpha - s(1-\alpha)$ and $\betaNEW = \max\left\{(1-\alpha)\rb{1+ s^{-1} }, \zeta\right\}$, where $s> 0$ is such that $(1-\alpha)(1+s) < 1$.	
	\begin{enumerate}		
	\item If Assumptions~\ref{ass:diff},~\ref{as:L_smoothness},~\ref{as:L_+} hold and the stepsize $\gamma$ satisfies $0 \leq \gamma \leq \nicefrac{1}{M}$, where $M = L_{-} + L_{+}\sqrt{\nicefrac{\max\left\{(1-\alpha)\rb{1+ s^{-1} }, \zeta\right\}}{(\alpha - s(1-\alpha))}}$, then for any $T \ge 0$ we have
	\begin{equation}
		\Exp{\norm{ \nabla f(\hat x^T)}^2} \leq \frac{2 \Delta^0}{\gamma T} + \fr{\Exp{G^0}}{(\alpha - s(1-\alpha)) T}, \label{eq:clag_gen_non_cvx}
	\end{equation}
	where $\hat x^T$ is sampled uniformly at random from the points $\{x^0, x^1, \ldots, x^{T-1}\}$ produced by \algname{CLAG}, $\Delta^0 = f(x^0) - f^{\inf}$, and $G^0$ is defined in \eqref{eq:G^t}. 
	\item If additionaly Assumption~\ref{as:PL} hold and $0 \leq \gamma \leq \nicefrac{1}{M}$ for $M = \max\left\{L_{-} + L_{+}\sqrt{\nicefrac{2\max\left\{(1-\alpha)\rb{1+ s^{-1} }, \zeta\right\}}{(\alpha - s(1-\alpha))}}, \nicefrac{(\alpha - s(1-\alpha))}{2\mu}\right\}$, then for any $T \ge 0$ we have
	\begin{equation}
		\Exp{f(x^T) - f(x^*)} \leq \left(1 - \gamma\mu\right)^T\left(\Delta^0 + \frac{\gamma}{\alpha - s(1-\alpha)}\Exp{G^{0}}\right). \label{eq:clag_PL}
	\end{equation}
	\end{enumerate}
\end{theorem}

Using this and Corollaries~\ref{cor:main_cor_gen_non_cvx}, \ref{cor:main_cor_PL}, we get the following complexity results.
\begin{corollary}\label{cor:clag}
	\begin{enumerate}
		\item Let the assumptions from the first part of Theorem~\ref{thm:clag} hold, $s = s_* = -1 + \sqrt{\nicefrac{1}{(1-\alpha)}}$, and
	\begin{equation*}
		\gamma = \frac{1}{L_{-} + L_{+}\sqrt{\max\left\{\tfrac{(1-\alpha)}{(1 - \sqrt{1-\alpha})^2}, \tfrac{\zeta}{1 - \sqrt{1-\alpha}}\right\}}}.
	\end{equation*}
	Then for any $T$ we have
	\begin{equation}
		\Exp{\norm{ \nabla f(\hat x^T)}^2} \leq \frac{2 \Delta^0\left(L_{-} + L_{+}\sqrt{\max\left\{\tfrac{(1-\alpha)}{(1 - \sqrt{1-\alpha})^2}, \tfrac{\zeta}{1 - \sqrt{1-\alpha}}\right\}}\right)}{ T} + \fr{\Exp{G^0}}{(1 - \sqrt{1-\alpha}) T},\notag
	\end{equation}
	i.e., to achieve $\Exp{\norm{ \nabla f(\hat x^T)}^2} \leq \varepsilon^2$ for some $\varepsilon > 0$ the method requires
	\begin{equation}
		T = \cO\left(\frac{\Delta^0\left(L_{-} + L_{+}\sqrt{\max\left\{\tfrac{(1-\alpha)}{\alpha^2}, \tfrac{\zeta}{\alpha}\right\}}\right)}{\varepsilon^2} + \fr{\Exp{G^0}}{\alpha \varepsilon^2}\right) \label{eq:clag_complexity_gen_non_cvx}
	\end{equation}
	iterations/communication rounds.
	\item Let the assumptions from the second part of Theorem~\ref{thm:clag} hold and
	\begin{equation*}
		\gamma = \min\left\{\frac{1}{L_{-} + L_{+}\sqrt{\max\left\{\tfrac{2(1-\alpha)}{(1 - \sqrt{1-\alpha})^2}, \tfrac{\zeta}{1 - \sqrt{1-\alpha}}\right\}}}, \frac{1-\sqrt{1-\alpha}}{2\mu}\right\}.
	\end{equation*}
	Then to achieve $\Exp{f(x^T) - f(x^*)} \leq \varepsilon$ for some $\varepsilon > 0$ the method requires
	\begin{equation}
		\cO\left(\max\left\{\frac{L_{-} + L_{+}\sqrt{\max\left\{\tfrac{(1-\alpha)}{\alpha^2}, \tfrac{\zeta}{\alpha}\right\}}}{\mu}, \alpha\right\}\log \frac{\Delta^0 + \Exp{G^{0}}\nicefrac{\gamma}{\alpha}}{\varepsilon}\right) \label{eq:clag_complexity_PL}
	\end{equation}
	iterations/communication rounds.
	\end{enumerate}
\end{corollary}

\newpage
\subsection{3PCv1 (NEW)}
Out of theoretical curiosity, we consider the following theoretical method.

\begin{algorithm}[h]
   \caption{Error Feedback 2021 -- gradient shift version (\algname{3PCv1})}\label{alg:b_diana_2}
\begin{algorithmic}[1]
   \STATE {\bfseries Input:} starting point $x^0$, stepsize $\gamma$, number of iterations $T$, starting vectors $g_i^0$, $i \in [n]$
   \FOR{$t=0,1,\ldots,T-1$}
   \STATE Broadcast $g^t$ to all workers
   \FOR{$i = 1,\ldots,n$ in parallel} 
   \STATE $x^{t+1} = x^t - \gamma g^t$
   \STATE Set $g_i^{t+1} = \nabla f_i(x^t) + \cC(\nabla f_i(x^{t+1}) - \nabla f_i(x^t))$ 
   \ENDFOR
   \STATE $g^{t+1} = \tfrac{1}{n}\sum_{i=1}^ng_i^{t+1}$
   \ENDFOR
   \STATE {\bfseries Return:} $\hat x^T$ chosen uniformly at random from $\{x^t\}_{t=0}^{T-1}$
\end{algorithmic}
\end{algorithm}

\algname{3PCv1} is impractical since the the compression does not help to reduce the cost of one iteration. Indeed, the server does not know the shifts $\nabla f_i(x^t)$ and the workers have to send them as well at each iteration.

Nevertheless, one can consider \algname{3PCv1} as an ideal version of \algname{EF21}. To illustrate that we derive the following lemma.
\begin{lemma}\label{lem:biased_diana_2}
	The compressor
	\begin{eqnarray}
	\cC_{h,y}(x) \eqdef y + \cC(x - y),  \label{eq:b_diana_2_compressor}
	\end{eqnarray}
	satisfies \eqref{eq:ttp} with $\thetaNEW \eqdef 1$ and $\betaNEW \eqdef 1-\alpha$.
\end{lemma}
\begin{proof}
	By definition of $\cC_{h,y}(x)$ and $\cC$ we have
	\begin{eqnarray*}
		\Exp{\sqnorm{\cC_{h,y}(x) - x }} &=& \Exp{\sqnorm{\cC(x - y) - (x - y) }}\\
		&\le& (1-\alpha) \sqnorm{x-y}.
	\end{eqnarray*}
\end{proof}

Therefore, \algname{3PCv1} fits our framework. Using our general analysis (Theorems~\ref{thm:main_thm_gen_non_cvx} and \ref{thm:main_thm_PL}) we derive the following result.

\begin{theorem}\label{thm:b_diana_2} 
	\algname{3PCv1} is a special case of the method from \eqref{eq:CGD_1}--\eqref{eq:CGD_2} with $\cC_{h,y}(x)$ defined in \eqref{eq:b_diana_2_compressor} and $\thetaNEW =  1$ and $\betaNEW = 1-\alpha$.	
	\begin{enumerate}		
	\item If Assumptions~\ref{ass:diff},~\ref{as:L_smoothness},~\ref{as:L_+} hold and the stepsize $\gamma$ satisfies $0 \leq \gamma \leq \nicefrac{1}{M}$, where $M = L_{-} + L_{+}\sqrt{1-\alpha}$, then for any $T \ge 0$ we have
	\begin{equation}
		\Exp{\norm{ \nabla f(\hat x^T)}^2} \leq \frac{2 \Delta^0}{\gamma T} + \fr{\Exp{G^0}}{T}, \label{eq:b_diana_2_gen_non_cvx}
	\end{equation}
	where $\hat x^T$ is sampled uniformly at random from the points $\{x^0, x^1, \ldots, x^{T-1}\}$ produced by \algname{3PCv1}, $\Delta^0 = f(x^0) - f^{\inf}$, and $G^0$ is defined in \eqref{eq:G^t}. 
	\item If additionaly Assumption~\ref{as:PL} hold and $0 \leq \gamma \leq \nicefrac{1}{M}$ for $M = \max\left\{L_{-} + L_{+}\sqrt{2(1-\alpha)}, \nicefrac{1}{2\mu}\right\}$, then for any $T \ge 0$ we have
	\begin{equation}
		\Exp{f(x^T) - f(x^*)} \leq \left(1 - \gamma\mu\right)^T\left(\Delta^0 + \gamma\Exp{G^{0}}\right). \label{eq:b_diana_2_PL}
	\end{equation}
	\end{enumerate}
\end{theorem}

Using this and Corollaries~\ref{cor:main_cor_gen_non_cvx}, \ref{cor:main_cor_PL}, we get the following complexity results.
\begin{corollary}\label{cor:b_diana_2}
	\begin{enumerate}
		\item Let the assumptions from the first part of Theorem~\ref{thm:b_diana_2} hold and
	\begin{equation*}
		\gamma = \frac{1}{L_{-} + L_{+}\sqrt{1-\alpha}}.
	\end{equation*}
	Then for any $T$ we have
	\begin{equation}
		\Exp{\norm{ \nabla f(\hat x^T)}^2} \leq \frac{2 \Delta^0\left(L_{-} + L_{+}\sqrt{1-\alpha}\right)}{ T} + \fr{\Exp{G^0}}{T},\notag
	\end{equation}
	i.e., to achieve $\Exp{\norm{ \nabla f(\hat x^T)}^2} \leq \varepsilon^2$ for some $\varepsilon > 0$ the method requires
	\begin{equation}
		T = \cO\left(\frac{\Delta^0\left(L_{-} + L_{+}\sqrt{1 - \alpha}\right)}{\varepsilon^2} + \fr{\Exp{G^0}}{\varepsilon^2}\right) \label{eq:b_diana_2_complexity_gen_non_cvx}
	\end{equation}
	iterations/communication rounds.
	\item Let the assumptions from the second part of Theorem~\ref{thm:b_diana_2} hold and
	\begin{equation*}
		\gamma = \min\left\{\frac{1}{L_{-} + L_{+}\sqrt{2(1-\alpha)}}, \frac{1}{2\mu}\right\}.
	\end{equation*}
	Then to achieve $\Exp{f(x^T) - f(x^*)} \leq \varepsilon$ for some $\varepsilon > 0$ the method requires
	\begin{equation}
		\cO\left(\frac{L_{-} + L_{+}\sqrt{1 - \alpha}}{\mu}\log \frac{\Delta^0 + \gamma\Exp{G^{0}}}{\varepsilon}\right) \label{eq:b_diana_2_complexity_PL}
	\end{equation}
	iterations/communication rounds.
	\end{enumerate}
\end{corollary}

\newpage
\subsection{3PCv2 (NEW)}

\begin{algorithm}[h]
   \caption{\algname{3PCv2}}\label{alg:anna}
\begin{algorithmic}[1]
   \STATE {\bfseries Input:} starting point $x^0$, stepsize $\gamma$, number of iterations $T$, starting vectors $g_i^0$, $i \in [n]$
   \FOR{$t=0,1,\ldots,T-1$}
   \STATE Broadcast $g^t$ to all workers
   \FOR{$i = 1,\ldots,n$ in parallel} 
   \STATE $x^{t+1} = x^t - \gamma g^t$
   \STATE Compute $b_i^t = g_i^t + \cQ(\nabla f_i(x^{t+1}) - \nabla f_i(x^t))$
   \STATE Set $g_i^{t+1} = b_i^t + \cC\left(\nabla f_i(x^{t+1}) - b_i^t\right)$ 
   \ENDFOR
   \STATE $g^{t+1} = \tfrac{1}{n}\sum_{i=1}^ng_i^{t+1}$
   \ENDFOR
   \STATE {\bfseries Return:} $\hat x^T$ chosen uniformly at random from $\{x^t\}_{k=0}^{T-1}$
\end{algorithmic}
\end{algorithm}

\begin{lemma}\label{lem:anna}
	The compressor 
	\begin{eqnarray}
	\cC_{h,y}(x) \eqdef b + \cC\left(x - b \right),\quad \text{where}\quad b = h + \cQ(x-y), \label{eq:anna_compressor}
	\end{eqnarray}
	satisfies \eqref{eq:ttp} with $\thetaNEW \eqdef \alpha$ and $\betaNEW \eqdef (1 - \alpha)\omega$.
\end{lemma}
\begin{proof}
	By definition of $\cC_{h,y}(x)$, $\cC$, and $\cQ$ we have
	\begin{eqnarray*}
		\Exp{\sqnorm{\cC_{h,y}(x) - x }} &=& \Exp{\Exp{\sqnorm{\cC_{h,y}(x) -x } \;|\; b}}\\
		&=& \Exp{\Exp{\sqnorm{b + \cC(x - b) -x } \;|\; b}}\\
		&\leq&  \Exp{(1-\alpha)\sqnorm{x - b} }\\
		&=&(1-\alpha) \Exp{\sqnorm{\cQ(x-y) - (x-h)} }\\
		&=& (1-\alpha) \left[ \Exp{\sqnorm{\cQ(x-y) - (x-y)} }  + \sqnorm{h-y} \right]\\
		&\leq & (1-\alpha)\sqnorm{h-y} + (1-\alpha) \omega \sqnorm{x-y} .
	\end{eqnarray*}
\end{proof}

Therefore, \algname{3PCv2} fits our framework. Before we formulate the main results for \algname{3PCv2}, we make several remarks on the proposed method. First of all, with \algname{3PCv2}, we need to communicate {\em two } compressed vectors: $\cQ(x-y)$ and $ \cC\left(x - (h + \cQ(x-y)) \right)$. This is similar to how the induced compressor works \citep{horvath2020better}, but \algname{3PCv2} compressor is {\em not} unbiased. If we set $\cQ\equiv 0$ (this compressor is not unbiased, so the above formulas for $\thetaNEW$ and $\betaNEW$ do not apply) and allow $\cC$ to be arbitrary, we obtain \algname{EF21} \citep{EF21}. Next, if we set \begin{equation}\label{eq:Bernoulli}\cC(x) = \begin{cases} x & \text{with probability} \quad p \\0 & \text{with probability} \quad 1-p \end{cases},\end{equation} 	and allow $\cQ$ to be arbitrary, we obtain \algname{MARINA} \citep{gorbunov2021marina}. Note that  $\cC$ defined above is biased since $\Exp{\cC(x)} = px$, and  the variance inequality is satisfied as an identity: $\Exp{\sqnorm{\cC(x)-x}} = (1-p)\sqnorm{x}$. By choosing a different biased compressor $\cC$, e.g., Top-$K$, we obtain a new variant of \algname{MARINA}. In particular, unlike \algname{MARINA}, this compressor never needs to communicate full gradients, which can be important in some cases.

Using our general analysis (Theorems~\ref{thm:main_thm_gen_non_cvx} and \ref{thm:main_thm_PL}) we derive the following result.

\begin{theorem}\label{thm:anna} 
	\algname{3PCv2} is a special case of the method from \eqref{eq:CGD_1}--\eqref{eq:CGD_2} with $\cC_{h,y}(x)$ defined in \eqref{eq:anna_compressor} and $\thetaNEW =  \alpha$ and $\betaNEW = (1-\alpha)\omega$.	
	\begin{enumerate}		
	\item If Assumptions~\ref{ass:diff},~\ref{as:L_smoothness},~\ref{as:L_+} hold and the stepsize $\gamma$ satisfies $0 \leq \gamma \leq \nicefrac{1}{M}$, where $M = L_{-} + L_{+}\sqrt{\nicefrac{(1-\alpha)\omega}{\alpha}}$, then for any $T \ge 0$ we have
	\begin{equation}
		\Exp{\norm{ \nabla f(\hat x^T)}^2} \leq \frac{2 \Delta^0}{\gamma T} + \fr{\Exp{G^0}}{\alpha T}, \label{eq:anna_gen_non_cvx}
	\end{equation}
	where $\hat x^T$ is sampled uniformly at random from the points $\{x^0, x^1, \ldots, x^{T-1}\}$ produced by \algname{3PCv2}, $\Delta^0 = f(x^0) - f^{\inf}$, and $G^0$ is defined in \eqref{eq:G^t}. 
	\item If additionaly Assumption~\ref{as:PL} hold and $0 \leq \gamma \leq \nicefrac{1}{M}$ for $M = \max\left\{L_{-} + L_{+}\sqrt{\nicefrac{2(1-\alpha)\omega}{\alpha}}, \nicefrac{\alpha}{2\mu}\right\}$, then for any $T \ge 0$ we have
	\begin{equation}
		\Exp{f(x^T) - f(x^*)} \leq \left(1 - \gamma\mu\right)^T\left(\Delta^0 + \frac{\gamma}{\alpha}\Exp{G^{0}}\right). \label{eq:anna_PL}
	\end{equation}
	\end{enumerate}
\end{theorem}

Using this and Corollaries~\ref{cor:main_cor_gen_non_cvx}, \ref{cor:main_cor_PL}, we get the following complexity results.
\begin{corollary}\label{cor:anna}
	\begin{enumerate}
		\item Let the assumptions from the first part of Theorem~\ref{thm:anna} hold and
	\begin{equation*}
		\gamma = \frac{1}{L_{-} + L_{+}\sqrt{\nicefrac{(1-\alpha)\omega}{\alpha}}}.
	\end{equation*}
	Then for any $T$ we have
	\begin{equation}
		\Exp{\norm{ \nabla f(\hat x^T)}^2} \leq \frac{2 \Delta^0\left(L_{-} + L_{+}\sqrt{\nicefrac{(1-\alpha)\omega}{\alpha}}\right)}{ T} + \fr{\Exp{G^0}}{\alpha T},\notag
	\end{equation}
	i.e., to achieve $\Exp{\norm{ \nabla f(\hat x^T)}^2} \leq \varepsilon^2$ for some $\varepsilon > 0$ the method requires
	\begin{equation}
		T = \cO\left(\frac{\Delta^0\left(L_{-} + L_{+}\sqrt{\nicefrac{(1-\alpha)\omega}{\alpha}}\right)}{\varepsilon^2} + \fr{\Exp{G^0}}{\alpha \varepsilon^2}\right) \label{eq:anna_complexity_gen_non_cvx}
	\end{equation}
	iterations/communication rounds.
	\item Let the assumptions from the second part of Theorem~\ref{thm:anna} hold and
	\begin{equation*}
		\gamma = \min\left\{\frac{1}{L_{-} + L_{+}\sqrt{\nicefrac{2(1-\alpha)\omega}{\alpha}}}, \frac{\alpha}{2\mu}\right\}.
	\end{equation*}
	Then to achieve $\Exp{f(x^T) - f(x^*)} \leq \varepsilon$ for some $\varepsilon > 0$ the method requires
	\begin{equation}
		\cO\left(\max\left\{\frac{L_{-} + L_{+}\sqrt{\nicefrac{2(1-\alpha)\omega}{\alpha}}}{\mu}, \alpha\right\}\log \frac{\Delta^0 + \Exp{G^{0}}\nicefrac{\gamma}{\alpha}}{\varepsilon}\right) \label{eq:anna_complexity_PL}
	\end{equation}
	iterations/communication rounds.
	\end{enumerate}
\end{corollary}

\newpage
\subsection{3PCv3 (NEW)}	
In this section, we introduce a new method called \algname{3PCv3}. It can be seen as a combination of any \algname{3PC} compressor with some biased compressor. We also notice that \algname{3PCv2} cannot be obtained as a special case of \algname{3PCv3} as $h +\cQ(x -y)$ does not satisfy \eqref{eq:ttp}.

\begin{algorithm}[h]
   \caption{\algname{3PCv3}}\label{alg:jeanne}
\begin{algorithmic}[1]
   \STATE {\bfseries Input:} starting point $x^0$, stepsize $\gamma$, number of iterations $T$, starting vectors $g_i^0$, $i \in [n]$
   \FOR{$t=0,1,\ldots,T-1$}
   \STATE Broadcast $g^t$ to all workers
   \FOR{$i = 1,\ldots,n$ in parallel} 
   \STATE $x^{t+1} = x^t - \gamma g^t$
   \STATE Compute $b_i^t = \cC_{g_i^t, \nabla f_i(x^t)}^1(\nabla f_i(x^{t+1}))$
   \STATE Set $g_i^{t+1} = b_i^t + \cC\left(\nabla f_i(x^{t+1}) - b_i^t\right)$ 
   \ENDFOR
   \STATE $g^{t+1} = \tfrac{1}{n}\sum_{i=1}^ng_i^{t+1}$
   \ENDFOR
   \STATE {\bfseries Return:} $\hat x^T$ chosen uniformly at random from $\{x^t\}_{t=0}^{T-1}$
\end{algorithmic}
\end{algorithm}

\begin{lemma}\label{lem:jeanne}
	Consider the compressor defined as 
	\begin{eqnarray}
	\cC_{h,y}(x) \eqdef b + \cC(x - b) ,\quad \text{where}\quad b = \cC^1_{h,y}(x) \label{eq:jeanne_compressor}
	\end{eqnarray}
	and $\cC^1_{h,y}(x)$ satisfies \eqref{eq:ttp} with some $\thetaNEW_1$ and $\betaNEW_1$. Then $\cC_{h,y}(x)$ satisfies \eqref{eq:ttp} with $\thetaNEW \eqdef 1 - (1 -\al)(1 -\thetaNEW_1)$ and $\betaNEW \eqdef (1 - \alpha)\betaNEW_1$.
\end{lemma}
\begin{proof}
	By definition of $\cC_{h,y}^1(x)$ and $\cC$ we have
	\begin{eqnarray*}
		\Exp{\sqnorm{\cC_{h,y}(x) - x }} &=& \Exp{\Exp{\sqnorm{b + \cC(x-b) - x }\mid b}}\\
		&\leq& (1-\alpha)\Exp{\sqnorm{x-b}}\\
		&=& (1-\alpha)\Exp{\sqnorm{\cC_{h,y}^1(x) - x}}\\
		&\leq& (1-\alpha)(1-\thetaNEW_1)\|h-y\|^2 + (1-\alpha)\betaNEW_1\|x-y\|^2.
	\end{eqnarray*}
\end{proof}

Therefore, \algname{3PCv3} fits our framework. Using our general analysis (Theorems~\ref{thm:main_thm_gen_non_cvx} and \ref{thm:main_thm_PL}) we derive the following result.

\begin{theorem}\label{thm:jeanne} 
	\algname{3PCv3} is a special case of the method from \eqref{eq:CGD_1}--\eqref{eq:CGD_2} with $\cC_{h,y}(x)$ defined in \eqref{eq:jeanne_compressor} and $\thetaNEW \eqdef 1 - (1 -\al)(1 -\thetaNEW_1)$ and $\betaNEW \eqdef (1 - \alpha)\betaNEW_1$.	
	\begin{enumerate}		
	\item If Assumptions~\ref{ass:diff},~\ref{as:L_smoothness},~\ref{as:L_+} hold and the stepsize $\gamma$ satisfies $0 \leq \gamma \leq \nicefrac{1}{M}$, where $M = L_{-} + L_{+}\sqrt{\nicefrac{(1 - \alpha)\betaNEW_1}{(1 - (1 -\al)(1 -\thetaNEW_1))}}$, then for any $T \ge 0$ we have
	\begin{equation}
		\Exp{\norm{ \nabla f(\hat x^T)}^2} \leq \frac{2 \Delta^0}{\gamma T} + \fr{\Exp{G^0}}{(1 - (1 -\al)(1 -\thetaNEW_1)) T}, \label{eq:jeanne_gen_non_cvx}
	\end{equation}
	where $\hat x^T$ is sampled uniformly at random from the points $\{x^0, x^1, \ldots, x^{T-1}\}$ produced by \algname{3PCv3}, $\Delta^0 = f(x^0) - f^{\inf}$, and $G^0$ is defined in \eqref{eq:G^t}. 
	\item If additionaly Assumption~\ref{as:PL} hold and $0 \leq \gamma \leq \nicefrac{1}{M}$ for $M = \max\left\{L_{-} + L_{+}\sqrt{\nicefrac{2(1 - \alpha)\betaNEW_1}{(1 - (1 -\al)(1 -\thetaNEW_1))}}, \nicefrac{1 - (1 -\al)(1 -\thetaNEW_1)}{2\mu}\right\}$, then for any $T \ge 0$ we have
	\begin{equation}
		\Exp{f(x^T) - f(x^*)} \leq \left(1 - \gamma\mu\right)^T\left(\Delta^0 + \frac{\gamma}{1 - (1 -\al)(1 -\thetaNEW_1)}\Exp{G^{0}}\right). \label{eq:jeanne_PL}
	\end{equation}
	\end{enumerate}
\end{theorem}

Using this and Corollaries~\ref{cor:main_cor_gen_non_cvx}, \ref{cor:main_cor_PL}, we get the following complexity results.
\begin{corollary}\label{cor:jeanne}
	\begin{enumerate}
		\item Let the assumptions from the first part of Theorem~\ref{thm:jeanne} hold and
	\begin{equation*}
		\gamma = \frac{1}{L_{-} + L_{+}\sqrt{\frac{(1 - \alpha)\betaNEW_1}{1 - (1 -\al)(1 -\thetaNEW_1)}}}.
	\end{equation*}
	Then for any $T$ we have
	\begin{equation}
		\Exp{\norm{ \nabla f(\hat x^T)}^2} \leq \frac{2 \Delta^0\left(L_{-} + L_{+}\sqrt{\frac{(1 - \alpha)\betaNEW_1}{1 - (1 -\al)(1 -\thetaNEW_1)}}\right)}{ T} + \fr{\Exp{G^0}}{(1 - (1 -\al)(1 -\thetaNEW_1)) T},\notag
	\end{equation}
	i.e., to achieve $\Exp{\norm{ \nabla f(\hat x^T)}^2} \leq \varepsilon^2$ for some $\varepsilon > 0$ the method requires
	\begin{equation}
		T = \cO\left(\frac{\Delta^0\left(L_{-} + L_{+}\sqrt{\nicefrac{(1-\alpha)\omega}{\alpha}}\right)}{\varepsilon^2} + \fr{\Exp{G^0}}{(1 - (1 -\al)(1 -\thetaNEW_1)) \varepsilon^2}\right) \label{eq:jeanne_complexity_gen_non_cvx}
	\end{equation}
	iterations/communication rounds.
	\item Let the assumptions from the second part of Theorem~\ref{thm:jeanne} hold and
	\begin{equation*}
		\gamma = \min\left\{\frac{1}{L_{-} + L_{+}\sqrt{\frac{2(1 - \alpha)\betaNEW_1}{1 - (1 -\al)(1 -\thetaNEW_1)}}}, \frac{1 - (1 -\al)(1 -\thetaNEW_1)}{2\mu}\right\}.
	\end{equation*}
	Then to achieve $\Exp{f(x^T) - f(x^*)} \leq \varepsilon$ for some $\varepsilon > 0$ the method requires
	\begin{equation}
		\cO\left(\max\left\{\frac{L_{-} + L_{+}\sqrt{\frac{(1 - \alpha)\betaNEW_1}{1 - (1 -\al)(1 -\thetaNEW_1)}}}{\mu}, 1 - (1 -\al)(1 -\thetaNEW_1)\right\}\log \frac{\Delta^0 + \Exp{G^{0}}\tfrac{\gamma}{1 - (1 -\al)(1 -\thetaNEW_1)}}{\varepsilon}\right) \label{eq:jeanne_complexity_PL}
	\end{equation}
	iterations/communication rounds.
	\end{enumerate}
\end{corollary}

\newpage
\subsection{3PCv4 (NEW)}
We now present another special case of \algname{3PC} compressor -- \algname{3PCv4}. This compressor can be seen as modification of \algname{3PCv2} that uses only biased compression operators.

\begin{algorithm}[h]
   \caption{\algname{3PCv4}}\label{alg:jacqueline}
\begin{algorithmic}[1]
   \STATE {\bfseries Input:} starting point $x^0$, stepsize $\gamma$, number of iterations $T$, starting vectors $g_i^0$, $i \in [n]$
   \FOR{$t=0,1,\ldots,T-1$}
   \STATE Broadcast $g^t$ to all workers
   \FOR{$i = 1,\ldots,n$ in parallel} 
   \STATE $x^{t+1} = x^t - \gamma g^t$
   \STATE Compute $b_i^t = g_i^t + \cC_{2}(\nabla f_i(x^{t+1}) - g_i^t)$
   \STATE Set $g_i^{t+1} = b_i^t + \cC_1\left(\nabla f_i(x^{t+1}) - b_i^t\right)$ 
   \ENDFOR
   \STATE $g^{t+1} = \tfrac{1}{n}\sum_{i=1}^ng_i^{t+1}$
   \ENDFOR
   \STATE {\bfseries Return:} $\hat x^T$ chosen uniformly at random from $\{x^t\}_{t=0}^{T-1}$
\end{algorithmic}
\end{algorithm}

\begin{lemma}\label{lem:Jacqueline}
	Let $\cC_1$ and $\cC_2$ are the contracive compressors with constants $\al_1$ and $\al_2$ respectively.
	Then the compressor defined as 
	\begin{eqnarray}
	\cC_{h,y}(x) \eqdef h + \cC_2(x - h) + \cC_1\left(x - (h + \cC_2(x - h))) \right), \label{eq:jacqueline_compressor}
	\end{eqnarray}
which  satisfies \eqref{eq:ttp} with $\thetaNEW \eqdef 1-\sqrt{1-\bar{\al}}$ and $\betaNEW \eqdef \frac{1-\bar{\al}}{1-\sqrt{1-\bar{\al}}}$, where $\bar{\alpha} \eqdef 1 - (1 - \al_1)(1 - \al_2)$.
\end{lemma}
\begin{proof}
	Let $a \eqdef h + \cC_2(x - h) $. Then 
	\begin{eqnarray}
	\Exp{\sqnorm{\cC_{h,y}(x) - x }} &=& \Exp{\Exp{\sqnorm{\cC_{h,y}(x) -x } \mid a}}\notag\\
	&=& \Exp{\Exp{\sqnorm{a + \cC_1(x - a) -x } \mid a}}\notag\\
	&\leq&  (1-\al_1)\Exp{\sqnorm{x - a} }\notag\\
	&=&  (1-\al_1)\Exp{\sqnorm{x - \rb{h + \cC_2(x - h)}} }\notag\\
	&\leq&  (1-\al_1)(1-\al_2)\Exp{\sqnorm{x - h} }\notag\\
	&\leq&  (1-\al_1)(1-\al_2)(1 + s)\sqnorm{h - y} + (1-\al_1)(1-\al_2)\rb{1 + s^{-1}}\sqnorm{x - y} \notag\\
	\end{eqnarray}
Optimal $s$ parameter can be found by direct minimization of the fraction (see Lemma \ref{lem:b_diana_technical}) $$\fr{\betaNEW(s)}{\thetaNEW(s)} = \fr{(1 - \bar{\alpha})\rb{1 + s^{-1}}}{1 - (1 -\bar{\al})(1 +s)},$$
where $\bar{\alpha} \eqdef 1 - (1 - \al_1)(1 - \al_2)$.
Using Lemma \ref{lem:b_diana_technical} we finally obtain
$\thetaNEW(s_*) \eqdef 1-\sqrt{1-\bar{\al}}$ and $\betaNEW(s_*) \eqdef \frac{1-\bar{\al}}{1-\sqrt{1-\bar{\al}}}$
\end{proof}

Therefore, \algname{3PCv4} fits our framework. Using our general analysis (Theorems~\ref{thm:main_thm_gen_non_cvx} and \ref{thm:main_thm_PL}) we derive the following result.

\begin{theorem}\label{thm:jacqueline} 
	\algname{3PCv4} is a special case of the method from \eqref{eq:CGD_1}--\eqref{eq:CGD_2} with $\cC_{h,y}(x)$ defined in \eqref{eq:anna_compressor} and $\thetaNEW \eqdef 1-\sqrt{1-\bar{\al}}$ and $\betaNEW \eqdef \frac{1-\bar{\al}}{1-\sqrt{1-\bar{\al}}}$, where $\bar{\alpha} \eqdef 1 - (1 - \al_1)(1 - \al_2)$.	
	\begin{enumerate}		
	\item If Assumptions~\ref{ass:diff},~\ref{as:L_smoothness},~\ref{as:L_+} hold and the stepsize $\gamma$ satisfies $0 \leq \gamma \leq \nicefrac{1}{M}$, where $M = L_{-} + L_{+}\sqrt{\nicefrac{(1-\bar\alpha)}{(1 - \sqrt{1-\bar\alpha})^2}}$, then for any $T \ge 0$ we have
	\begin{equation}
		\Exp{\norm{ \nabla f(\hat x^T)}^2} \leq \frac{2 \Delta^0}{\gamma T} + \fr{\Exp{G^0}}{(1-\sqrt{1-\bar{\al}}) T}, \label{eq:jacqueline_gen_non_cvx}
	\end{equation}
	where $\hat x^T$ is sampled uniformly at random from the points $\{x^0, x^1, \ldots, x^{T-1}\}$ produced by \algname{3PCv4}, $\Delta^0 = f(x^0) - f^{\inf}$, and $G^0$ is defined in \eqref{eq:G^t}. 
	\item If additionaly Assumption~\ref{as:PL} hold and $0 \leq \gamma \leq \nicefrac{1}{M}$ for $M = \max\left\{L_{-} + L_{+}\sqrt{\nicefrac{2(1-\bar\alpha)}{(1 - \sqrt{1-\bar\alpha})^2}}, \nicefrac{1 - \sqrt{1-\bar\alpha}}{2\mu}\right\}$, then for any $T \ge 0$ we have
	\begin{equation}
		\Exp{f(x^T) - f(x^*)} \leq \left(1 - \gamma\mu\right)^T\left(\Delta^0 + \frac{\gamma}{1-\sqrt{1-\bar{\al}}}\Exp{G^{0}}\right). \label{eq:jacqueline_PL}
	\end{equation}
	\end{enumerate}
\end{theorem}

Using this and Corollaries~\ref{cor:main_cor_gen_non_cvx}, \ref{cor:main_cor_PL}, we get the following complexity results.
\begin{corollary}\label{cor:jacqueline}
	\begin{enumerate}
		\item Let the assumptions from the first part of Theorem~\ref{thm:jacqueline} hold and
	\begin{equation*}
		\gamma = \frac{1}{L_{-} + L_{+}\sqrt{\nicefrac{(1-\bar\alpha)}{(1 - \sqrt{1-\bar\alpha})^2}}}.
	\end{equation*}
	Then for any $T$ we have
	\begin{equation}
		\Exp{\norm{ \nabla f(\hat x^T)}^2} \leq \frac{2 \Delta^0\left(L_{-} + L_{+}\sqrt{\nicefrac{(1-\bar\alpha)}{(1 - \sqrt{1-\bar\alpha})^2}}\right)}{ T} + \fr{\Exp{G^0}}{(1 - \sqrt{1-\bar\alpha}) T},\notag
	\end{equation}
	i.e., to achieve $\Exp{\norm{ \nabla f(\hat x^T)}^2} \leq \varepsilon^2$ for some $\varepsilon > 0$ the method requires
	\begin{equation}
		T = \cO\left(\frac{\Delta^0\left(L_{-} + L_{+}\sqrt{\nicefrac{(1 - \bar\alpha)}{\bar\alpha^2}}\right)}{\varepsilon^2} + \fr{\Exp{G^0}}{\bar\alpha \varepsilon^2}\right) \label{eq:jacqueline_complexity_gen_non_cvx}
	\end{equation}
	iterations/communication rounds.
	\item Let the assumptions from the second part of Theorem~\ref{thm:jacqueline} hold and
	\begin{equation*}
		\gamma = \min\left\{\frac{1}{L_{-} + L_{+}\sqrt{\nicefrac{2(1-\bar\alpha)}{(1 - \sqrt{1-\bar\alpha})^2}}}, \frac{1-\sqrt{1-\bar\alpha}}{2\mu}\right\}.
	\end{equation*}
	Then to achieve $\Exp{f(x^T) - f(x^*)} \leq \varepsilon$ for some $\varepsilon > 0$ the method requires
	\begin{equation}
		\cO\left(\max\left\{\frac{L_{-} + L_{+}\sqrt{\nicefrac{(1 - \bar\alpha)}{\bar\alpha^2}}}{\mu}, \bar\alpha\right\}\log \frac{\Delta^0 + \Exp{G^{0}}\nicefrac{\gamma}{\bar\alpha}}{\varepsilon}\right) \label{eq:jacqueline_complexity_PL}
	\end{equation}
	iterations/communication rounds.
	\end{enumerate}
\end{corollary}

\newpage
\subsection{3PCv5 (NEW)}

In this section, we consider a version of \algname{MARINA} that uses biased compression instead of unbiased one.

\begin{algorithm}[h]
   \caption{Biased \algname{MARINA} (\algname{3PCv5})}\label{alg:b_marina}
\begin{algorithmic}[1]
   \STATE {\bfseries Input:} starting point $x^0$, stepsize $\gamma$, probability $p\in(0,1]$, number of iterations $T$, starting vectors $g_i^0$, $i \in [n]$
   \FOR{$t=0,1,\ldots,T-1$}
   \STATE Sample $c_t \sim \text{Be}(p)$
   \STATE Broadcast $g^t$ to all workers
   \FOR{$i = 1,\ldots,n$ in parallel} 
   \STATE $x^{t+1} = x^t - \gamma g^t$
   \STATE Set $g_i^{t+1} = \begin{cases}\nabla f_i(x^{t+1}), & \text{if } c_t = 1,\\ g_i^t + \cC\left(\nabla f_{i}(x^{t+1}) - \nabla f_{i}(x^t))\right), & \text{if } c_t = 0\end{cases}$ 
   \ENDFOR
   \STATE $g^{t+1} = \tfrac{1}{n}\sum_{i=1}^ng_i^{t+1}$
   \ENDFOR
   \STATE {\bfseries Return:} $\hat x^T$ chosen uniformly at random from $\{x^t\}_{t=0}^{T-1}$
\end{algorithmic}
\end{algorithm}

The next lemma shows that \algname{3PCv5} uses a special three points compressor.
\begin{lemma}\label{lem:biased_marina}
	The compressor
	\begin{eqnarray}\label{eq:b_marina_compressor}
	\cC_{h,y}(x) = \begin{cases}x,& \text{w.p.\ } p\\ h + \cC(x - y),& \text{w.p.\ } 1 - p \end{cases}
	\end{eqnarray} 
	satisfies \eqref{eq:ttp} with $\thetaNEW =  p - s(1-p)$ and $\betaNEW = (1-p)\rb{1+s^{-1}}(1-\alpha)$, where $s>0$ is such that $(1-p)(1+s) < 1$.
\end{lemma}
\begin{proof}
By definition of $\cC_{h,y}(x)$ and $\cC$ we have
	\begin{eqnarray*}
		\Exp{\sqnorm{\cC_{h,y}(x) - x }} &\overset{\eqref{eq:b_marina_compressor}}{=}& (1-p) \Exp{\sqnorm{ h + \cC(x - y) - x }}\\
		&=& (1-p) \Exp{\sqnorm{ h - y + \cC(x - y) - (x -y) }}\\
		&\le& (1-p)(1+s) \sqnorm{h - y} + (1-p)\rb{1+ s^{-1} }\Exp{\sqnorm{ \cC(x - y) - (x -y)}}\\
		&\le& (1-p)(1+s) \sqnorm{h - y} + (1-p)\rb{1+ s^{-1} }\rb{1 - \alpha}\sqnorm{x-y},
	\end{eqnarray*}
	where in the third row we use that $\|a+b\|^2 \leq (1+s)\|a\|^2 + (1+s^{-1})\|b\|^2$ for all $s > 0$, $a,b \in \R^d$. Assuming $(1-p)(1+s) < 1$, we get the result.
\end{proof}

Therefore, \algname{3PCv5} fits our framework. Using our general analysis (Theorems~\ref{thm:main_thm_gen_non_cvx} and \ref{thm:main_thm_PL}) we derive the following result.

\begin{theorem}\label{thm:b_marina} 
	\algname{3PCv5} is a special case of the method from \eqref{eq:CGD_1}--\eqref{eq:CGD_2} with $\cC_{h,y}(x)$ defined in \eqref{eq:b_marina_compressor} and $\thetaNEW =  p - s(1-p)$ and $\betaNEW = (1-p)\rb{1+s^{-1}}(1-\alpha)$, where $s> 0$ is such that $(1-p)(1+s) < 1$.	
	\begin{enumerate}		
	\item If Assumptions~\ref{ass:diff},~\ref{as:L_smoothness},~\ref{as:L_+} hold and the stepsize $\gamma$ satisfies $0 \leq \gamma \leq \nicefrac{1}{M}$, where $M = L_{-} + L_{+}\sqrt{\nicefrac{(1-p)\rb{1+s^{-1}}(1-\alpha)}{(p - s(1-p))}}$, then for any $T \ge 0$ we have
	\begin{equation}
		\Exp{\norm{ \nabla f(\hat x^T)}^2} \leq \frac{2 \Delta^0}{\gamma T} + \fr{\Exp{G^0}}{(p - s(1-p)) T}, \label{eq:b_marina_gen_non_cvx}
	\end{equation}
	where $\hat x^T$ is sampled uniformly at random from the points $\{x^0, x^1, \ldots, x^{T-1}\}$ produced by \algname{3PCv5}, $\Delta^0 = f(x^0) - f^{\inf}$, and $G^0$ is defined in \eqref{eq:G^t}. 
	\item If additionaly Assumption~\ref{as:PL} hold and $0 \leq \gamma \leq \nicefrac{1}{M}$ for $M = \max\left\{L_{-} + L_{+}\sqrt{\nicefrac{2(1-p)\rb{1+s^{-1}}(1-\alpha)}{(p - s(1-p))}}, \nicefrac{(p - s(1-p))}{2\mu}\right\}$, then for any $T \ge 0$ we have
	\begin{equation}
		\Exp{f(x^T) - f(x^*)} \leq \left(1 - \gamma\mu\right)^T\left(\Delta^0 + \frac{\gamma}{p - s(1-p)}\Exp{G^{0}}\right). \label{eq:b_marina_PL}
	\end{equation}
	\end{enumerate}
\end{theorem}

Neglecting the term that depends on $G^0$ (for simplicity, one can assume that $g_i^0 = \nabla f_i(x^0)$ for $i \in [n]$), one can notice that the smaller $\nicefrac{\betaNEW}{\thetaNEW}$, the better the rate. Considering $\nicefrac{\betaNEW}{\thetaNEW}$ as a function of $s$ and optimizing this function in $s$, we find the optimal value of this ratio.

\begin{lemma}\label{lem:b_marina_technical}
	The optimal value of
	\begin{equation}
		\frac{\betaNEW}{\thetaNEW}(s) = \frac{(1-p)\rb{1+s^{-1}}(1-\alpha)}{(p - s(1-p))}\notag
	\end{equation}
	under the constraint $0 < s < \nicefrac{p}{(1-p)}$ equals
	\begin{equation}
		\frac{\betaNEW}{\thetaNEW}(s_*) = \frac{(1-p)(1-\alpha)}{(1 - \sqrt{1-p})^2} \leq \frac{4(1-p)(1-\alpha)}{p^2}\notag
	\end{equation}
	and it is achieved at $s^* = -1 + \sqrt{\nicefrac{1}{(1-p)}}$.
\end{lemma}
\begin{proof}
	First of all, we find the derivative of the considered function:
	\begin{eqnarray*}
		\left(\frac{\betaNEW}{\thetaNEW}(s)\right)' &=& (1-p)(1-\alpha)\frac{(1-p)s^2 + 2(1-p)s - p}{(ps - s^2(1-p))^2}.
	\end{eqnarray*}
	The function has 2 critical points: $-1 \pm \sqrt{\nicefrac{1}{(1-p)}}$. Moreover, the derivative is non-positive for $s \in (0, -1 + \sqrt{\nicefrac{1}{(1-p)}}]$ and negative for $s \in (-1 + \sqrt{\nicefrac{1}{(1-p)}}, +\infty)$. This implies that the optimal value on the interval $s \in (0, \nicefrac{p}{(1-p)})$ is achieved at $s_* = -1 + \sqrt{\nicefrac{1}{(1-p)}}$. Via simple computations one can verify that
	\begin{equation*}
		\frac{\betaNEW}{\thetaNEW}(s_*) = \frac{(1-p)(1-\alpha)}{(1 - \sqrt{1-p})^2}.
	\end{equation*}
	Finally, since $1 - \sqrt{1-p} \geq \nicefrac{p}{2}$, we have
	\begin{equation}
		\frac{\betaNEW}{\thetaNEW}(s_*) \leq \frac{4(1-p)(1-\alpha)}{p^2}.\notag
	\end{equation}
\end{proof}

Using this and Corollaries~\ref{cor:main_cor_gen_non_cvx}, \ref{cor:main_cor_PL}, we get the following complexity results.
\begin{corollary}\label{cor:b_marina}
	\begin{enumerate}
		\item Let the assumptions from the first part of Theorem~\ref{thm:b_marina} hold, $s = s_* = -1 + \sqrt{\nicefrac{1}{(1-p)}}$, and
	\begin{equation*}
		\gamma = \frac{1}{L_{-} + L_{+}\sqrt{\nicefrac{(1-p)(1-\alpha)}{(1 - \sqrt{1-p})^2}}}.
	\end{equation*}
	Then for any $T$ we have
	\begin{equation}
		\Exp{\norm{ \nabla f(\hat x^T)}^2} \leq \frac{2 \Delta^0\left(L_{-} + L_{+}\sqrt{\nicefrac{(1-p)(1-\alpha)}{(1 - \sqrt{1-p})^2}}\right)}{ T} + \fr{\Exp{G^0}}{(1 - \sqrt{1-p}) T},\notag
	\end{equation}
	i.e., to achieve $\Exp{\norm{ \nabla f(\hat x^T)}^2} \leq \varepsilon^2$ for some $\varepsilon > 0$ the method requires
	\begin{equation}
		T = \cO\left(\frac{\Delta^0\left(L_{-} + L_{+}\sqrt{\nicefrac{(1-p)(1 - \alpha)}{p^2}}\right)}{\varepsilon^2} + \fr{\Exp{G^0}}{p \varepsilon^2}\right) \label{eq:b_marina_complexity_gen_non_cvx}
	\end{equation}
	iterations/communication rounds.
	\item Let the assumptions from the second part of Theorem~\ref{thm:b_marina} hold and
	\begin{equation*}
		\gamma = \min\left\{\frac{1}{L_{-} + L_{+}\sqrt{\nicefrac{2(1-p)(1-\alpha)}{(1 - \sqrt{1-p})^2}}}, \frac{1-\sqrt{1-p}}{2\mu}\right\}.
	\end{equation*}
	Then to achieve $\Exp{f(x^T) - f(x^*)} \leq \varepsilon$ for some $\varepsilon > 0$ the method requires
	\begin{equation}
		\cO\left(\max\left\{\frac{L_{-} + L_{+}\sqrt{\nicefrac{(1-p)(1 - \alpha)}{p^2}}}{\mu}, p\right\}\log \frac{\Delta^0 + \Exp{G^{0}}\nicefrac{\gamma}{p}}{\varepsilon}\right) \label{eq:b_marina_complexity_PL}
	\end{equation}
	iterations/communication rounds.
	\end{enumerate}
\end{corollary}

\newpage

\section{MARINA}
In this section, we show that \algname{MARINA} \citep{gorbunov2021marina} can be analyzed using a similar proof technique that we use for the methods based on three points compressors.

\begin{algorithm}[h]
   \caption{\algname{MARINA} \citep{gorbunov2021marina}}\label{alg:marina}
\begin{algorithmic}[1]
   \STATE {\bfseries Input:} starting point $x^0$, stepsize $\gamma$, probability $p\in(0,1]$, number of iterations $T$
   \STATE Initialize $g^0 = \nabla f(x^0)$
   \FOR{$t=0,1,\ldots,T-1$}
   \STATE Sample $c_t \sim \text{Be}(p)$
   \STATE Broadcast $g^t$ to all workers
   \FOR{$i = 1,\ldots,n$ in parallel} 
   \STATE $x^{t+1} = x^t - \gamma g^t$
   \STATE Set $g_i^{t+1} = \begin{cases}\nabla f_i(x^{t+1}), & \text{if } c_t = 1,\\ g_i^t + \cQ\left(\nabla f_{i}(x^{t+1}) - \nabla f_{i}(x^t))\right), & \text{if } c_t = 0\end{cases}$ 
   \ENDFOR
   \STATE $g^{t+1} = \tfrac{1}{n}\sum_{i=1}^ng_i^{t+1}$
   \ENDFOR
   \STATE {\bfseries Return:} $\hat x^T$ chosen uniformly at random from $\{x^t\}_{t=0}^{T-1}$
\end{algorithmic}
\end{algorithm}

The next lemma casts \algname{MARINA} to our theoretical framework.

\begin{lemma}\label{lem:main_lem_marina}
	Let Assumption~\ref{as:L_+} hold. Then, \algname{MARINA} satisfies inequality \eqref{eq:key_inequality} with $G^t = \|g^t - \nabla f(x^t)\|^2,$ $\thetaNEW = p$, and $\betaNEW = \nicefrac{(1-p)\omega}{n}$.
\end{lemma}
\begin{proof}
The formula for $g_i^{t+1}$ implies that
\begin{equation*}
	g^{t+1} = \begin{cases}\nabla f(x^{t+1}), & \text{if } c_t = 1,\\ g^t + \frac{1}{n}\sum\limits_{i=1}^n\cQ\left(\nabla f_{i}(x^{t+1}) - \nabla f_{i}(x^t))\right), & \text{if } c_t = 0. \end{cases}
\end{equation*}
Using this and independence of $\cQ\left(\nabla f_{i}(x^{t+1}) - \nabla f_{i}(x^t))\right)$ for $i \in [n]$ and fixed $x^t,x^{t+1}$, we derive
	\begin{eqnarray*}
		\Exp{G^{t+1}} &=& \Exp{\sqnorm{g^{t+1} - \nabla f(x^{t+1})}}\\
		&=& (1-p) \Exp{\sqnorm{ g^t + \frac{1}{n}\sum\limits_{i=1}^n\cQ\left(\nabla f_{i}(x^{t+1}) - \nabla f_{i}(x^t))\right) - \nabla f(x^{t+1})}}\\	
		&=& (1-p) \Exp{\sqnorm{ g^t - \nabla f(x^t) + \frac{1}{n}\sum\limits_{i=1}^n\cQ\left(\nabla f_{i}(x^{t+1}) - \nabla f_{i}(x^t))\right) - \left(\nabla f(x^{t+1}) - \nabla f(x^t)\right)}}\\		
		&\overset{\eqref{eq:unb_compressor}}{=}& (1-p)\Exp{\sqnorm{g^t - \nabla f(x^t)}}\\
		&&\qquad + (1-p) \Exp{\sqnorm{\frac{1}{n}\sum\limits_{i=1}^n\left(\cQ\left(\nabla f_{i}(x^{t+1}) - \nabla f_{i}(x^t))\right) - \left(\nabla f_i(x^{t+1}) - \nabla f_i(x^t)\right)\right)}}\\
		&=& (1-p)\Exp{G^t} + \frac{1-p}{n^2}\sum\limits_{i=1}^n \Exp{\sqnorm{\cQ\left(\nabla f_{i}(x^{t+1}) - \nabla f_{i}(x^t))\right) - \left(\nabla f_i(x^{t+1}) - \nabla f_i(x^t)\right)}}\\
		&\overset{\eqref{eq:unb_compressor}}{=}& (1-p)\Exp{G^t} + \frac{(1-p)\omega}{n^2}\sum\limits_{i=1}^n \Exp{\sqnorm{\nabla f_i(x^{t+1}) - \nabla f_i(x^t)}}.
	\end{eqnarray*}
	It remains to apply Assumption~\ref{as:L_+} to get the result.
\end{proof}

We notice that the proofs of Theorems~\ref{thm:main_thm_gen_non_cvx}~and~\ref{thm:main_thm_PL} rely only on the inequality \eqref{eq:key_inequality}, the update rule $x^{t+1} = x^t - \gamma g^t$, and the fact that $G^t \geq \|g^t - \nabla f(x^t)\|^2.$ Therefore, using Lemma~\ref{lem:main_lem_marina} and our general results (Theorems~\ref{thm:main_thm_gen_non_cvx}~and~\ref{thm:main_thm_PL}), we recover the rates for \algname{MARINA} from \citet{gorbunov2021marina}.

\begin{theorem}\label{thm:marina} 
	\begin{enumerate}		
	\item If Assumptions~\ref{ass:diff},~\ref{as:L_smoothness},~\ref{as:L_+} hold and the stepsize $\gamma$ satisfies $0 \leq \gamma \leq \nicefrac{1}{M}$, where $M = L_{-} + L_{+}\sqrt{\nicefrac{(1-p)\omega}{np}}$, then for any $T \ge 0$ we have
	\begin{equation}
		\Exp{\norm{ \nabla f(\hat x^T)}^2} \leq \frac{2 \Delta^0}{\gamma T} + \fr{\Exp{G^0}}{p T}, \label{eq:marina_gen_non_cvx}
	\end{equation}
	where $\hat x^T$ is sampled uniformly at random from the points $\{x^0, x^1, \ldots, x^{T-1}\}$ produced by \algname{MARINA}, $\Delta^0 = f(x^0) - f^{\inf}$, and $G^0$ is defined in \eqref{eq:G^t}. 
	\item If additionaly Assumption~\ref{as:PL} hold and $0 \leq \gamma \leq \nicefrac{1}{M}$ for $M = \max\left\{L_{-} + L_{+}\sqrt{\nicefrac{2(1-p)\omega}{np}}, \nicefrac{p}{2\mu}\right\}$, then for any $T \ge 0$ we have
	\begin{equation}
		\Exp{f(x^T) - f(x^*)} \leq \left(1 - \gamma\mu\right)^T\left(\Delta^0 + \frac{\gamma}{p}\Exp{G^{0}}\right). \label{eq:marina_PL}
	\end{equation}
	\end{enumerate}
\end{theorem}

Next, this theorem and Corollaries~\ref{cor:main_cor_gen_non_cvx} and \ref{cor:main_cor_PL} imply the following complexity results.
\begin{corollary}\label{cor:marina}
	\begin{enumerate}
		\item Let the assumptions from the first part of Theorem~\ref{thm:marina} hold and
	\begin{equation*}
		\gamma = \frac{1}{L_{-} + L_{+}\sqrt{\nicefrac{(1-p)\omega}{np}}}.
	\end{equation*}
	Then for any $T$ we have
	\begin{equation}
		\Exp{\norm{ \nabla f(\hat x^T)}^2} \leq \frac{2 \Delta^0\left(L_{-} + L_{+}\sqrt{\nicefrac{(1-p)\omega}{np}}\right)}{ T} + \fr{\Exp{G^0}}{p T},\notag
	\end{equation}
	i.e., to achieve $\Exp{\norm{ \nabla f(\hat x^T)}^2} \leq \varepsilon^2$ for some $\varepsilon > 0$ the method requires
	\begin{equation}
		T = \cO\left(\frac{\Delta^0\left(L_{-} + L_{+}\sqrt{\nicefrac{(1-p)\omega}{np}}\right)}{\varepsilon^2} + \fr{\Exp{G^0}}{p \varepsilon^2}\right) \label{eq:marina_complexity_gen_non_cvx}
	\end{equation}
	iterations/communication rounds.
	\item Let the assumptions from the second part of Theorem~\ref{thm:marina} hold and
	\begin{equation*}
		\gamma = \min\left\{\frac{1}{L_{-} + L_{+}\sqrt{\nicefrac{2(1-p)\omega}{np}}}, \frac{p}{2\mu}\right\}.
	\end{equation*}
	Then to achieve $\Exp{f(x^T) - f(x^*)} \leq \varepsilon$ for some $\varepsilon > 0$ the method requires
	\begin{equation}
		\cO\left(\max\left\{\frac{L_{-} + L_{+}\sqrt{\nicefrac{(1-p)\omega}{np}}}{\mu}, p\right\}\log \frac{\Delta^0 + \Exp{G^{0}}\nicefrac{\gamma}{p}}{\varepsilon}\right) \label{eq:marina_complexity_PL}
	\end{equation}
	iterations/communication rounds.
	\end{enumerate}
\end{corollary}

\clearpage

\section{More Experiments}\label{sec:extra_experiments}
This section is organized as follows. We report more details on the experiment with autoencoder in \cref{sec:autoenc_exps}. In \Cref{sec:quadratic_exps}, we validate the new methods \algname{3PCv1}, \ldots, \algname{3PCv5} on a synthetic quadratic problem with a careful control of heterogenity level. Finally, in \cref{sec:clag}, we provide additional experiments with compressed lazy aggregation \algname{CLAG}. We refer the reader to \cref{sec:3points_special_cases_appendix,sec:contractive} for a formal definition of the algorithms and compressors.

All methods are implemented in Python 3.8 and run on 3 different CPU cluster nodes 
\begin{itemize}
	\item AMD EPYC 7702 64-Core; 
	\item Intel(R) Xeon(R) Gold 6148 CPU @ 2.40GHz;
	\item Intel(R) Xeon(R) Gold 6248 CPU @ 2.50GHz. 
\end{itemize}
Communication between server and clients is emulated in one computing node.

\subsection{Learning autoencoder model}\label{sec:autoenc_exps}

In this set of experiments, we test the proposed optimization methods on the task of learning a representation of MNIST dataset \cite{LeCun_MNIST}. We recall that we consider the following optimization problem

\begin{equation}
	\min _{\boldsymbol{D} \in \R^{d_{f} \times d_{e}}, \boldsymbol{E} \in \mathbb{R}^{d_{e} \times d_{f}}} \left[f(\boldsymbol{D}, \boldsymbol{E}):=\frac{1}{N} \sum_{i=1}^{N}\left\|\boldsymbol{D} \boldsymbol{E} a_{i}-a_{i}\right\|^{2}\right],
\end{equation}
where $a_i$ are flattened represenations of images with $d_f = 784$, $\boldsymbol{D}$ and $\boldsymbol{E}$ are learned parameters of the autoencoder model. We fix the encoding dimensions as $d_e = 16$ and distribute the data samples across $n = 10, 100$, or $1000$ clients. In order to control the heterogenity of this distribution, we use the following randomized procedure. First, split the dataset randomly into $n+1$ equal parts $D_0, D_1, \dots, D_n$ and fix the \textit{homogenity level} parameter $0\leq\hat{p}\leq 1$. Then let the $i$-th client take $D_0$ with probability $\hat{p}$ or $D_i$ otherwise. If $\hat{p} = 1$, we are in homogeneous regime. If $\hat{p} = 0$, all clients have different randomly shuffled data samples. Additionally, we study even more heterogeneous setting where we perform the \textit{split by labels}. This means that the clients from $1$ to $\nfr{n}{10}$ own the images corresponding to the first class, nodes from $\nfr{n}{10} + 1$ to $\nfr{2n}{10}$ own the images corresponding to the second class and so on (MNIST dataset has $10$ different classes). 

In this section, we choose $K=\nfr{d}{n}$, where $d = 2 \cdot d_f \cdot d_e = 25088$ is the total dimension of learning parameters $\boldsymbol{D}$ and $\boldsymbol{E}$. It is argued by \cite{PermK} that this is a suitable choice for \algname{MARINA} method with Rand-$K$ or Perm-$K$ sparsifiers. Methods involving two compressor such as \algname{3PCv2}, require to communicate two sparse sequences at every communication round. To account for this, we select $K_1, K_2$ from the set $\{\nfr{K}{2}, K\}$, that is there are four possible choices  for compression levels $K_1, K_2$ of two sparsifiers in \algname{3PCv2}. Then we select the pair which works best.

We fine-tune every method with the step-sizes from the set $\{2^{-12}, 2^{-11}, \dots, 2^{5}\}$ and select the best run based on the value of $\sqnorm{\nfxt}$ at the last iterate. The step-size for each method is indicated in the legend of each plot.

\paragraph{\algname{EF21} embraces different sparsifiers.}

Since \citet{Seide2014} proposed the error feedback style scheme, it has been successfully used in distributed training combined with some contractive compressor. A popular choice is Top-$K$, which preserves the "most important" coordinates  and shows empirical superiority. However, a natural question arises:
\begin{quote}\em Is the success of \algname{EF21} with Top-$K$ attributed to a careful algorithm design or to a greedy sparsifier in use?
\end{quote}
We compare \algname{EF21} with three different compressors: Top-$K$, cPerm-$K$, cRand-$K$ in \cref{fig:ef21-nodes-grads}. \algname{MARINA} with Perm-$K$ is added for the reference. In all cases, Top-$K$ demonstrates fast improvement in the first communication rounds. When $n=10$, the randomized compressors (cPerm-$K$ and cRand-$K$) work best for \algname{EF21}. When $n = 100$ the picture is similar, but cPerm-$K$ shows better performance than cRand-$K$ when \textit{homogenity level} is high ($1$ or $0.5$). Finally, Top-$K$ wins in the competition for $n = 1000$.

Takeaway $1$: \algname{EF21} is well designed and works well with different contractive compressors, including the randomized ones.

Takeaway $2$: \algname{EF21} combined with Top-$K$ is particularly useful if
\begin{itemize}
	\item we are interested in the progress during the initial phase of training;
	\item agressive sparsification is applied ($\nfr{k}{d} \ll 1\%$) and $n$ is large; or
	\item nodes own very different parts of dataset, i.e., we are in heterogeneous regime.
\end{itemize}

\begin{figure}[t]
	\centering
	\begin{subfigure}{}
		{Number of clients $n = 10$, compression level $K = 2509$. }
		\includegraphics[width=\linewidth]{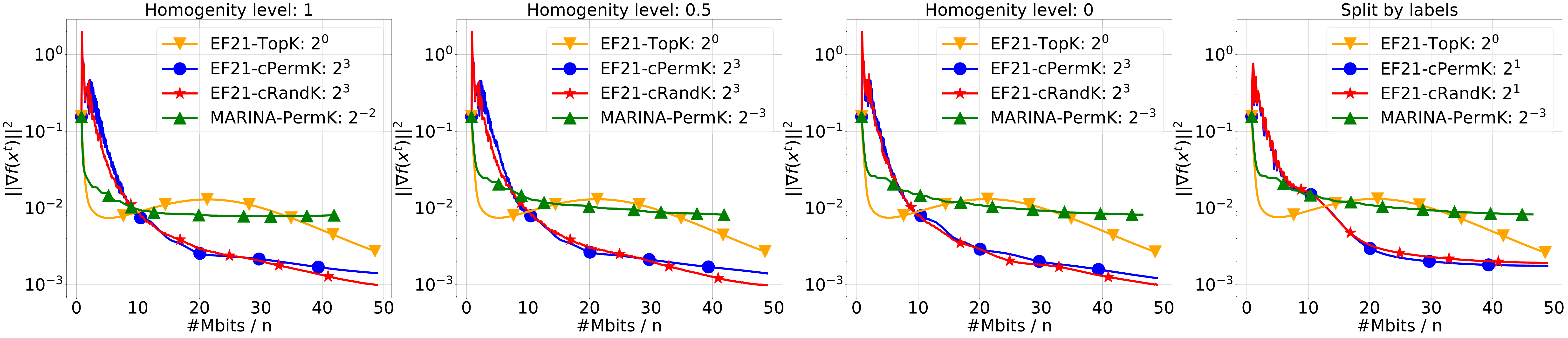}
	\end{subfigure}
	\begin{subfigure}{}
		{Number of clients $n = 100$, compression level $K = 251$. }
		\includegraphics[width=\linewidth]{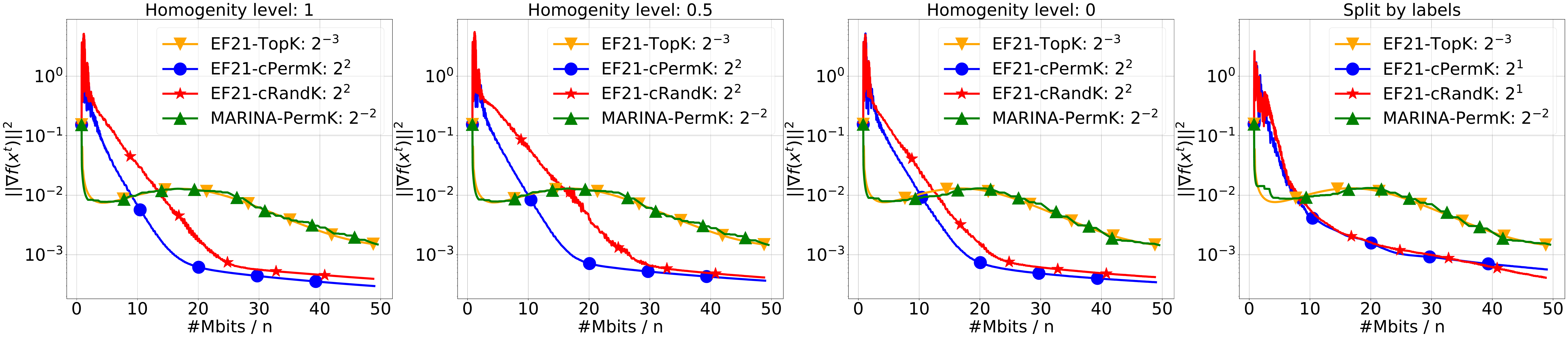}
	\end{subfigure}
	\begin{subfigure}{}
		{Number of clients $n = 1000$, compression level $K = 25$. }
		\includegraphics[width=\linewidth]{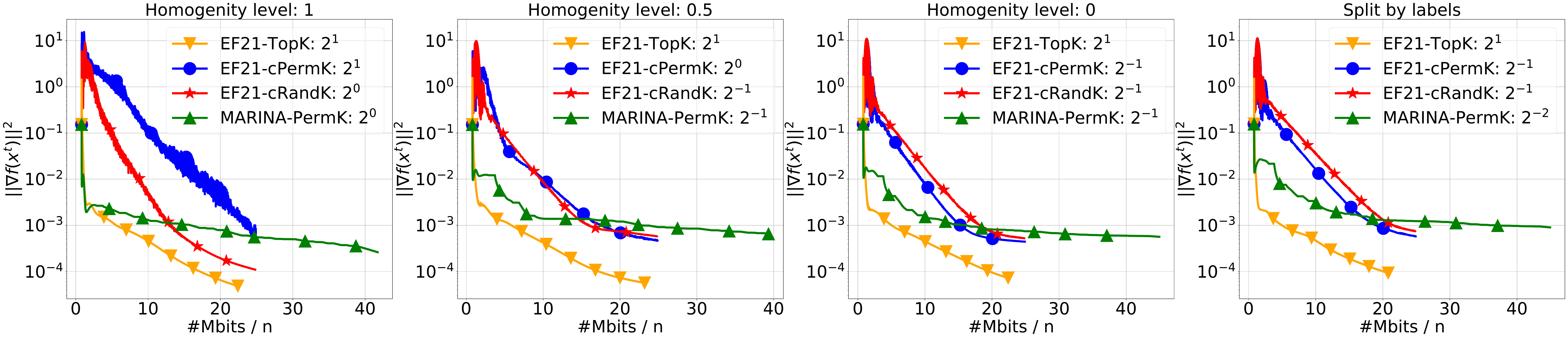}
	\end{subfigure}
	\caption{Comparison of \algname{EF21} with Top-$K$, cPerm-$K$ and cRand-$K$ compressors. \algname{MARINA} with Perm-$K$ is provided for the reference.}
	\label{fig:ef21-nodes-grads}
\end{figure}

\paragraph{\algname{MARINA} and greedy sparsification (\algname{3PCv5})}
We now draw our attention to one of the newly proposed methods: \algname{MARINA} combined with biased compression operators (named as \algname{3PCv5} in \cref{alg:b_marina} and \cref{tab:methods}). According to our theory, see \cref{tab:methods}, \algname{3PCv5} has the same compexity as \algname{EF21}. In this experiment, we aim to validate the proposed method with greedy Top-$K$ sparsifier. We compare it to \algname{MARINA} with Perm-$K$ and Rand-$K$ and include \algname{EF21} as a reference method. Interestingly, Top-$K$ improves over Perm-$K$ and Rand-$K$ when $n = 10$; in homogeneus case, the behavior of Top-$K$ and Perm-$K$ is similar, see \cref{fig:marina-nodes-grads}. However, this improvement vanishes when $n$ is increased ($n = 100, 1000$) and sparsification is more agressive; \algname{MARINA} with Top-$K$ requires much smaller step-sizes to converge. In all cases, \algname{EF21} with Top-$K$ is faster. 

\begin{figure}[t]
	\centering
	\begin{subfigure}{}
		{Number of clients $n = 10$, compression level $K = 2509$. }
		\includegraphics[width=\linewidth]{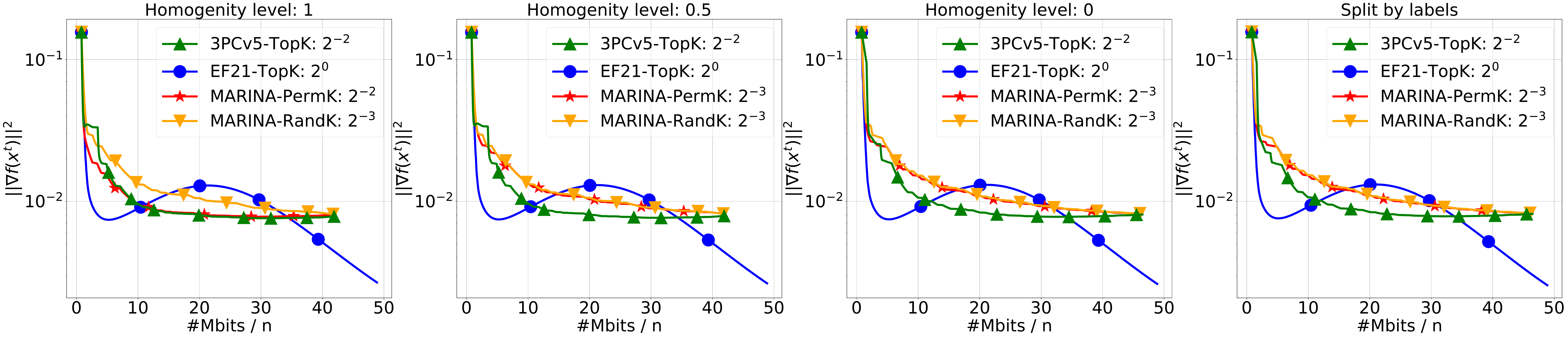}
	\end{subfigure}
	\begin{subfigure}{}
		{Number of clients $n = 100$, compression level $K = 251$. }
		\includegraphics[width=\linewidth]{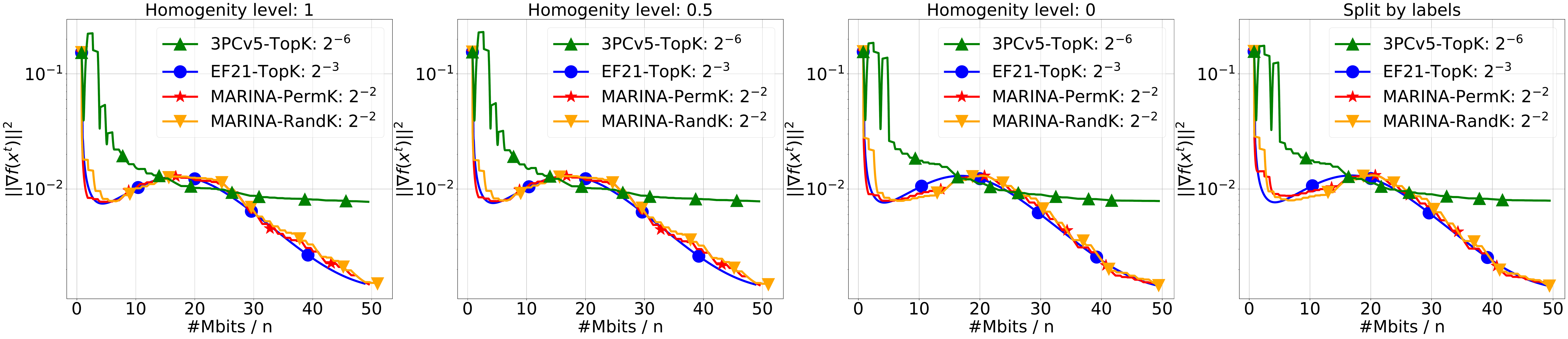}
	\end{subfigure}
	\begin{subfigure}{}		
		{Number of clients $n = 1000$, compression level $K = 25$. }
		\includegraphics[width=\linewidth]{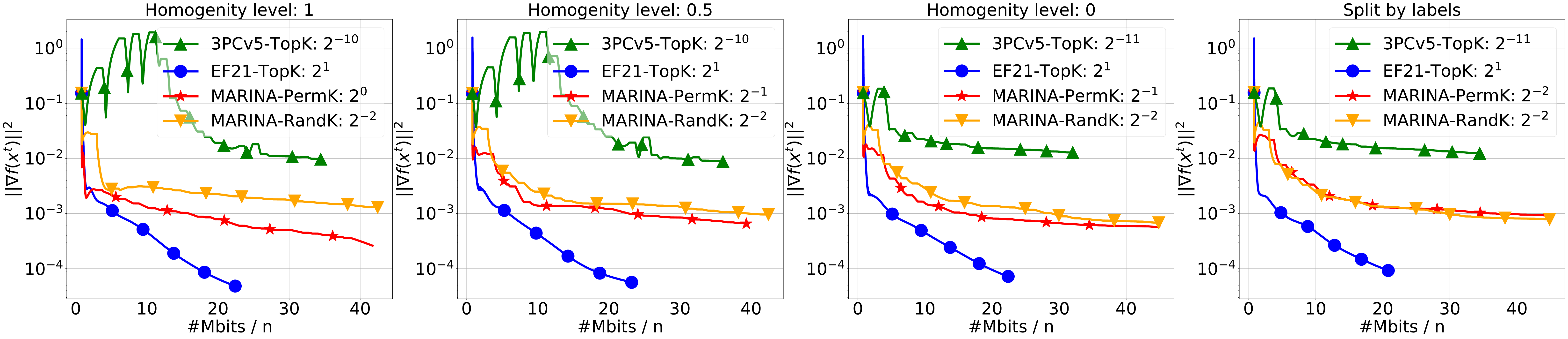}
	\end{subfigure}
		\caption{Comparison of \algname{MARINA} with Perm-$K$, Rand-$K$ and \algname{3PCv5} with Top-$K$.}
		\label{fig:marina-nodes-grads}
\end{figure}

\paragraph{Other \algname{3PC} variants}
Motivated by the success of greedy sparsification and favorable properties of randomized sparsifiers, we aim to investigate if one can combine the two in a nontrivial way and obtain even faster method. One possible way to do so is to look more closely to one of the special cases of \algname{3PC} named \algname{3PCv2}. With \algname{3PCv2} (\cref{alg:anna}), we have more freedom because it has two compressors. In our experiments, we consider three different sparsifiers (Top-$K$, Rand-$K$, Perm-$K$) as for the first compressor and fix the second one as Top-$K$, see \cref{fig:anna-nodes-grads}. 

For $n = 10$, the performance of \algname{3PCv2} with Rand-$K$-Top-$K$ and Top-$K$-Top-$K$ is very similar to the one of \algname{EF21} with Top-$K$. Interestingly, \algname{3PCv2}-Rand-$K$-Top-$K$ becomes superior for $n = 100$ converging even faster than \algname{EF21}. The difference is especially prominent in heterogeneous setting. Finally, \algname{EF21} shows slightly beter performance in the experiments with $1000$ nodes. We can conclude that:
\begin{itemize}
	\item \algname{3PCv2} can outperform \algname{EF21} in some cases, for example, \cref{fig:anna-100-nodes-grads},
	\item \algname{EF21} is still superior when $n$ is large.
\end{itemize} 
However, more emprirical evidence is needed to investigate the behavior of \algname{3PCv2} and other new methods fitting \algname{3PC} framework.

\begin{figure}[H]
	\centering
	\begin{subfigure}{}
		{Number of clients $n = 10$, compression level $K = 2509$. }
		\includegraphics[width=\linewidth]{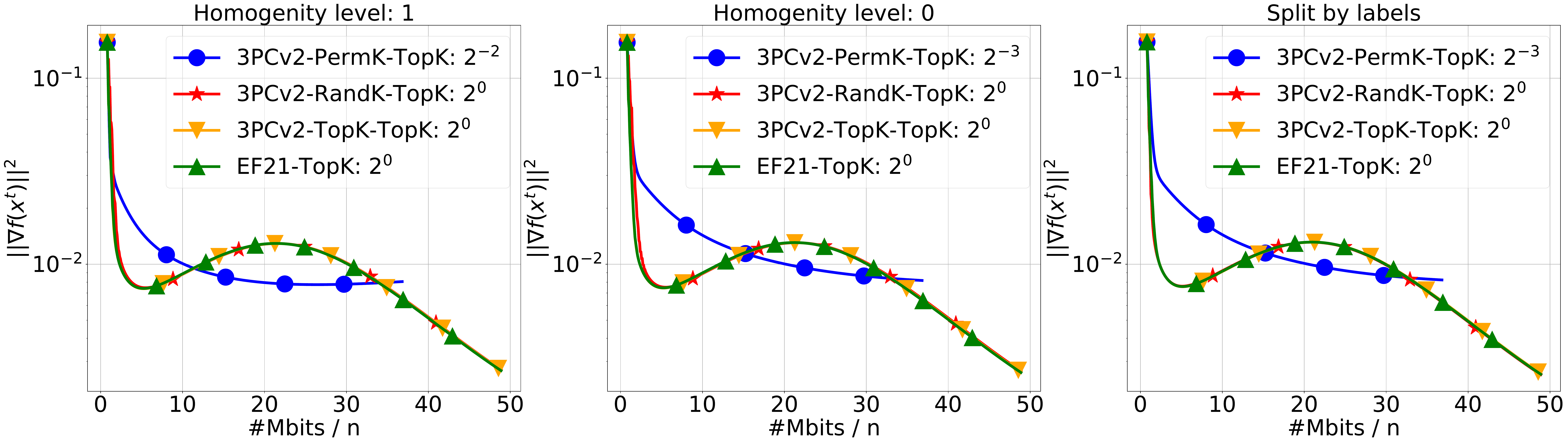}
		\label{fig:anna-10-nodes-grads}
	\end{subfigure}
	\begin{subfigure}{}
		{Number of clients $n = 100$, compression level $K = 251$. }
		\includegraphics[width=\linewidth]{plots/auto_encoder/ANNA-100-nodes-grads}
		\label{fig:anna-100-nodes-grads}
	\end{subfigure}
	\begin{subfigure}{}
		{Number of clients $n = 100$, compression level $K = 25$. }
		\includegraphics[width=\linewidth]{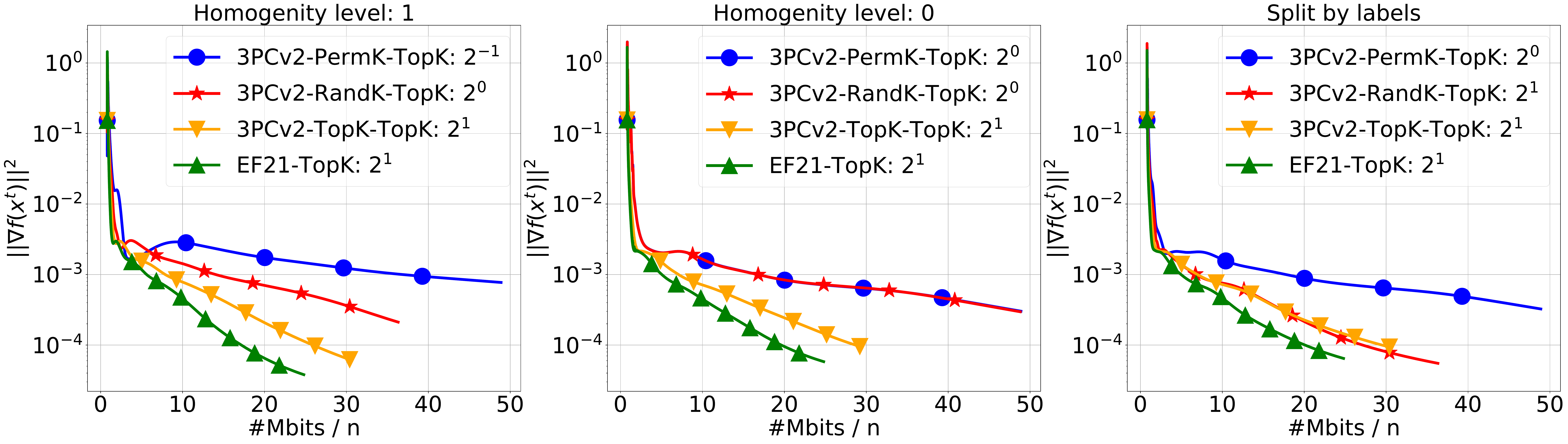}
		\label{fig:anna-1000-nodes-grads}
	\end{subfigure}
	\caption{Comparison of \algname{3PCv2} with Perm-$K$, Rand-$K$ and Top-$K$ as the first compressor. Top-$K$ is used as the second compressor. \algname{EF21} with Top-$K$ is provided for the reference.}
	\label{fig:anna-nodes-grads}
\end{figure}

\subsection{Solving synthetic quadratic problem}\label{sec:quadratic_exps}

In this experimental section we compare practical performance of the proposed methods \algname{3PCv1},   \algname{3PCv2},  \algname{3PCv4},  \algname{3PCv5}
against existing state-of-the-art methods for compressed distributed optimization \algname{MARINA} and \algname{EF21}.
For this comparison we set up the similar setting that was introduced in \cite{PermK}.
Firstly, let us describe the experimental setup in detail. We consider the finite sum function 
$f(x) = \suminn f_i(x)$, consisting of  synthetic quadratic functions
\begin{eqnarray}
	f_i(x) = \fr{1}{2}x^{\top}\bA_i x - x^{\top} b_i,	
\end{eqnarray}
 where $\bA_i \in \R^{d \times d}$, $b_i \in \R^d,$  and $\bA_i = \bA_i^{\top}$ is the training data that belongs to the device/worker  $i$. In all experiments of this section, we  have $d = 1000$ and generated $\bA_i$ in a such way that $f$ is $\lambda$--strongly convex ( i.e., $\frac{1}{n}\sum_{i=1}^n \bA_i\succcurlyeq \lambda \bI $ for $\lambda > 0$) with $\lambda = 1\mathrm{e}^{-6}$. We now present Algorithm \ref{algorithm:matrix_generation} which is used to generate these synthetic matrices (training data).
 
 \begin{algorithm}[!h]
 	\caption{Quadratic optimization task generation \citep{PermK}}
 	\begin{algorithmic}[1]
 		\label{algorithm:matrix_generation}
 		\STATE \textbf{Parameters:} number nodes $n$, dimension $d$, regularizer $\lambda$, and noise scale $s$.
 		\FOR{$i = 1, \dots, n$}
 		\STATE Generate random noises $\nu_i^s = 1 + s \xi_i^s$ and $\nu_i^b = s \xi_i^b,$ i.i.d. $\xi_i^s, \xi_i^b \sim \mathcal {N}(0, 1)$
 		\STATE Take vector $b_i = \frac{\nu_i^s}{4}(-1 + \nu_i^b, 0, \cdots, 0) \in \R^{d}$
 		\STATE Take the initial tridiagonal matrix
 		\[\bA_i = \frac{\nu_i^s}{4}\left( \begin{array}{cccc}
 			2 & -1 & & 0\\
 			-1 & \ddots & \ddots & \\
 			& \ddots & \ddots & -1 \\
 			0 & & -1 & 2 \end{array} \right) \in \R^{d \times d}\]
 		\ENDFOR
 		\STATE Take the mean of matrices $\bA = \frac{1}{n}\sum_{i=1}^n \bA_i$
 		\STATE Find the minimum eigenvalue $\lambda_{\min}(\bA)$
 		\FOR{$i = 1, \dots, n$}
 		\STATE Update matrix $\bA_i = \bA_i + (\lambda - \lambda_{\min}(\bA)) \bI$
 		\ENDFOR
 		\STATE Take starting point $x^0 = (\sqrt{d}, 0, \cdots, 0)$
 		\STATE \textbf{Output:} matrices $\bA_1, \cdots, \bA_n$, vectors $b_1, \cdots, b_n$, starting point $x^0$
 	\end{algorithmic}
 \end{algorithm}
We generated optimization tasks  having different number of nodes $n = \cb{10, 100, 1000}$ and capturing various data-heterogeneity regimes that are controlled by so-called \textit{Hessian variance}\footnote{For more details, see the original paper \citet{PermK} introducing this concept.} term:
 \begin{definition}[Hessian variance \citep{PermK}] \label{ass:HV} Let $\Lpm\geq 0$ be the smallest quantity such that
 	\begin{equation}\label{eq:HV} \textstyle \frac{1}{n} \sum \limits_{i=1}^n \norm{\nabla f_i(x) - \nabla f_i(y)}^2 - \norm{\nabla f(x) - \nabla f(y)}^2 \leq L^2_{\pm} \norm{x-y}^2, \quad \forall x,y\in \R^d. \end{equation}
 	We refer to the quantity $L_{\pm}^2$ by the name {\em Hessian variance.} 
 \end{definition}
From the definition , it follows that the case of similar (or even identical) functions $f_i$ relates to the small (or even $0$) Hessian variance, whereas in the case of completely different $f_i$
(which relate to heterogeneous data regime)  $L_{\pm}$ can be large.

In our experiments, homogeneity of  each optimizations task is controlled by noise scale $s$ introduced in the Algorithm \ref{algorithm:matrix_generation}. Indeed, for the noise scale $s = 0$, all matrices $\bA_i$ are equal, whereas  with the increase of the noise scale, functions become less “similar” and $L_{\pm}$ rises. We take noise scales $s \in \cb{0.0, 0.05, 0.8, 1.6, 6.4}$. A summary of the $L_{\pm}$ and $L_{-}$ values corresponding to these noise scales is given in the Tables \ref{tbl:L_pm_summary} and \ref{tbl:L_m_summary}. For the considered quadratic problem $L_{\pm}$ can be analytically expressed as  $L_{\pm} = \sqrt{\lam_{\max}\rb{\suminn \bA_i^2 - \rb{\suminn \bA_i}^2 } }$.

\begin{table}[h]
	\caption{Summary of the Hessian variance terms $L_{\pm}$ for different number of nodes $n$ various noise scales $s$.} 
	\label{tbl:L_pm_summary}
	\centering
	\begin{tabular}{l l l l l l}
		\toprule
		\diagbox[width=1.2cm ,  height=0.5cm]{ $n$ }{\raisebox{0.5ex}{ $s$}}  & $0$ & $0.05$ &$0.8$  & $1.6$& $6.4$ \\
		\midrule 			
		$10$ 			& $0$  & $0.06$ & $0.9$ &  $1.79$ & $7.17$		\\  
		$100$		  & $0$ & $0.05$& $0.82$  & $1.65$ & $6.58$	\\
		$1000$	     & $0$  & $0.05$ & $0.81$  & $1.62$ &	$6.48$  \\
		\bottomrule
	\end{tabular}
\end{table}

\begin{table}[h]
	\caption{Summary of the Hessian variance terms $L_{-}$ for different number of nodes $n$ various noise scales $s$.} 
	\label{tbl:L_m_summary}
	\centering
	\begin{tabular}{l l l l l l}
		\toprule
		\diagbox[width=1.2cm ,  height=0.5cm]{ $n$ }{\raisebox{0.5ex}{ $s$}}  & $0$ & $0.05$ &$0.8$  & $1.6$& $6.4$ \\
		\midrule 			
		$10$ 			& $1.0$  & $1.02$ & $1.35$ &  $1.7$ & $3.82$		\\  
		$100$		  & $1.0$ & $1.0$& $0.97$  & $0.94$ & $0.77$	\\
		$1000$	     & $1.0$  & $1.0$ & $0.97$  & $0.95$ &	$0.78$  \\
		\bottomrule
	\end{tabular}
\end{table}

For all algorithms, at each iteration we compute the squared norm of the exact/full gradient for comparison of the methods performance. We terminate our algorithms either if they reach the certain number of iterations or the following stopping criterion is satisfied:  $\sqnorm{\nf{\xt}} \le 10^{-7}$. 

In all experiments, the stepsize of each method is set to the largest theoreticaly possible stepsize multiplied by some constant multiplier which was individually tuned in all cases within powers of $2$ : $\cb{2, 4, 8, 16, 32, 64, 128, 256, 512, 1024, 2048, 4096, 8192, 16384, 32768}$ .
\paragraph{\algname {EF21} and different compressors}
Following the same order as in the section \ref{sec:autoenc_exps} we start by comparing existing SOTA methods (\algname{MARINA} with Perm-$K$ and \algname{EF21} with Top-$K$) against \algname{EF21} with cPerm-$K$ and cRand-$K$.
In Figure \ref{fig:ef21-methods},  paramater $K = \nicefrac{d}{n}$ is fixed for each row. Each column corresponds to a heterogeneity levels defined by the averaged $L_{\pm}$  and $L_{-}$  per values $n$ (averaged per column in the Tables \ref{tbl:L_pm_summary} and \ref{tbl:L_m_summary}).

\begin{figure}[t]
	\centering
		\includegraphics[width=\linewidth]{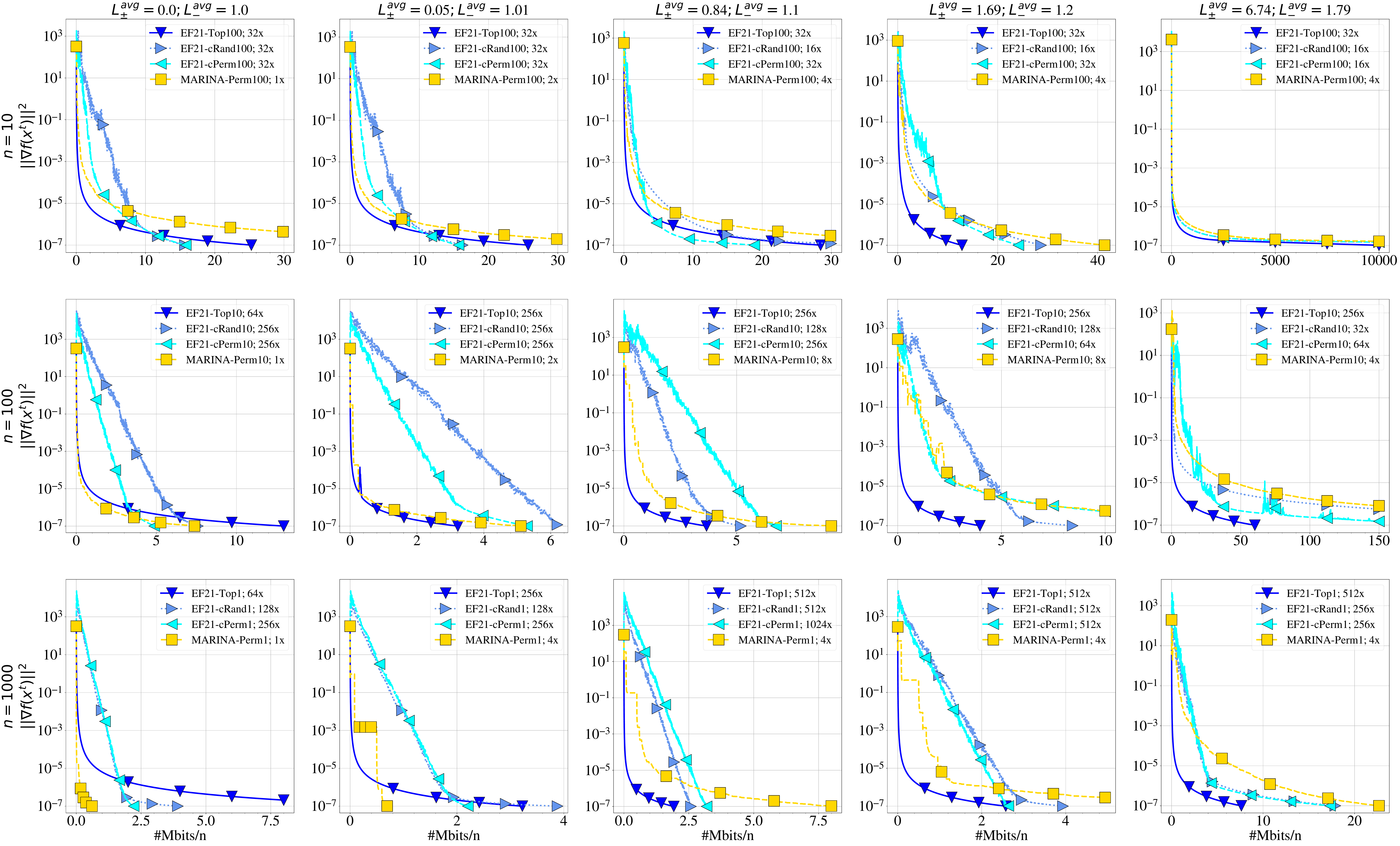}
		\caption{Comparison of \algname{MARINA} with Perm-$K$, \algname{EF21} with Top-$K$, cPerm-$K$ and cRand-$K$ with $K = \nicefrac{d}{n}$ and tuned stepsizes. By $1\times, 2\times, 4\times$ (and so on) we indicate that the stepsize is set to a multiple of the largest stepsize predicted by theory. $L_{\pm}^{avg}$ and $L_{-}^{avg}$ are the averaged constants $L_{\pm}$ and $L_{-}$ per column.}
		\label{fig:ef21-methods}
\end{figure}

These experiments shows that, in low Hessian variance regime \algname{EF21} with cPerm-$K$ and cRand-$K$ in some cases improves \algname{MARINA} with Perm-$K$ for $n=10, 100$, whereas for $n=1000$ \algname{MARINA} with cPerm-$K$ still dominates. Moreover, even in big Hessian variance regime $\algname{EF21}$ methods converges faster than \algname{MARINA} with cPerm-$K$ but not as fast as \algname{EF21} with Top-$K$. We are not aware of any prior empirical study for \algname{EF21} combined with cPerm-$K$ or cRand-$K$.

\paragraph{\algname {MARINA} and different compressors}
In this section, we keep the same setting and compare a new method \algname{3PCv5} with Top-$K$ against  \algname{MARINA} with Perm-$K$, Rand-$K$ and \algname{EF21} with Top-$K$. In Figure \ref{fig:marina-methods}, one can see that  \algname{3PCv5} with Top-$K$ outperforms  \algname{MARINA} methods only in a couple of cases for $n=10$, whereas for the most of the regimes it converges slower than other methods.
\begin{figure}[t]
	\centering
	\includegraphics[width=\linewidth]{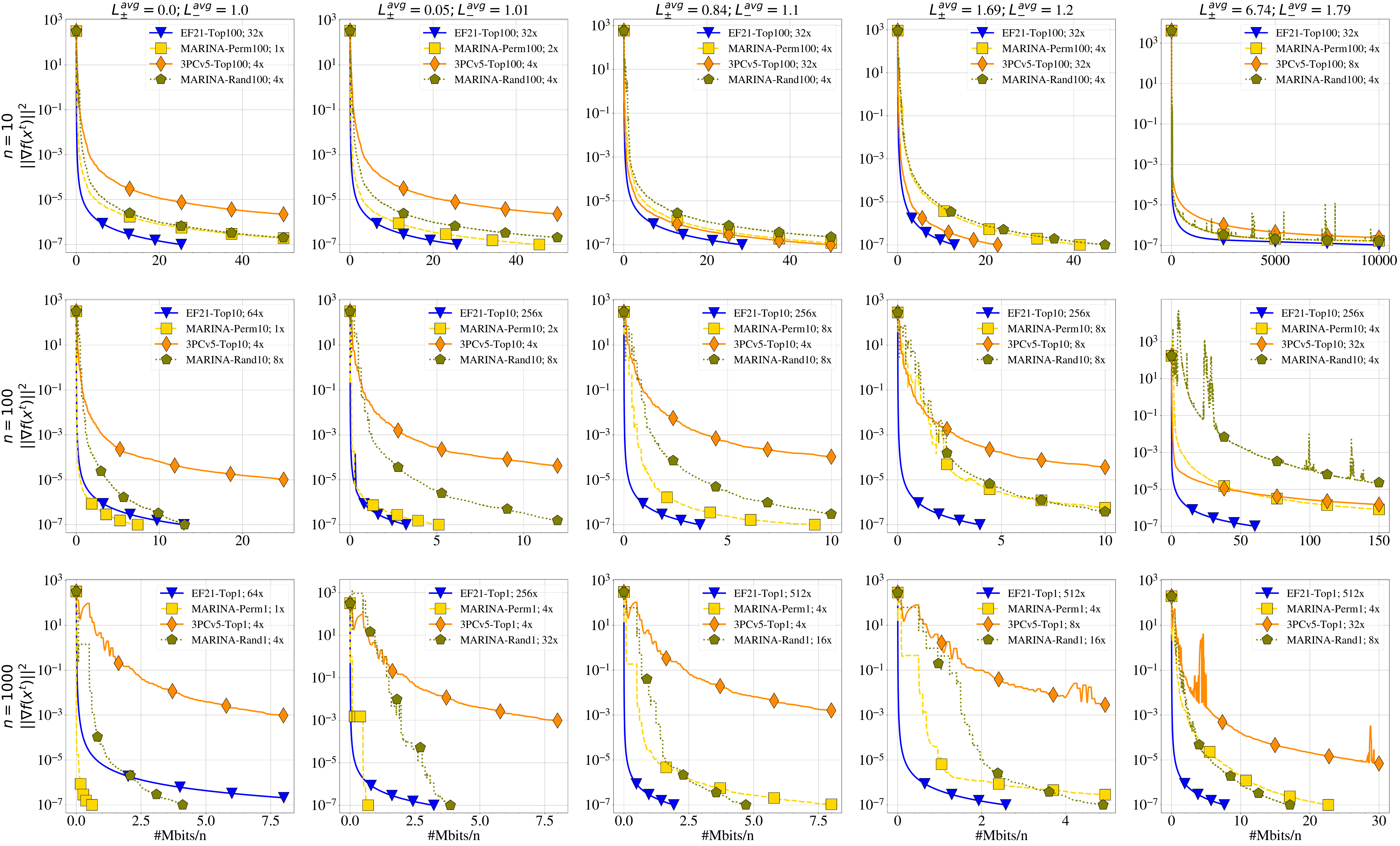}
	\caption{Comparison of \algname{MARINA} with Perm-$K$, Rand-$K$, \algname{EF21} with Top-$K$ and \algname{3PCv5} with $K = \nicefrac{d}{n}$ and tuned stepsizes. By $1\times, 2\times, 4\times$ (and so on) we indicate that the stepsize is set to a multiple of the largest stepsize predicted by theory. $L_{\pm}^{avg}$ and $L_{-}^{avg}$ are the averaged constants $L_{\pm}$ and $L_{-}$ per column.}
	\label{fig:marina-methods}
\end{figure}

\paragraph{\algname {3PCv2} beat SOTA methods in the most cases!}
In this series of experiments, we stick to the previous setting and append the results of the new method \algname {3PCv2} with 2 different combinataion of compressors: Rand$K_1$-Top$K_2$ and Rand$K_1$ $*$Perm$K$ -Top$K_2$, where  Rand$K_1$ $*$Perm-$K$  is the composition of  Rand-$K_1$ and Perm-$K$. For both methods, constants $K_1$ and $K_2$ were extensively tuned over the set of $9$ different pairs (see Figures \ref{fig:don-anna-rt} and \ref{fig:don-anna-rpt} for details). In the Figure \ref{fig:don-best} it is shown that both variants \algname {3PCv2} methods converge quickly for $n=100$ in all heterogeneity regimes, outperforming \algname {MARINA} and \algname{EF21}. In the big Hessian variance regime and  $n=10$,  \algname{3PCv2} also converges faster than \algname{EF21} with Top-$K$, however, for even more homogeneous cases \algname{3PCv2} slightly looses to \algname{EF21} with cPerm-$K$ or cRand-$K$. We also would like to note that we excluded \algname{3PCv4} with Top$K_1$-Top$K_2$ from our comparison here since in practice for $K = \nicefrac{d}{n}$ it coincides with \algname{EF21} with Top-$K$ (see Figure \ref{fig:don-jacqueline} for more details)\footnote{We believe that this behaviors of \algname{3PCv4} with Top$K_1$-Top$K_2$ takes place is due to the problem sparsity.}

\begin{figure}[H]
	\centering
		\includegraphics[width=\linewidth]{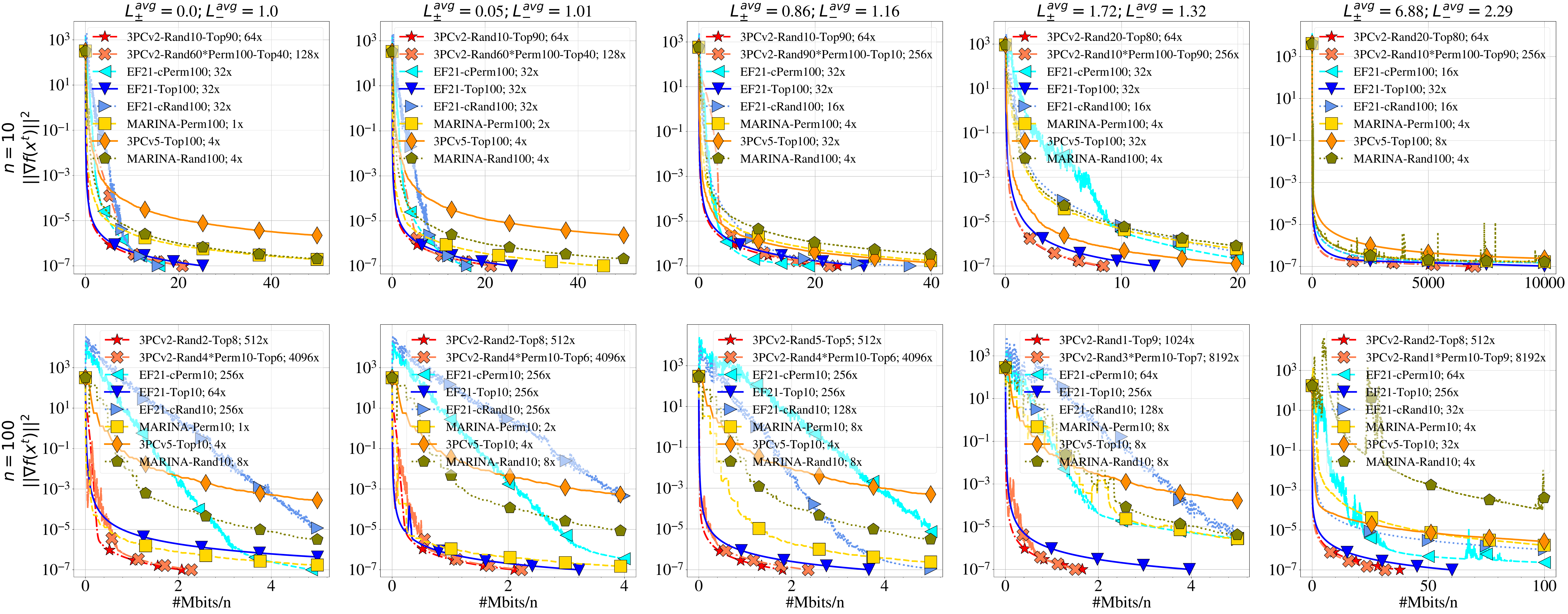}
\caption{Comparison of \algname{MARINA}, \algname{EF21}, \algname{3PCv2} and \algname{3PCv5} with various compressors , $K = \nicefrac{d}{n}$ and tuned stepsizes. By $1\times, 2\times, 4\times$ (and so on) we indicate that the stepsize is set to a multiple of the largest stepsize predicted by theory. $L_{\pm}^{avg}$ and $L_{-}^{avg}$ are the averaged constants $L_{\pm}$ and $L_{-}$ per column.}
		\label{fig:don-best}
\end{figure}

We further continue with the setting where $\nfr{K}{d} = 0.02$ is fixed for each $n$. In the Figure \ref{fig:cmprsd-best} illustrates that \algname {3PCv2} remaines the best choice for $n=10$ and $n=100$, whereas in the homogeneous regime and $n=1000$ \algname {EF21} with cRand-$K$ can reach the desired tolerance a bit faster. However, for big Hessian varince regime \algname {3PCv2} as again preferable.
\begin{figure}[H]
		\includegraphics[width=\linewidth]{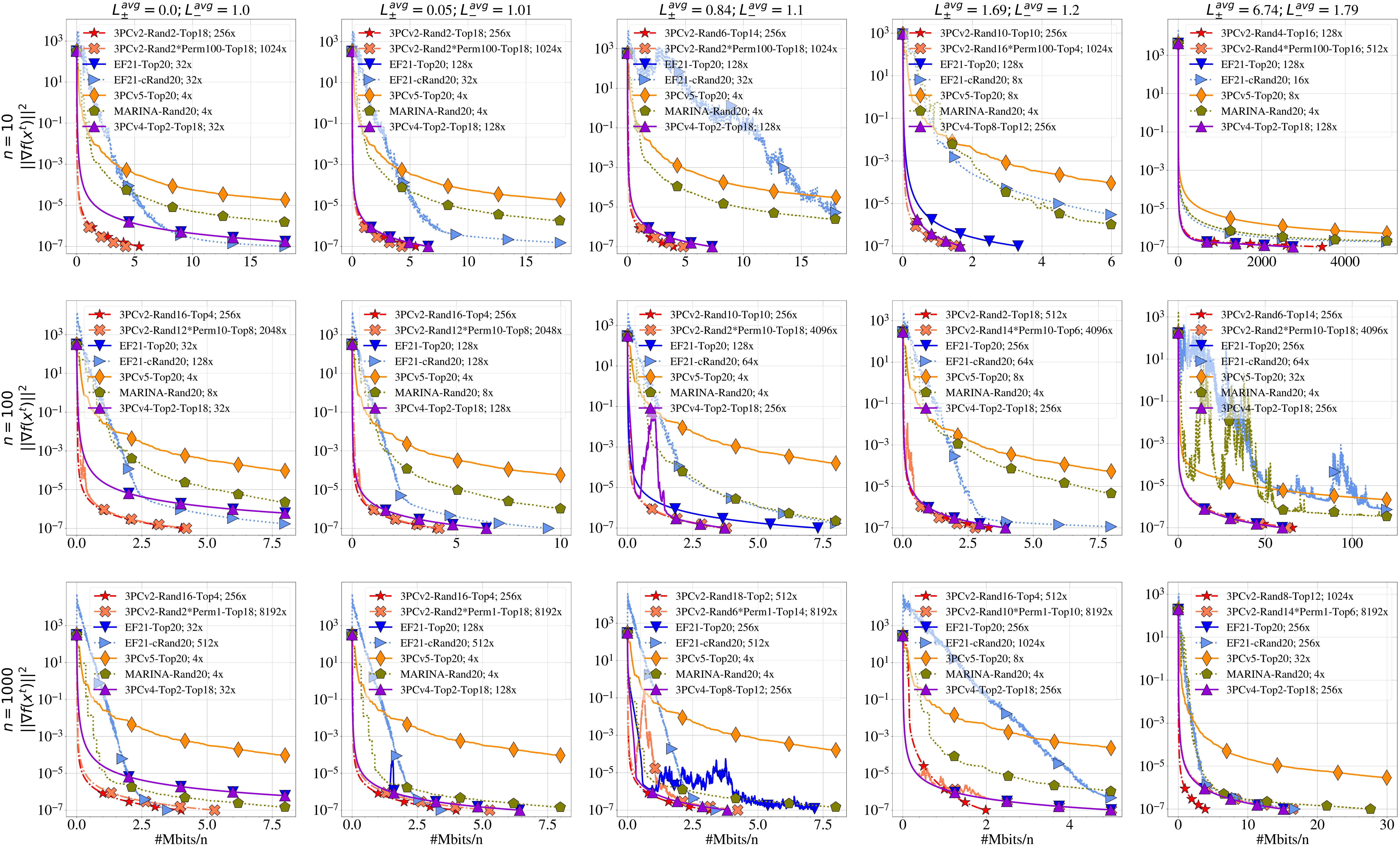}
		\caption{Comparison of \algname{MARINA}, \algname{EF21}, \algname{3PCv2}, \algname{3PCv5} and \algname{3PCv5} with various compressors , $K = 0.02{d}$ and tuned stepsizes. By $1\times, 2\times, 4\times$ (and so on) we indicate that the stepsize is set to a multiple of the largest stepsize predicted by theory. $L_{\pm}^{avg}$ and $L_{-}^{avg}$ are the averaged constants $L_{\pm}$ and $L_{-}$ per column.}
		\label{fig:cmprsd-best}
\end{figure}

\paragraph{Fine-tuning of $(K_1, K_2)$ pairs for \algname{3PCv2} and \algname{3PCv4}}
In this section we provide with some auxillary results on the demonstration of the tuning $(K_1, K_2)$ pairs for \algname{3PCv2} and \algname{3PCv4} on different compressors. In the $K = \nicefrac{d}{n}$ scenario (see Figure \ref{fig:don-anna-rt}), in all cases the best performance of \algname{3PCv2} with Rand$K_1$-Top$K_2$ is achieved when $K_2 > K_1$, whereas for the case when $\nfr{K}{d}=0.02$ (see Figure \ref{fig:cmprsd-anna-rt}) there is a dependence on $n$; for $n=10$, the choice when $K_2 > K_1$ is preferable in all cases, whereas for $n=100$ and $n=1000$ it is the case only in big Hessian variance regime. At the same time, for optimal pairs $(K_1,K_2)$ of the method \algname{3PCv2} with Rand$K_1*$Perm$K$-Top$K_2$ we observe that the choice is $K_2 > K_1$ (see Figures \ref{fig:don-anna-rpt},  \ref{fig:cmprsd-anna-rpt}).

\begin{figure}[H]
	\centering
		\includegraphics[width=\linewidth]{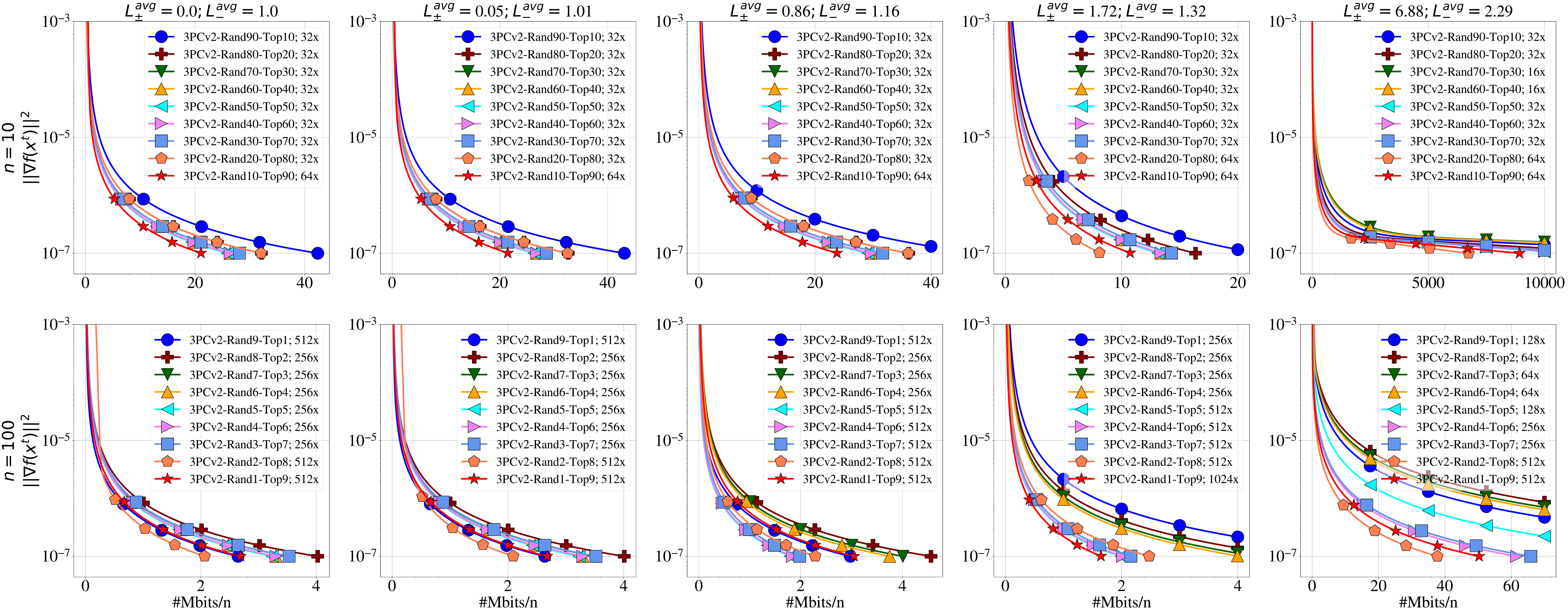}
		\caption{Comparison of \algname{3PCv2} with methods with Rand$K_1$-Top$K_2$ with  $K_1 + K_2 =  \nicefrac{d}{n}$ and tuned stepsizes. By $1\times, 2\times, 4\times$ (and so on) we indicate that the stepsize is set to a multiple of the largest stepsize predicted by theory. $L_{\pm}^{avg}$ and $L_{-}^{avg}$ are the averaged constants $L_{\pm}$ and $L_{-}$ per column.}
		\label{fig:don-anna-rt}
	\end{figure}
		\begin{figure}[H]
			\includegraphics[width=\linewidth]{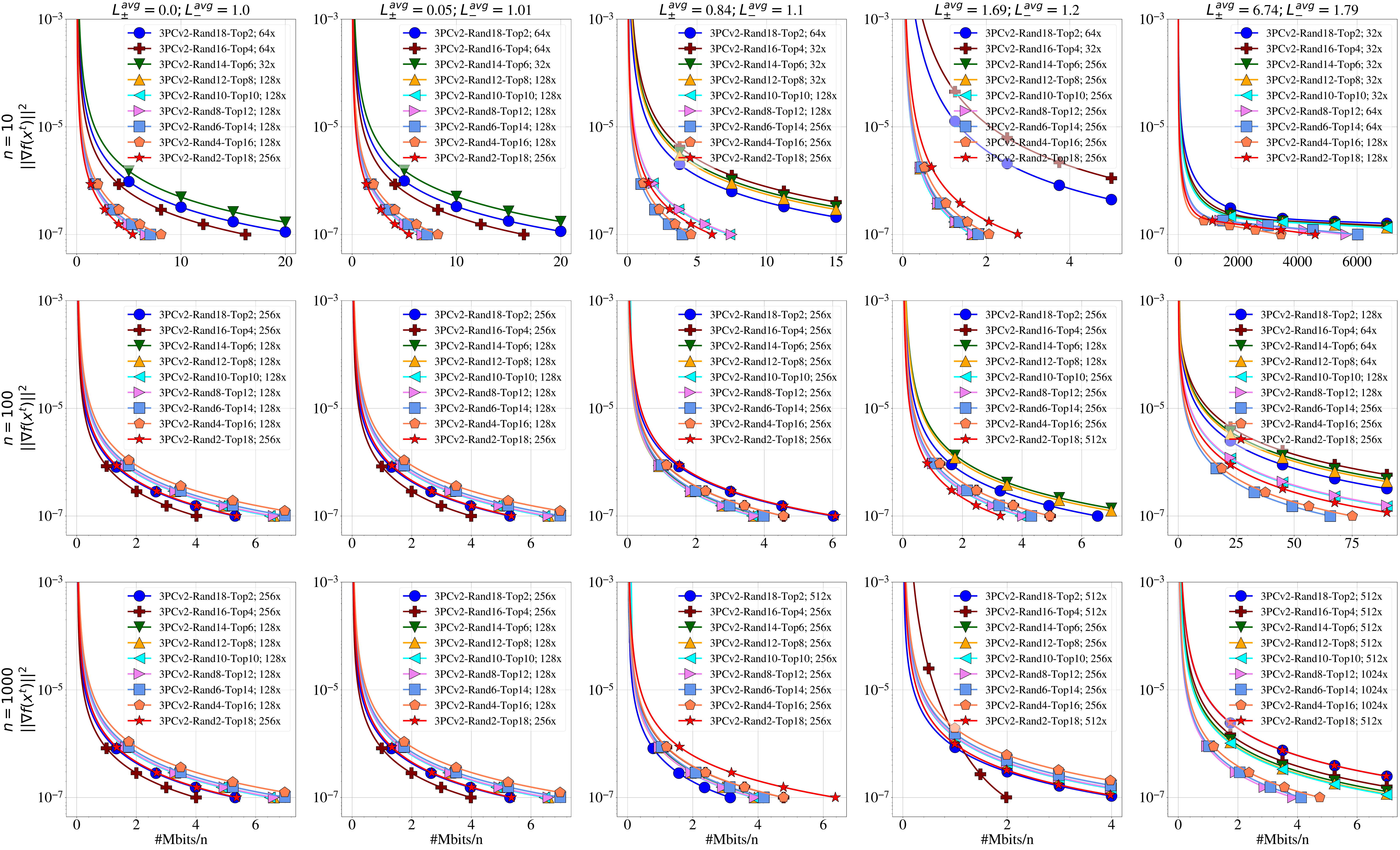}
			\caption{Comparison of \algname{3PCv2} with methods with Rand$K_1$-Top$K_2$ with $K_1 + K_2 =  0.02{d}$ and tuned stepsizes. By $1\times, 2\times, 4\times$ (and so on) we indicate that the stepsize is set to a multiple of the largest stepsize predicted by theory. $L_{\pm}^{avg}$ and $L_{-}^{avg}$ are the averaged constants $L_{\pm}$ and $L_{-}$ per column.}
			\label{fig:cmprsd-anna-rt}
\end{figure}

\begin{figure}[H]

		\includegraphics[width=\linewidth]{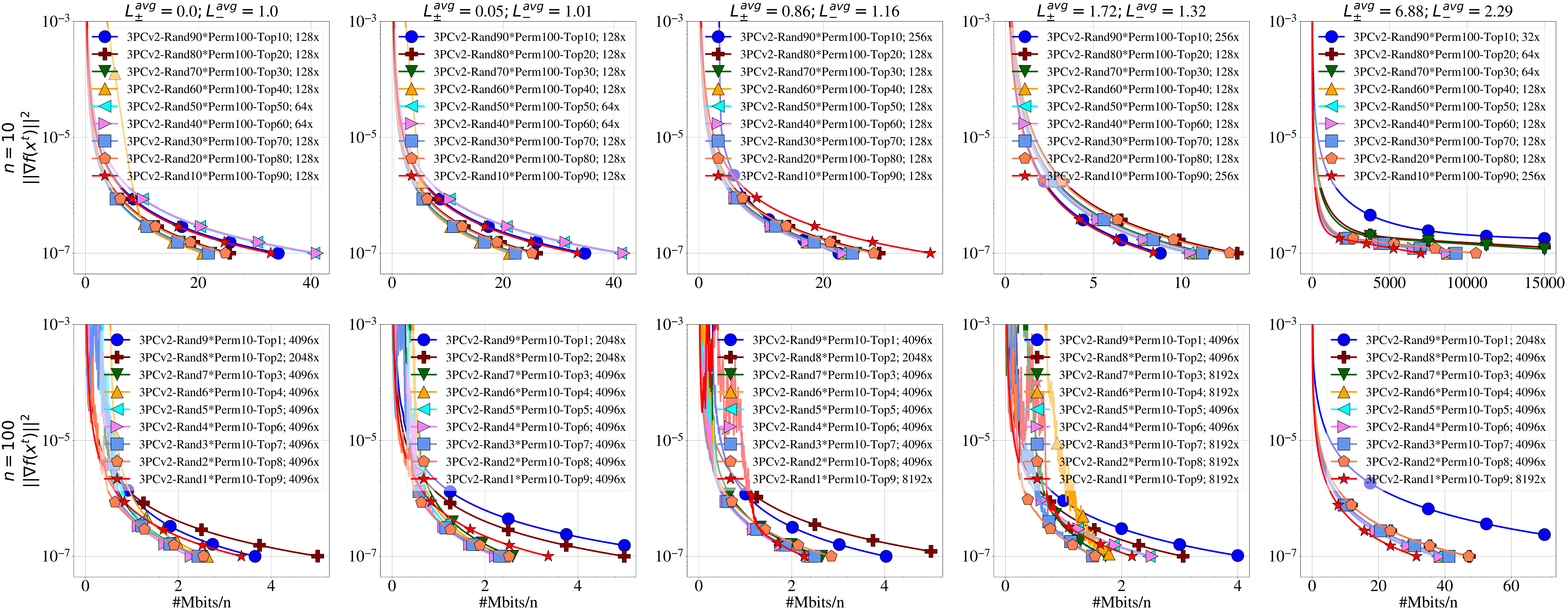}
		\caption{Comparison of \algname{3PCv2} with methods with Rand$K_1*$Perm$K$-Top$K_2$ with $K_1 + K_2 =  \nicefrac{d}{n}$ and tuned stepsizes. By $1\times, 2\times, 4\times$ (and so on) we indicate that the stepsize is set to a multiple of the largest stepsize predicted by theory. $L_{\pm}^{avg}$ and $L_{-}^{avg}$ are the averaged constants $L_{\pm}$ and $L_{-}$ per column.}
		\label{fig:don-anna-rpt}
		\end{figure}
\begin{figure}[H]
		\includegraphics[width=\linewidth]{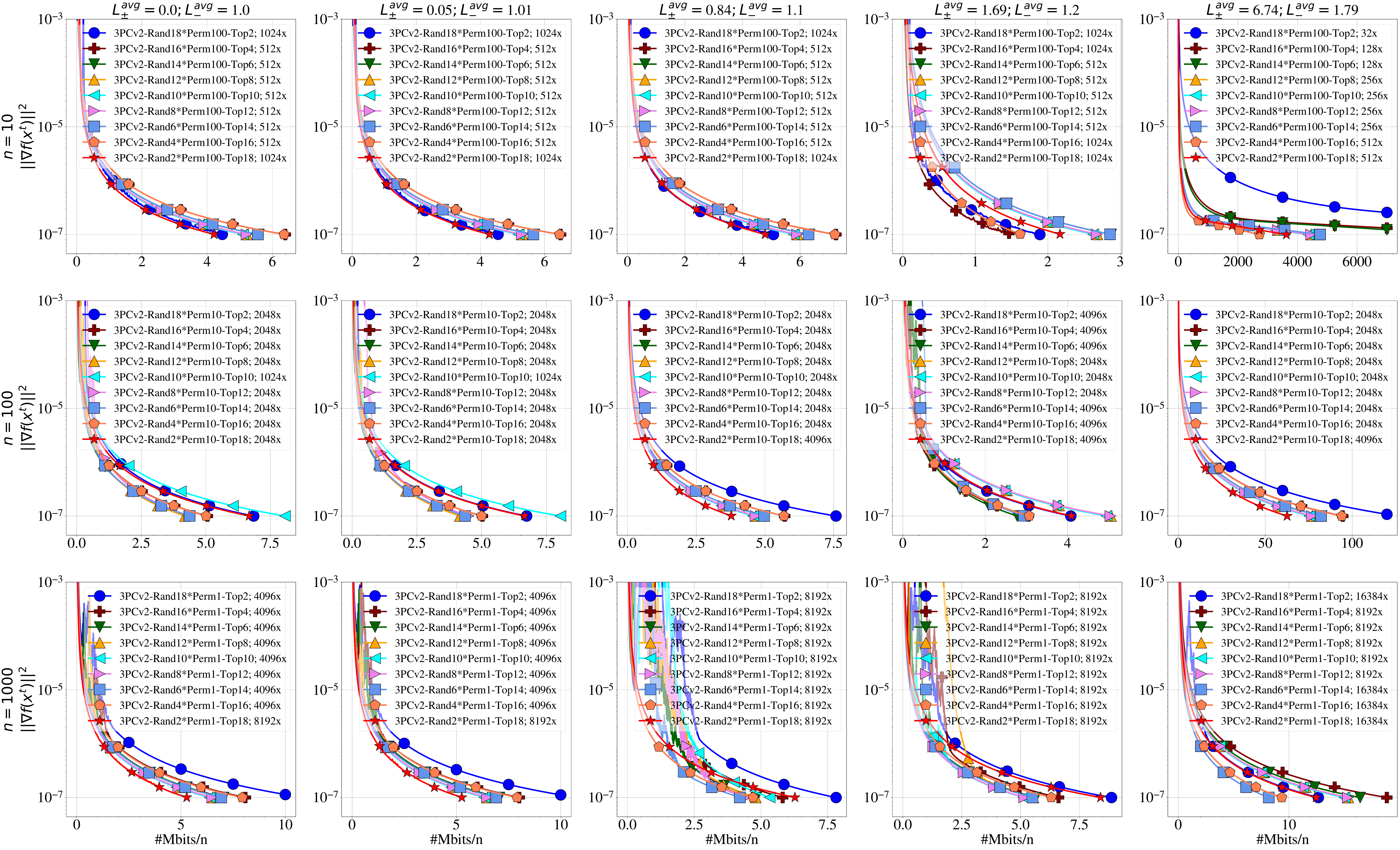}
		\caption{Comparison of \algname{3PCv2} with methods with Rand$K_1*$Perm$K$-Top$K_2$ with $K_1 + K_2 =  0.02{d}$ and tuned stepsizes. By $1\times, 2\times, 4\times$ (and so on) we indicate that the stepsize is set to a multiple of the largest stepsize predicted by theory. $L_{\pm}^{avg}$ and $L_{-}^{avg}$ are the averaged constants $L_{\pm}$ and $L_{-}$ per column.}
		\label{fig:cmprsd-anna-rpt}
\end{figure}

Figures \ref{fig:don-jacqueline} and \ref{fig:cmprsd-jacqueline} show that for the considered sparse quadratic problem in most cases the method \algname{3PCv4} with Top$K_1$-Top$K_2$ compressors behaves as a \algname{EF21} with Top-$K$. Only in a few cases \algname{3PCv4} shows an improvement  over \algname{EF21}.

\begin{figure}[H]
		\includegraphics[width=\linewidth]{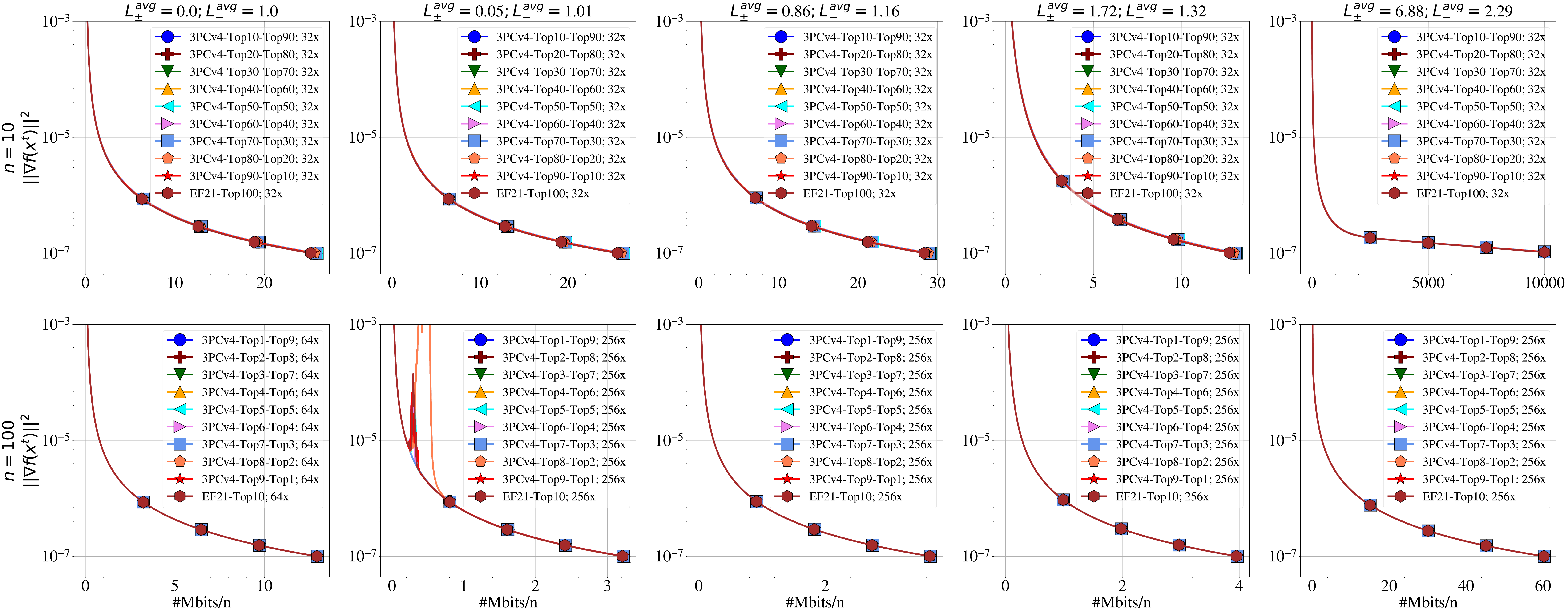}
		\caption{Comparison of \algname{3PCv4} with methods with Top$K_1$-Top$K_2$ with $K_1 + K_2 =  \nicefrac{d}{n}$ and tuned stepsizes. By $1\times, 2\times, 4\times$ (and so on) we indicate that the stepsize is set to a multiple of the largest stepsize predicted by theory. $L_{\pm}^{avg}$ and $L_{-}^{avg}$ are the averaged constants $L_{\pm}$ and $L_{-}$ per column.}
		\label{fig:don-jacqueline}
	\end{figure}
\begin{figure}[H]
		\includegraphics[width=\linewidth]{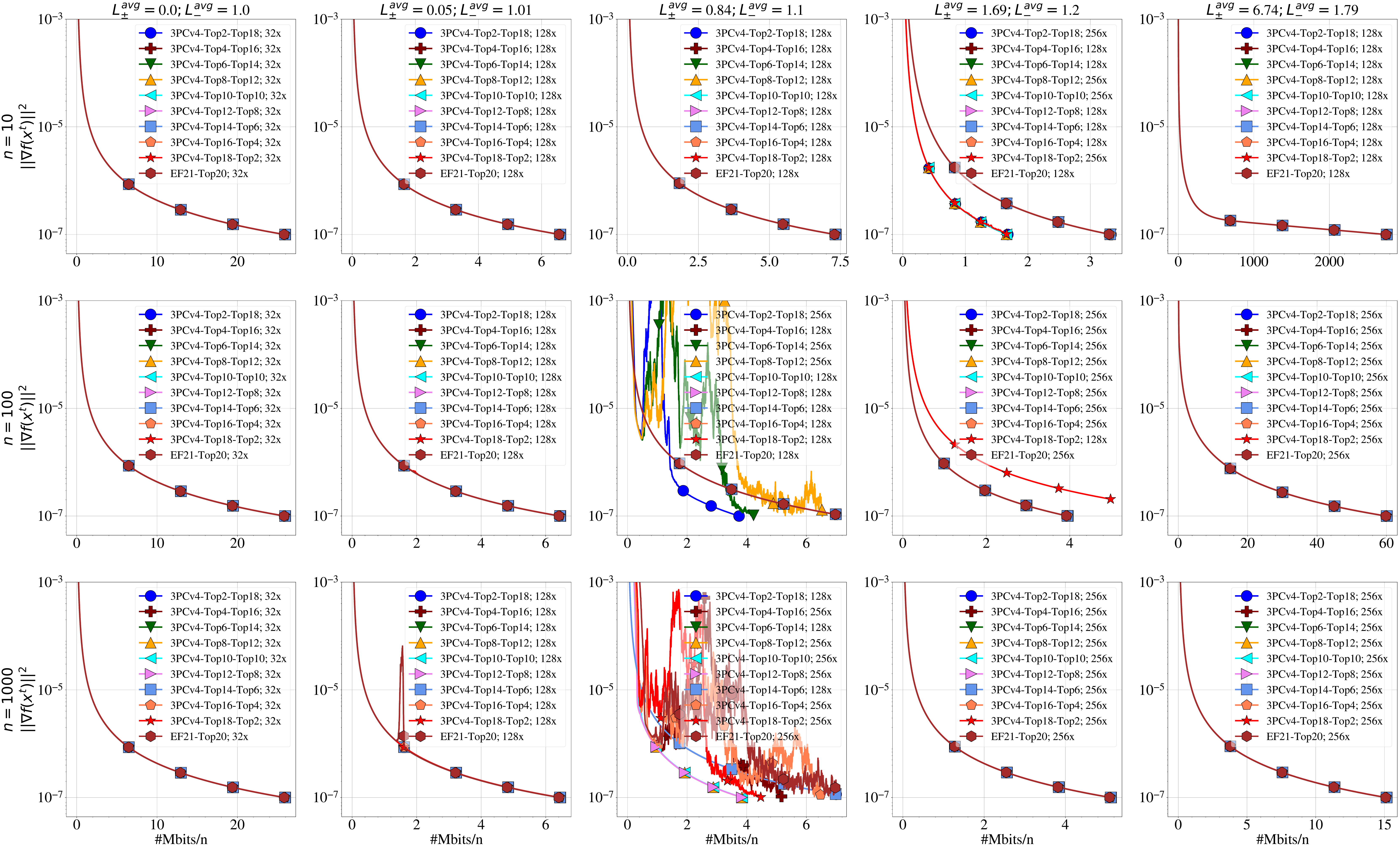}
		\caption{Comparison of \algname{3PCv4} with methods with Top$K_1$-Top$K_2$ with $K_1 + K_2 =  0.02{d}$ and tuned stepsizes. By $1\times, 2\times, 4\times$ (and so on) we indicate that the stepsize is set to a multiple of the largest stepsize predicted by theory. $L_{\pm}^{avg}$ and $L_{-}^{avg}$ are the averaged constants $L_{\pm}$ and $L_{-}$ per column.}
		\label{fig:cmprsd-jacqueline}
\end{figure}

\paragraph{3PCv1}

The next sequence of plots compares \algname{EF21} with Top-$K$, \algname{3PCv1} with Top-$K$ and classical \algname{GD}. Since all methods communicates different amount of floats\footnote{Each node  in \algname{EF21} with Top-$K$ send exactly $K$ floats on server, whereas for \algname{3PCv1} with Top-$K$ and \algname{GD} server receives $d+K$ and $d$ floats from each node respectively.} on each iteration they are compared in terms of the \# communication rounds. Yet being unpractical, \algname{3PCv1} can provide an intuition of how the intermediate method between \algname{GD} and \algname{EF21} could work and what performance can be achieved in \algname{3PCv1} by additional sending $d$ dimensional vector from each node to the server. 
Figure \ref{fig:cmprsd-bd2} illustrates that in low Hessian variance regime \algname{3PCv1} with Top-$K$ behaves as a classical \algname{GD}, whereas in a more heterogeneous regime, it can loose \algname{GD} in terms of the number of communication rounds.
\begin{figure}[H]
	\includegraphics[width=\linewidth]{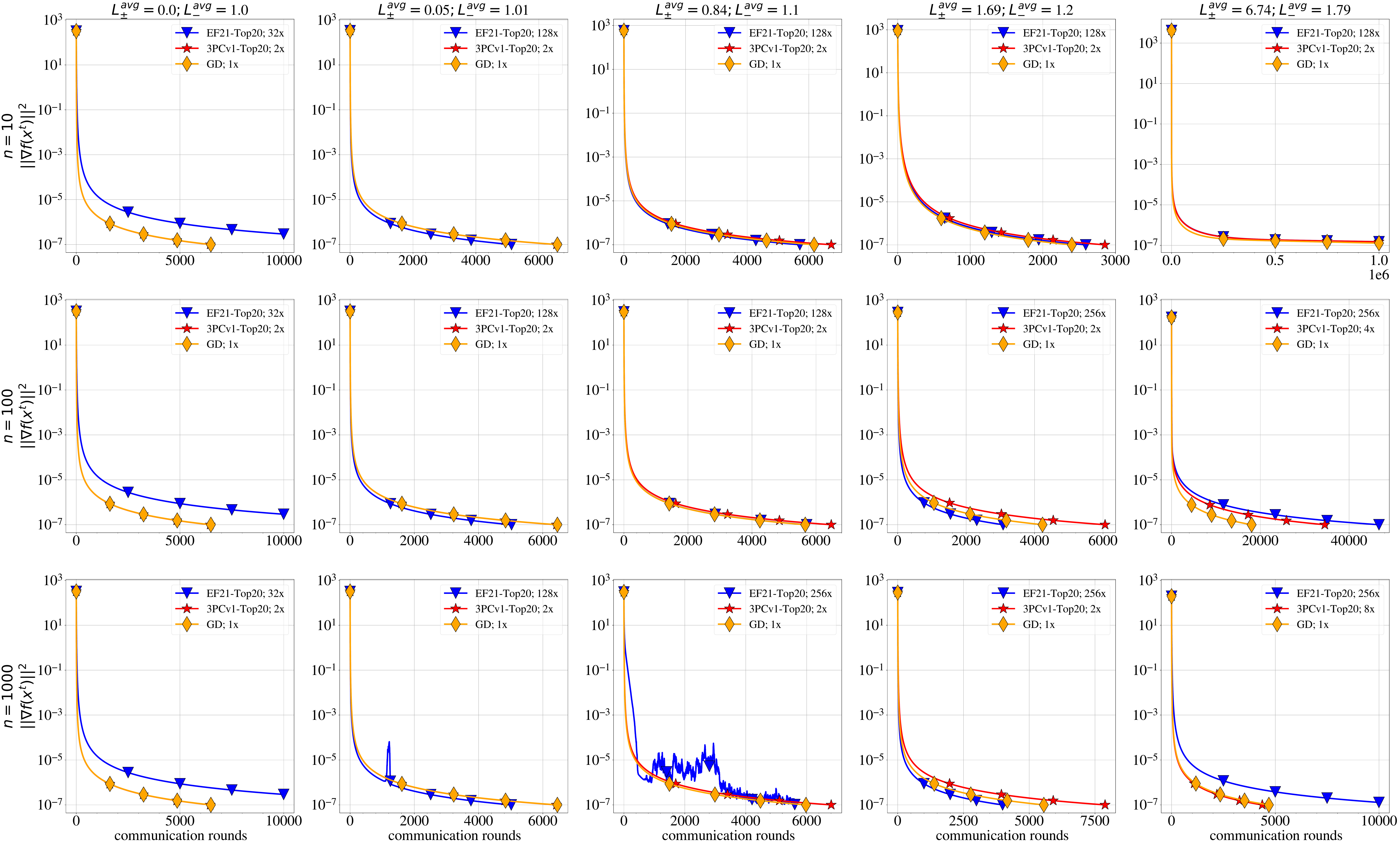}
	\caption{Comparison of \algname{GD} , \algname{3PCv1} with Top-$K$ and \algname{EF21} with Top-$K$ for $ K=  0.02{d}$ and tuned stepsizes. By $1\times, 2\times, 4\times$ (and so on) we indicate that the stepsize is set to a multiple of the largest stepsize predicted by theory. $L_{\pm}^{avg}$ and $L_{-}^{avg}$ are the averaged constants $L_{\pm}$ and $L_{-}$ per column.}
	\label{fig:cmprsd-bd2}
\end{figure}

\subsection{Testing compressed lazy aggregation (\algname{CLAG})}\label{sec:clag}

Following ~\citet{EF21}, we show the performance advantages of \algname{CLAG}. We recall that we are interested in solving the non-convex logistic regression problem,
\begin{align}
\min\limits_{x \in \R^d} f(x) =  \frac{1}{N}\sum\limits_{i=1}^N \log(1 + \exp(-y_i a_i^\top x)) + \lambda \sum\limits_{j=1}^d \frac{x_j^2}{1 + x_j^2},
\end{align}
where $a_i \in \R^d, y_i \in \{-1, 1\}$ are the training data, and $\lambda > 0$ is a regularization parameter. Parameter $\lambda$ is always set to $0.1$ in the experiments. We use four LIBSVM~\cite{chang2011libsvm}  datasets \emph{phishing, w6a, a9a, ijcnn1} as training data. A dataset has been evenly split into $n=20$ equal parts where each part represents a separate client dataset (the remainder of partition between clients has been withdrawn).

{\bfseries Heatmap of communication complexities of \algname{CLAG} for different combinations of parameters.} In our first group of experiments (see Figures~\ref{fig:heatmap_phishing},~\ref{fig:heatmap_w6a},~\ref{fig:heatmap_ijcnn1},~\ref{fig:heatmap_a9a}), we run \algname{CLAG} with Top-$K$ compressor. The compression level $K$ varies evenly between $1$ and $d$, where $d$ is the number of features of a chosen dataset. Trigger $\zeta$ passes zero and subsequent powers of two from zero to eleven. For each combination of $K$ and $\zeta$, we compute empirically the {\em minimum} number of bits per worker sent from clients to the server. Minimum is taken among 12 launches of \algname{CLAG} with different scalings of the theoretical stepsize, scales are powers of two from zero to eleven. The stopping criterion for each launch is based on the condition: $\|\nabla f(x)\| < \delta$, where $\delta$ equals to $10^{-4}$ for {\em phishing} and to $10^{-2}$ for {\em a9a, ijcnn1} and {\em w6a} datasets. Since the algorithm may not converge with too large stepsizes, the time limit of five minutes has been set for one launch. We stress that \algname{CLAG} reduces to \algname{LAG} when $k=d$ and to \algname{EF21} when $\zeta = 0$. The experiment shows that for the most of datasets (excluding {\em phishing}) the minimum communication complexity is attained at a combination of $(K, \zeta)$, which does not reduce \algname{CLAG} to \algname{EF21} or \algname{LAG}. Thus \algname{CLAG} can be consistently faster than \algname{EF21} and \algname{LAG}.

{\bfseries Plots for limited communication cost.} In our second group of experiments (see Figures ~\ref{fig:plot_phishing},~\ref{fig:plot_w6a},~\ref{fig:plot_ijcnn1},~\ref{fig:plot_a9a}), we are in the same setup as in the previous one but this time the stopping criterion bounds the communication cost of algorithms; \algname{CLAG}, \algname{LAG} and \algname{EF21} stop when they first hit the communication cost of 32 Mbits per client. Compression levels $K$ for each dataset at each plot correspond to $1$, $25\%$ and $50\%$ of features. Stepsizes for each algorithm is fine-tuned over the same grid as in the previous experiment. The best $\zeta$ are chosen for \algname{CLAG} and \algname{LAG} from the same grid as in the previous experiment. The experiment exhibits again but from the different perspective the advantages of \algname{CLAG} over its counterparts.

\begin{figure}[H]
\centering
\includegraphics[width=\linewidth]{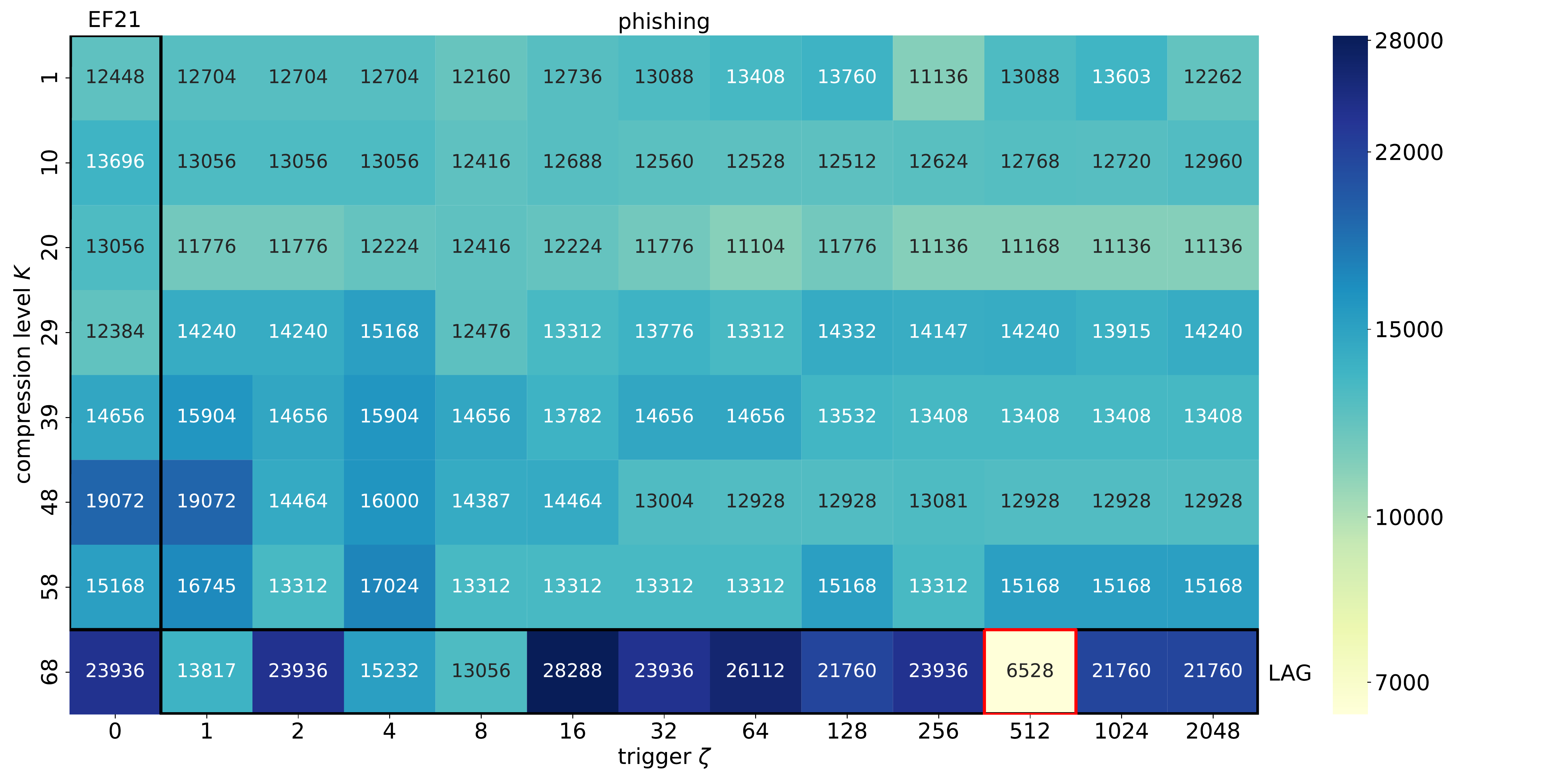}
\caption{Heatmap of communication complexities of \algname{CLAG} for different combination of compression levels and triggers with fine-tuned stepsizes on {\em phishing} dataset. We contour cells corresponding to \algname{EF21} and \algname{LAG}, as special cases of \algname{CLAG}, by black rectangles. The red-countered cell indicates the experiment with the smallest communication cost.}
\label{fig:heatmap_phishing}
\end{figure}
\begin{figure}[H]
	\centering
	\includegraphics[width=\linewidth]{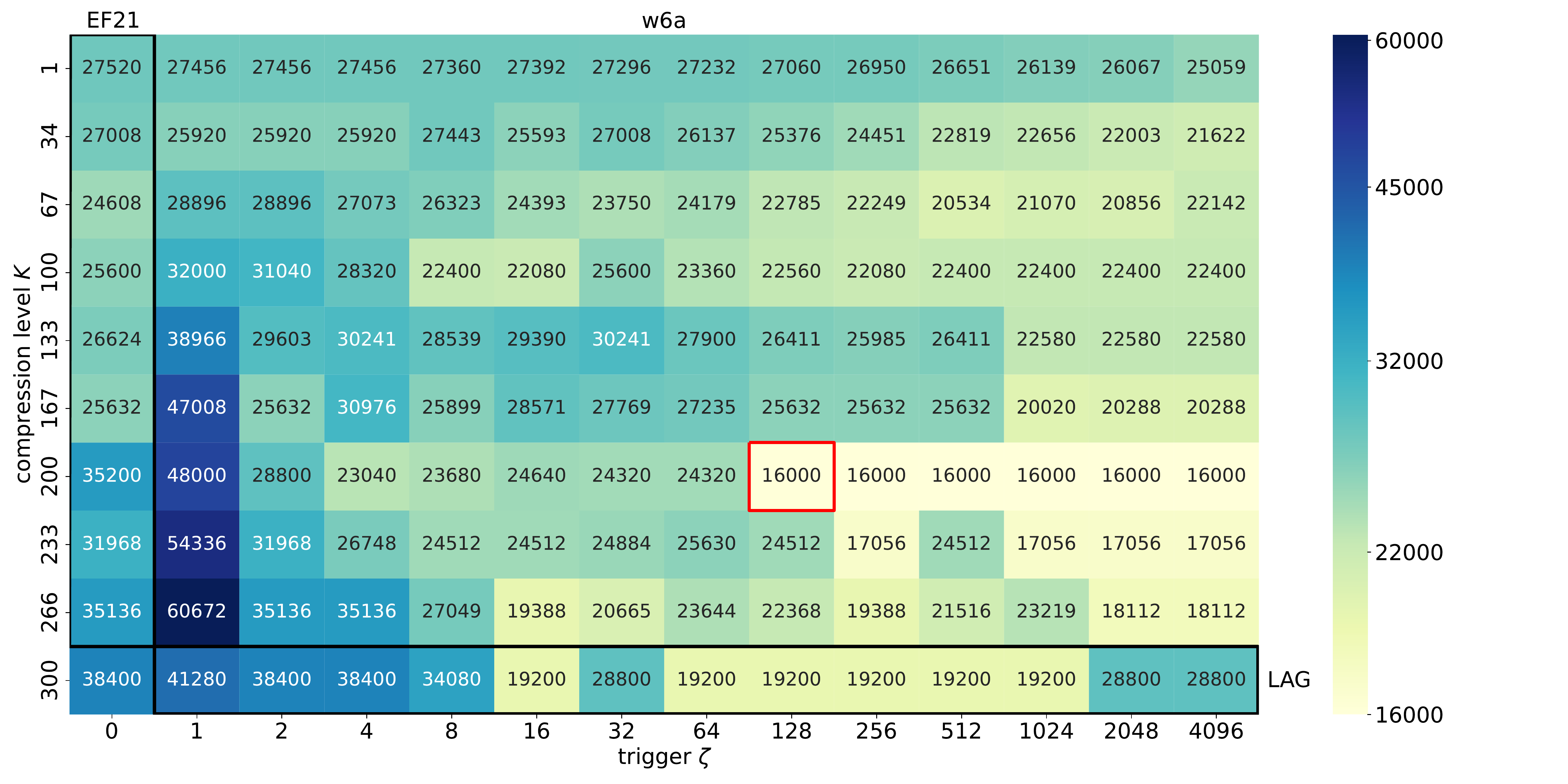}
	\caption{Heatmap of communication complexities of \algname{CLAG} for different combination of compression levels and triggers with fine-tuned stepsizes on {\em w6a} dataset. We contour cells corresponding to \algname{EF21} and \algname{LAG}, as special cases of \algname{CLAG}, by black rectangles. The red-countered cell indicates the experiment with the smallest communication cost.}
	\label{fig:heatmap_w6a}
\end{figure}
\begin{figure}[H]
	\centering
	\includegraphics[width=\linewidth]{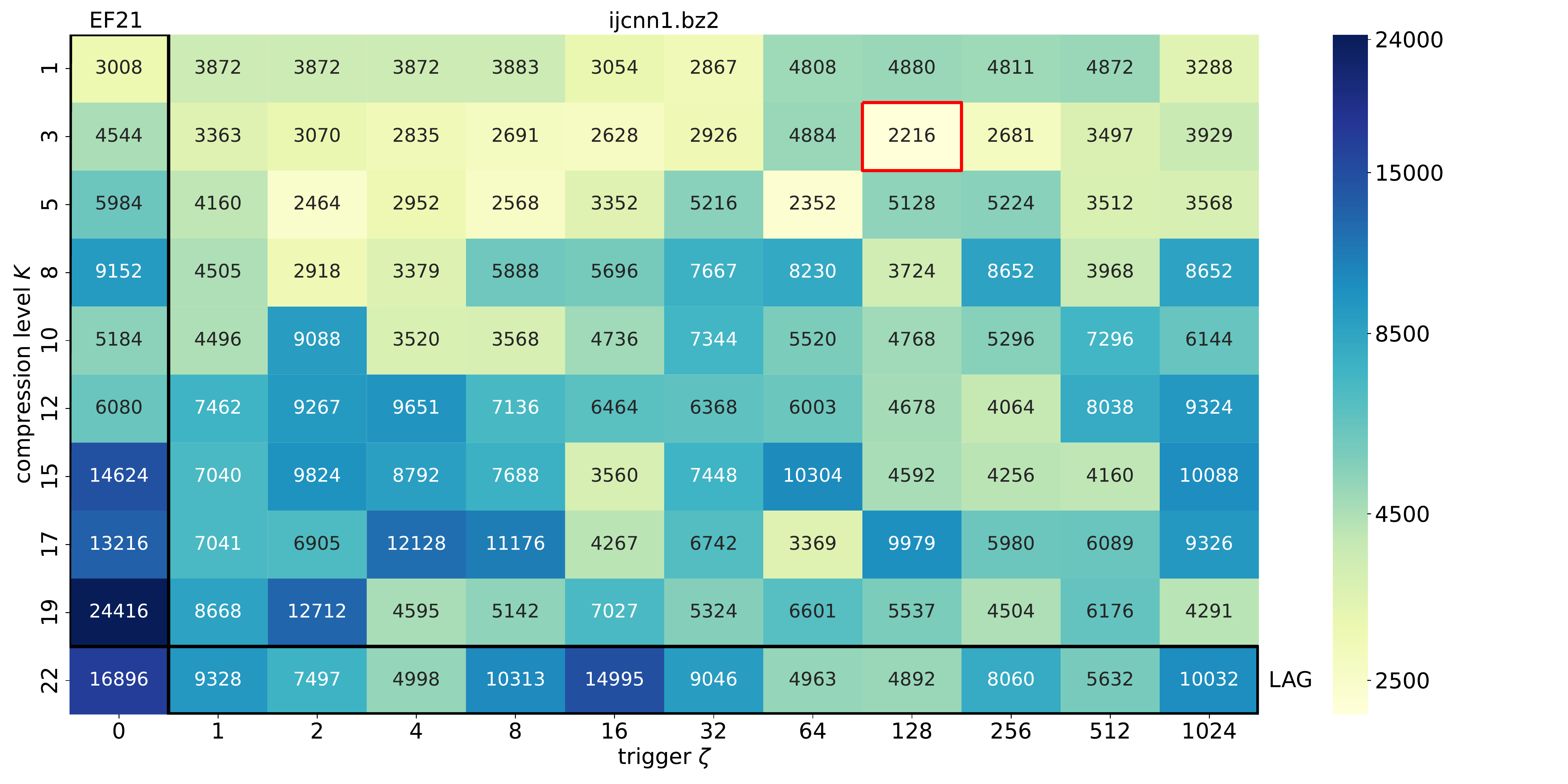}
	\caption{Heatmap of communication complexities of \algname{CLAG} for different combination of compression levels and triggers with fine-tuned stepsizes on {\em ijcnn1} dataset. We contour cells corresponding to \algname{EF21} and \algname{LAG}, as special cases of \algname{CLAG}, by black rectangles. The red-countered cell indicates the experiment with the smallest communication cost.}
	\label{fig:heatmap_ijcnn1}
\end{figure}
\begin{figure}[H]
	\centering
	\includegraphics[width=\linewidth]{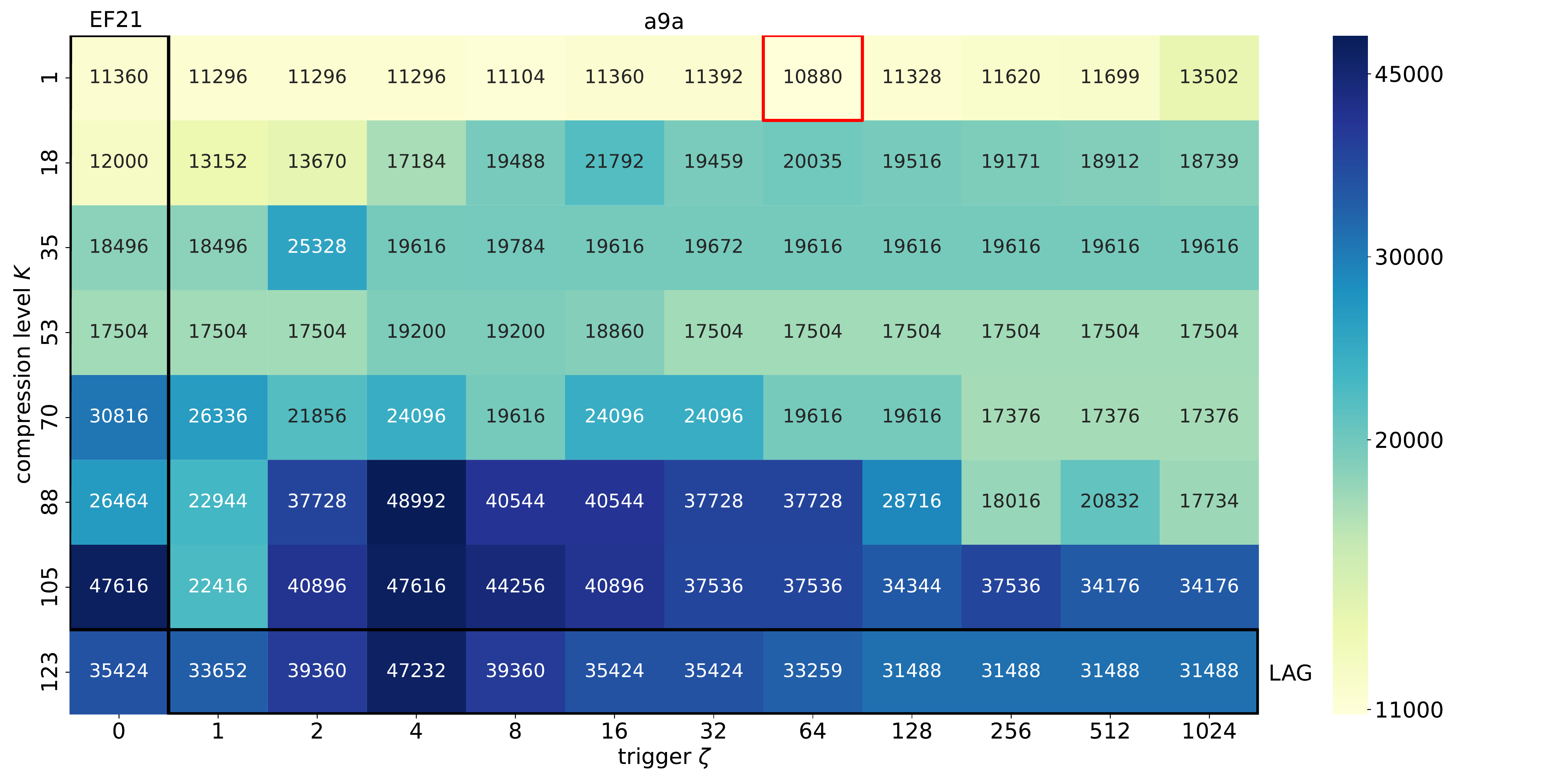}
	\caption{Heatmap of communication complexities of \algname{CLAG} for different combination of compression levels and triggers with fine-tuned stepsizes on {\em a9a} dataset. We contour cells corresponding to \algname{EF21} and \algname{LAG}, as special cases of \algname{CLAG}, by black rectangles. The red-countered cell indicates the experiment with the smallest communication cost.We contour cells corresponding to \algname{EF21} and \algname{LAG}, as special cases of \algname{CLAG}, by black rectangles. The red-countered cell indicates the experiment with the smallest communication cost.}
	\label{fig:heatmap_a9a}
\end{figure}


\begin{figure}[H]
	\centering
	\includegraphics[width=\linewidth]{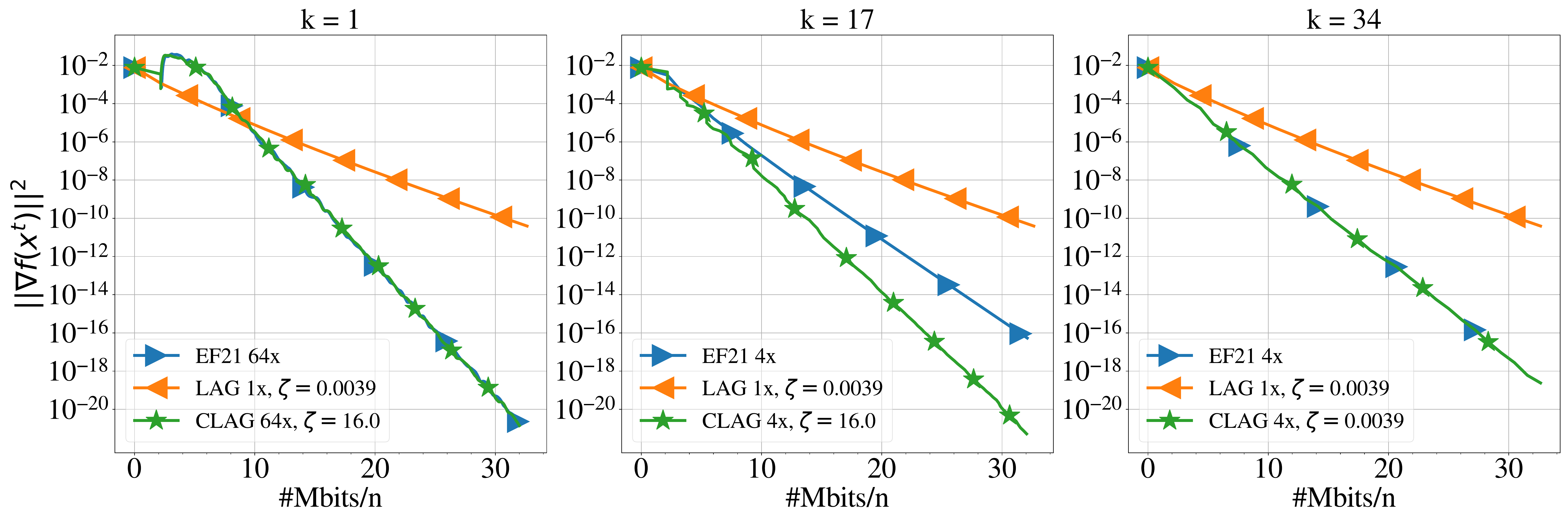}
	\caption{Comparison of \algname{CLAG}, \algname{LAG} and \algname{EF21} with Top-$K$ with fine-tuned stepsizes on {\em phishing} dataset}
	\label{fig:plot_phishing}
\end{figure}
\begin{figure}[H]
	\centering
	\includegraphics[width=\linewidth]{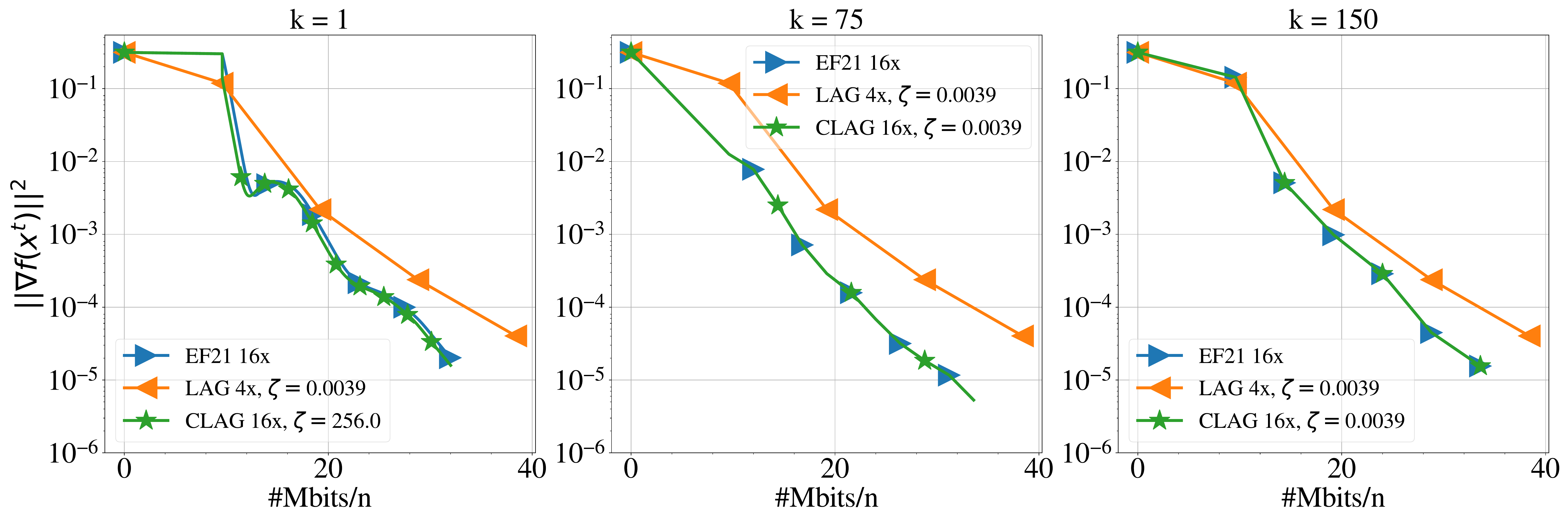}
	\caption{Comparison of \algname{CLAG}, \algname{LAG} and \algname{EF21} with Top-$K$ with fine-tuned stepsizes on {\em w6a} dataset}
	\label{fig:plot_w6a}
\end{figure}
\begin{figure}[H]
	\centering
	\includegraphics[width=\linewidth]{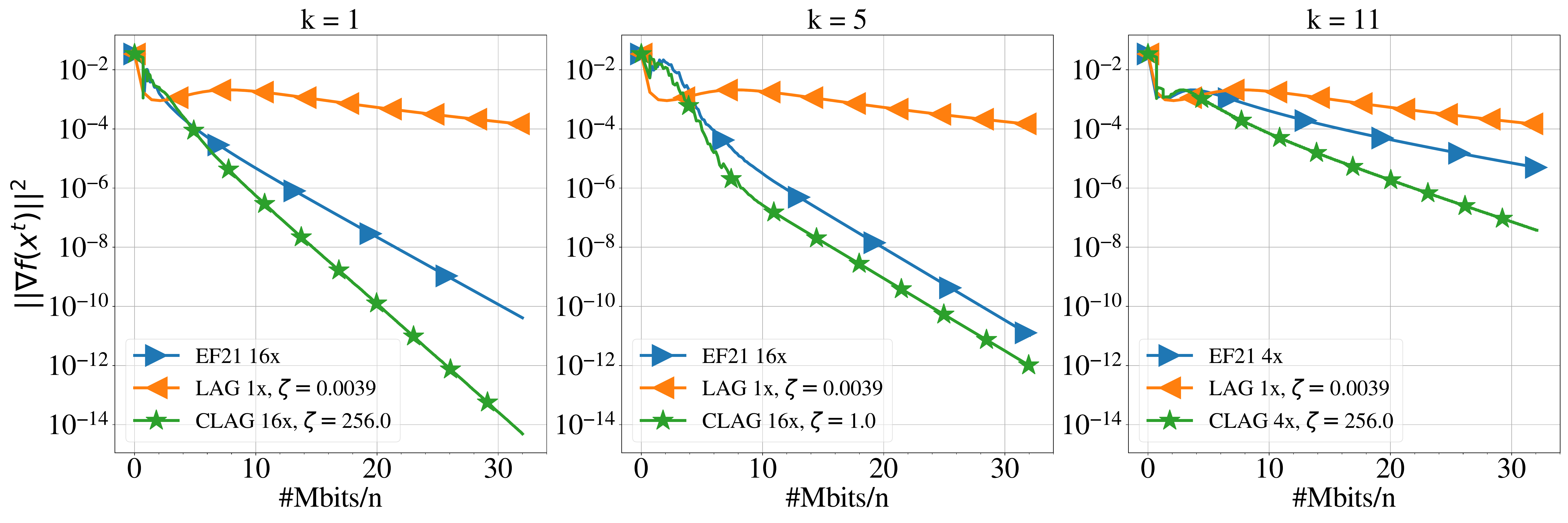}
	\caption{Comparison of \algname{CLAG}, \algname{LAG} and \algname{EF21} with Top-$K$ with fine-tuned stepsizes on {\em ijcnn1} dataset}
	\label{fig:plot_ijcnn1}
\end{figure}
\begin{figure}[H]
	\centering
	\includegraphics[width=\linewidth]{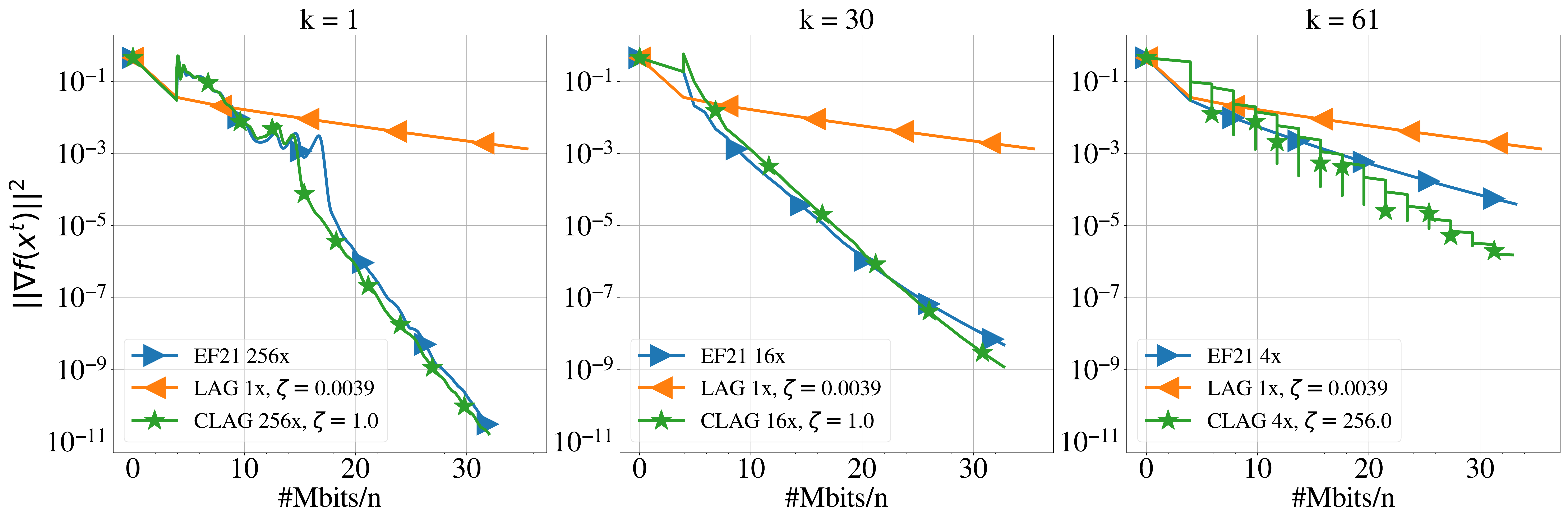}
	\caption{Comparison of \algname{CLAG}, \algname{LAG} and \algname{EF21} with Top-$K$ with fine-tuned stepsizes on {\em a9a} dataset}
	\label{fig:plot_a9a}
\end{figure}

\end{document}